\def\year{2020}\relax
\documentclass[letterpaper]{article} 
\usepackage{aaai20}  
\usepackage{times}  
\usepackage{helvet} 
\usepackage{courier}  
\usepackage[hyphens]{url}  
\usepackage{graphicx} 
\usepackage{graphicx}  
\usepackage{microtype}
\usepackage{graphicx}
\usepackage{subcaption}
\usepackage{booktabs} 

\usepackage{algorithm}
\usepackage[noend]{algorithmic}

\usepackage{amsmath}
\usepackage{amssymb}
\usepackage{mathrsfs}
\usepackage{amsthm}
\usepackage{stmaryrd}
\usepackage{mathtools}

\usepackage{bm}
\usepackage{dblfloatfix}

\usepackage{thmtools, thm-restate}
\declaretheorem{theorem}
\declaretheorem{lemma}

\declaretheorem{definition}

\usepackage{tikz}
\usetikzlibrary{automata, positioning, arrows}
\tikzset{
  ->,
  >=stealth',
  node distance=3cm,
  every state/.style={thick, fill=gray!10},
  initial text=$ $,
}

\usepackage{listings}
\usepackage{color}
\definecolor{codegreen}{rgb}{0,0.6,0}
\definecolor{codegray}{rgb}{0.5,0.5,0.5}
\definecolor{codepurple}{rgb}{0.58,0,0.82}
\definecolor{backcolour}{rgb}{0.95,0.95,0.92}

\lstdefinestyle{mystyle}{
    aboveskip=1em,
    belowskip=0.2em,
    backgroundcolor=\color{backcolour},   
    commentstyle=\color{codegreen},
    keywordstyle=\color{magenta},
    numberstyle=\tiny\color{codegray},
    stringstyle=\color{codepurple},
    basicstyle=\footnotesize\ttfamily,
    breakatwhitespace=false,         
    breaklines=true,                 
    captionpos=b,                    
    keepspaces=true,                 
    numbers=left,                    
    numbersep=3pt,                  
    showspaces=false,                
    showstringspaces=false,
    showtabs=false,                  
    tabsize=2
}
 
\usepackage{hyperref}
\usepackage{cleveref}

\newcommand{\cA}{\mathcal{A}}
\newcommand{\cD}{\mathcal{D}}
\newcommand{\cF}{\mathcal{F}}

\newcommand{\cL}{\mathcal{L}}
\newcommand{\cM}{\mathcal{M}}
\newcommand{\cN}{\mathscr{N}}
\newcommand{\cP}{\mathcal{P}}
\newcommand{\cR}{\mathcal{R}}
\newcommand{\cS}{\mathcal{S}}
\newcommand{\cT}{\mathcal{T}}
\newcommand{\cW}{\mathcal{W}}
\newcommand{\rE}{\mathbb{E}}
\newcommand{\rM}{\mathbb{M}}
\newcommand{\rN}{\mathbb{N}}
\newcommand{\rP}{\mathbb{P}}
\newcommand{\rR}{\mathbb{R}}

\newcommand{\citet}[1]{\citeauthor{#1} (\citeyear{#1})}

\title{Scalable methods for computing state similarity in \\deterministic {M}arkov {D}ecision {P}rocesses}

\author{
  Pablo Samuel Castro \\
  Google Brain \\
  \texttt{psc@google.com} \\
}

\begin{document}

\maketitle

\begin{abstract}
  We present new algorithms for computing and approximating bisimulation
  metrics in Markov Decision Processes (MDPs). Bisimulation metrics are an
  elegant formalism that capture behavioral equivalence between states and
  provide strong theoretical guarantees on differences in optimal behaviour.
  Unfortunately, their computation is expensive and requires a tabular
  representation of the states, which has thus far rendered them impractical
  for large problems. In this paper we present a new version of the
  metric that is tied to a behavior policy in an MDP, along with an analysis of
  its theoretical properties. We then present two new algorithms for
  approximating bisimulation metrics in large, deterministic MDPs. The first
  does so via sampling and is guaranteed to converge to the true metric. The
  second is a differentiable loss which allows us to learn an approximation
  even for continuous state MDPs, which prior to this work had not been
  possible.
\end{abstract}

\section{Introduction}
\label{sec:intro}
A {\bf finite Markov Decision Process (MDP)} is defined as a 5-tuple $\cM =
\langle \cS, \cA, \cP, \cR, \gamma \rangle$, where $\cS$ is a finite set of
states, $\cA$ is a finite set of actions, $\cP:\cS\times\cA\rightarrow
\Delta(\cS)$ is the next state transition function (where $\Delta(X)$ is the
probability simplex over the set $X$), $\cR:\cS\times\cA\rightarrow
\rR$ is the reward function (assumed to be bounded by $R_{max}$), and
$\gamma\in[0, 1)$ is a discount factor. An MDP is the standard formalism
for expressing sequential decision problems, typically in the context of
planning or reinforcement learning (RL). The set of states $\cS$ is one of the
central components of this formalism, where each state $s\in\cS$ is meant to
encode sufficient information about the environment such that an agent can
learn how to behave in a {\em consistent} manner.
\autoref{fig:duplicated_mdp} illustrates a simple MDP where each cell
represents a state.

\begin{figure}
  \centering
  \includegraphics[width=0.35\textwidth]{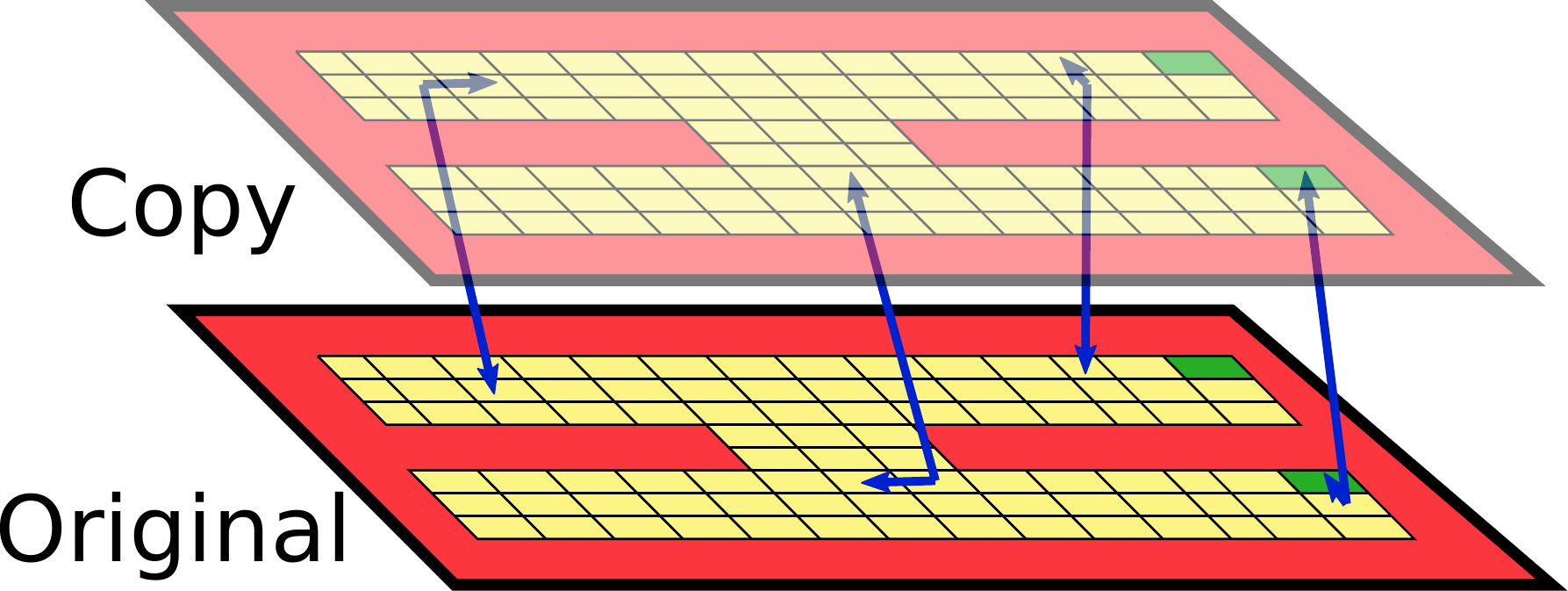}
  \caption{A grid MDP (bottom) with a copy of itself (top). The goal of an
  agent is to find the shortest path to the green cells. At each iteration, the
  agent has a 50\% chance of jumping to the other level.}
  \label{fig:duplicated_mdp}
\end{figure}

There is no canonical way of defining the set of states for a problem. Indeed,
improperly designed state spaces can have drastic effects on the algorithm
used. Consider the grid MDP in the bottom of \autoref{fig:duplicated_mdp},
where an agent must learn how to navigate to the green cells, and imagine we
create an exact replica of the MDP such that the agent randomly transitions
between the two layers for each move. By doing so we have doubled the number of
states and the complexity of the problem. However, from a planning perspective
the two copies of each state should be indistinguishable. A stronger notion of
{\em state identity} is needed that goes beyond the labeling of states and
which is able to capture {\em behavioral indistinguishability}.

In this paper we explore notions of behavioral similarity via state
pseudometrics\footnote{A pseudometric is a metric $d$ where $\forall
s,t\in\cS.\quad s = t \implies d(s, t) = 0$, but not the converse.}
$d:\cS\times\cS\rightarrow\rR$, and in particular those which assign a distance
of $0$ to states that are behaviorally indistinguishable. Pseudometrics
further allow us to reason about states based on what we may know about other
similar states. This is a common use-case in fields such as formal
verification, concurrency theory, and in safe RL, where one may want to provide
(non-)reachability guarantees. In the context of planning and reinforcement
learning, these can be useful for state aggregation and abstraction.

Our work builds on bisimulation metrics \cite{ferns04metrics} which provide us
with theoretical guarantees such as states that are close to each other (with
respect to the metric) will have similar optimal value functions.  These
theoretical properties render them appealing for planning and learning, and
they have previously been used for state aggregation
\cite{ferns04metrics,ferns06methods}, policy transfer \cite{castro10using},
representation discovery \cite{ruan15representation}, and exploration
\cite{santara19extra}.

Unfortunately, these metrics are expensive to compute \cite{chen12complexity}
and require fully enumerating the states, even when using sampling-based
approximants \cite{ferns06methods}, on-the-fly approximants
\cite{comanici12onthefly,bacci13onthefly}, or approximants which exploit
structure in the state space \cite{bacci13computing}.  The full-state
enumeration requirement has thus far rendered bisimulation metrics impractical
in problems with large state spaces, and completely incompatible with
continuous state spaces. Additionally, bisimulation metrics can be overly {\em
pessimistic} in the sense that they consider worst-case differences between
states. Although desirable for certain applications, such as in guaranteeing
safe behaviors, it can prove overly restrictive for many practical problems of
interest.

In this paper we address these impracticalities with the following key
contributions:
\begin{enumerate}
  \item An {\em on-policy} variant of bisimulation which focuses only on the
    behavior of interest, rather than worst-case scenarios, along with an
    analysis of its theoretical properties.
  \item A new sampling-based online algorithm for exact computation of the
    original and on-policy bisimulation metrics with guaranteed convergence.
  \item A differentiable loss function for learning an approximation of the two
    bisimulation metrics using neural networks. We provide empirical evidence
    of this learning algorithm on MDPs with large and continuous state
    spaces. To the best of the author's knowledge, this is the first work
    proposing a mechanism for approximating bisimulation metrics with neural
    networks.
\end{enumerate}


\section{Background}
\label{sec:background}
Given an MDP $\cM$, a {\em policy} $\pi:\cS\rightarrow \Delta(\cA)$ induces a
corresponding state-value function $V^{\pi}:\cS\rightarrow\rR$
\cite{puterman94mdp}: $V^{\pi}(s) = \rE_{a\sim\pi(s)} \left[ \cR(s, a) +
\gamma\rE_{s'\sim\cP(s, a)}V^{\pi}(s')\right]$.  In the control setting, we are
typically in search of the optimal value function~\cite{bellman57dp}:\\ $V^*(s)
= \max_{a\in\cA}\left[\cR(s, a) + \gamma\rE_{s'\sim\cP(s, a)}V^*(s')\right]$.

Bisimulation relations, originally introduced in the field of concurrency
theory, were adapted for MDPs by \citet{givan03equivalence}, capture a
strong form of behavioral equivalence: if $s,t\in\cS$ are bisimilar,
then $V^*(s) = V^*(t)$.

\begin{definition}
  Given an MDP $\cM$, an equivalence relation $E\subseteq\cS\times\cS$ is a
  {\bf bisimulation relation} if whenever $(s,t)\in E$ the following
  properties hold, where $\cS_{E}$ is the state space $\cS$ partitioned into
  equivalence classes defined by $E$:
  \begin{enumerate}
    \item $\forall a\in\cA,\cR(s, a)=\cR(t, a)$
    \item $\forall a\in\cA,\forall c\in\cS_{E},\cP(s, a)(c) = \cP(t, a)(c)$,
      where $\cP(s, a)(c) = \sum_{s'\in c}\cP(s, a)(s')$,
  \end{enumerate}
  Two states $s,t\in\cS$ are {\bf bisimilar} if there exists a bisimulation
  relation $E$ such that $(s,t)\in E$. We denote the largest\footnote{Note that
  there can be a number of equivalence relations satisfying these properties.
  The smallest is the identity relation, which is vacuously a bisimulation
  relation.} bisimulation relation as $\sim$.
\end{definition}

%

Equivalence relations can be brittle: they require exact equivalence under
probabilistic transitions. This can be especially problematic if we are
estimating transition probabilities from data, as it is highly unlikely they
will match exactly.

Extending the work of \citet{desharnais99metrics} for labeled Markov processes,
\citet{ferns04metrics} generalized the notion of MDP bisimulation relations to
metrics, yielding a smoother notion of similarity than equivalence relations.
Let $\rM$ be the set of all pseudometrics on $\cS$. A pseudometric $d\in\rM$
induces an equivalence relation $E_d := \lbrace (s, t) | d(s, t) = 0\rbrace$.
That is, any two states with distance $0$ will be collapsed onto the same
equivalence class.

\begin{definition}
  \cite{ferns04metrics}
  A pseudometric $d\in\rM$ is a {\bf bisimulation metric} if $E_d$ is $\sim$.
\end{definition}

Bisimulation metrics use the 1-Wasserstein metric
$\cW_1:\rM\rightarrow\rP$, where $\rP$ is the set of all metrics between
probability distributions over $\cS$. Given two state distributions $X,Y\in
\Delta(\cS)$ and a pseudometric $d\in\rM$, the Wasserstein $\cW_1(d)(X,
Y)$ can be expressed by the following (primal) linear program (LP),
which ``lifts'' a pseudometric $d\in\rM$ onto
one in $\rP$ \cite{villani08optimal}:
\begin{align}
  \label{eqn:primal_lp}
  \max_\textbf{u}\in\rR^{|\cS|} \sum_{s\in\cS}\left(X(s) - Y(s)\right)u_{s} \\
  \forall s, s'\in\cS,\textrm{ }u_{s} - u_{s'} \leq d(s, s') \nonumber \\
  0\leq \textbf{u}\leq 1 \nonumber
\end{align}
\begin{theorem}
  \label{thm:ferns04}
  \cite{ferns04metrics}:
  Define $\cF:\rM\rightarrow\rM$ by
  \begin{align}
    \label{eqn:functor}
    \cF&(d)(s, t) = \\
    &\max_{a\in\cA}\left(\left|\cR(s,a) - \cR(t,a)\right| + \gamma \cW_1(d)(\cP(s, a), \cP(t, a))\right) \nonumber
  \end{align}
  then $\cF$ has a unique fixed point, $d_{\sim}$, and $d_{\sim}$ is a
  bisimulation metric.
\end{theorem}

The operator $\cF$ can be used to iteratively compute a bisimulation metric as
follows. Starting from an initial estimate $d_0$, we can compute
$d_{n+1} = \cF(d_n) = \cF^{n+1}(d_0)$.
By iteratively applying $\cF$ $\lceil\frac{\ln\delta}{\ln\gamma}\rceil$ times,
one can compute $d_{\sim}$ up to an accuracy $\delta$, with an overall
complexity of $O\left(|\cA||\cS|^4\log |\cS|\frac{\ln\delta}{\ln
\gamma}\right)$.

\section{On-policy bisimulation}
\label{sec:onpolicy_bisim}
The strong theoretical guarantees of bisimulation relations and metrics are
largely due to their inherent ``pessimism'': they consider equivalence under
all actions, even pathologically bad ones (i.e. actions that never lead to positive
outcomes for the agent). Indeed, there exist systems where
$V^*(s)=V^*(t)$, but $d_{\sim}(s,t)$ can be arbitrarily large, providing no
useful bounds on the optimal behaviour from $s$ and $t$ (see
\autoref{fig:pessimism}). \citet{castro10using} also demonstrated that this
pessimism yields poor results when using bisimulation metrics for policy
transfer.

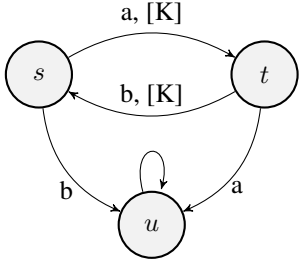
\begin{figure}
  \centering
  \begin{tikzpicture}
    \node[state] (s) {$s$};
    \node[state, right of=s] (t) {$t$};
    \node[state] at (1.5, -2) (u) {$u$};

    \draw (s) edge[bend left, above] node{a, [K]} (t)
          (s) edge[bend right, below] node{b} (u)
          (t) edge[bend left, above] node{b, [K]} (s)
          (t) edge[bend left, below] node{a} (u)
          (u) edge[loop above] (u);
  \end{tikzpicture}
  \caption{Edge labels indicate action ($\lbrace a,b\rbrace$) and non-zero
  rewards ($[K]$). When $\gamma=0.9$, $V^*(s)=V^*(t)=10K$, while $d_{\sim}(s,
  t)=10K$.  Lax bisimulation assigns distance $0$ between $s$ and $t$ (action
  $a$ from $s$ would be matched with action $b$ from $t$).}
  \label{fig:pessimism}
\end{figure}

Another disadvantage of bisimulation relations and metrics is that they are
computed via exact action matching between states; however, actions with the
same label may induce very different behaviours from different states,
resulting in an improper behavioral comparison when using bisimulation.
In the system in \autoref{fig:pessimism}, $s$ and $t$ have equal optimal
values, but their optimal action is different ($a$ from $s$, $b$ from $t$).
\citet{taylor09bounding} overcame this problem by the introduction of
lax-bisimulation metrics (definition and theoretical results provided in the
supplemental). We note, however, that their method is still susceptible to the
pessimism discussed above.

It is often the case that one is interested in behaviours relative
to a particular policy $\pi$. In reinforcement learning, for example, many
algorithms maintain a behaviour policy which is improved iteratively as the
agent interacts with the environment. In these situations the pessimism of
bisimulation can become a hindrance: if the action maximizing the distance
between two states is never chosen by $\pi$, we should not include it in the
computation!

We introduce a new notion of bisimulation, {\em on-policy
bisimulation}, defined relative to a policy $\pi$. This new notion
also removes the requirement of matching on action labels by considering the
dynamics induced by $\pi$, rather than the dynamics induced by each action.
We first define:
\begin{align*}
  \cR^{\pi}_s & := \sum_{a}\pi(a | s) \cR(s, a) \\
  \forall C\in \cS_{E^{\pi}}, \cP^{\pi}_s(C) & := \sum_{a}\pi(a | s)\sum_{s'\in C}P(s, a)(s')
\end{align*}

\begin{definition}
  \label{def:pi_bisim_rel}
  Given an MDP $\cM$, an equivalence relation $E^{\pi}\subseteq\cS\times\cS$ is
  a {\bf $\pi$-bisimulation relation} if whenever $(s,t)\in E^{\pi}$ the
  following properties hold:
  \begin{enumerate}
    \item $\cR^{\pi}_s = \cR^{\pi}_{t}$
    \item $\forall C\in \cS_{E^{\pi}}, \cP^{\pi}_s(C) = P^{\pi}_{t}(C)$
  \end{enumerate}
  Two states $s,t\in\cS$ are {\bf $\pi$-bisimilar} if there exists a
  $\pi$-bisimulation relation $E^{\pi}$ such that $(s,t)\in E^{\pi}$. Denoting
  the largest bisimulation relation as $\sim_{\pi}$, $d\in\rM$ is a {\bf
  $\pi$-bisimulation metric} if $E_d$ is $\sim_{\pi}$.
\end{definition}

\begin{restatable}{theorem}{pibisimop}
  \label{thm:pi_bisim_operator}
  Define $\cF^{\pi}:\mathcal{M}\rightarrow\mathcal{M}$ by
  $\cF^{\pi}(d)(s, t) = |\cR^{\pi}_s - \cR^{\pi}_{t}| + \gamma
  \cW_1(d)(\cP^{\pi}_s, \cP^{\pi}_{t})$, 
  then $\cF^{\pi}$ has a least fixed point $d^{\pi}_{\sim}$, and
  $d^{\pi}_{\sim}$ is a $\pi$-bisimulation metric.
\end{restatable}
\begin{proof}
  (Sketch) This proof mimics the proof of Theorem~4.5 from
  \cite{ferns04metrics}. All complete proofs are provided in the supplemental
  material.
\end{proof}

The following result demonstrates that $\pi$-bisimulation metrics provide
similar theoretical guarantees as regular bisimulation metrics, but with
respect to the value function induced by $\pi$.
\begin{restatable}{theorem}{pibisimbound}
  Given any two states $s,t\in\cS$ in an MDP $\cM$, $|V^{\pi}(s) - V^{\pi}(t) |
  \leq d^{\pi}_{\sim}(s, t)$.
\end{restatable}
\begin{proof}
  (Sketch) This is proved by induction. The result follows by expanding
  $V^{\pi}$, taking the absolute value of each term separately, and noticing
  that $V^{\pi}$ is a feasible solution to the primal LP in
  \autoref{eqn:primal_lp}, so is upper-bounded by $\cW_1(d^{\pi}_n)$.
\end{proof}

Under a fixed policy $\pi$, an MDP reduces to a Markov chain. Bisimulation
relations for Markov chains have previously been studied in concurrency theory
\cite{baier06bisimulation,katoen07bisimulation}.  Further, $\pi$-bisimulation
can be used to define a notion of {\em weak}-bisimulation for MDPs
\cite{baier06bisimulation,fioriti16deciding}.

\section{Bisimulation metrics for deterministic MDPs}
\label{sec:deterministic_bisim}
In this section we investigate the properties of deterministic MDPs, which in
concurrency theory are known as transition systems
\cite{sangiorgi11introduction}.

\begin{definition}
  A {\em deterministic} MDP $\cM$ is one where for all $s\in\cS,a\in\cA$,
  there exists a unique $\cN(s,a)\in\cS$ such that $\cP(s, a)(\cN(s,a))=1$.
\end{definition}

As the next lemma shows, under a system with deterministic transitions,
computing the Wasserstein metric (and approximants) is no longer necessary.

\begin{restatable}{lemma}{kantisd}
  \label{lemma:kant_is_d}
  Given a deterministic MDP $\cM$, for any two states $s,t\in\cS$, action
  $a\in\cA$, and pseudometric $d\in\rM$, $\cW_1(d)(\cP(s, a), \cP(t, a))=d(\cN(s,
  a), \cN(t, a))$.
\end{restatable}
\begin{proof}
  (Sketch) The result follows by considering the dual formulation of the primal LP
  in \autoref{eqn:primal_lp}, which implies the dual variables $\lambda_{s,t}$
  must all be either $1$ or $0$, by virtue of determinism.
\end{proof}

By considering only deterministic policies (e.g. policies that assign
probability $1$ to a single action) in the on-policy case,
\autoref{lemma:kant_is_d} allows us to rewrite the operator $\cF(d)(s, t)$ in \autoref{thm:ferns04}
and $\cF^{\pi}(d)(s, t)$ in \autoref{thm:pi_bisim_operator} as:\\
$\max_{a\in\cA}\left(\left|\cR(s,a) - \cR(t,a)\right| + \gamma d(\cN(s, a), \cN(t, a))\right)$ and\\ 
$|\cR(s, \pi(s)) - \cR(t, \pi(t))| + \gamma d(\cN(s, \pi(s)), \cN(t, \pi(t)))$, respectively.
Note the close resemblance to value functions, there is in fact
a strong connection between the two: \citet{ferns14bisimulation} proved that
$d_{\sim}$ is the optimal value function of an optimal coupling of two copies
of the original MDP.

Even in the deterministic setting, the computation of bisimulation
metrics can be intractable in MDPs with very large or continuous state
spaces. In the next sections we will leverage the results just presented to
introduce new algorithms that are able to scale to large state spaces and
learn an approximant for continuous state spaces.

\section{Computing bisimulation metrics with sampled transitions}
\label{sec:online_bisim}
We present the algorithm and results in this section for the original
bisimulation metric, $d_{\sim}$, but all the results presented here hold for
the on-policy variant $d^{\pi}_{\sim}$; the main difference is that actions in
the trajectory are given by $\pi$ and thus, may differ between states being
compared.

The update operator $\cF$ is generally applied in a dynamic-programming fashion:
all state-pairs are updated in each iteration by considering all possible
actions. However, requiring access to all state-pairs and actions in each
iteration is often not possible, especially when data is concurrently being
collected by an agent interacting with an environment. In this section we
present an algorithm for computing the bisimulation metric via access to
\emph{transition samples}. Specifically, assume we are able to sample pairs of
transitions $\lbrace\langle s, a, \cR(s, a), \cN(s, a) \rangle, \langle t, a,
\cR(t, a), \cN(t, a)\rangle\rbrace$ from an underlying distribution $\cD$ (note
the action is the same for both). This can be, for instance, a uniform
distribution over all transitions in a replay memory \cite{mnih15humanlevel}
or some other sampling procedure. Let $\cT$ be the set of all pairs of
valid transitions; for legibility we will use the shorthand
$\tau_{s,t,a}\in\cT$ to denote a pair of transitions from states $s, t\in\cS$
under action $a\in\cA$. We assume that $\cD(\tau) > 0$ for all $\tau\in\cT$.

We first define an iterative procedure for computing $d_{\sim}$ by sampling
from $\cD$.  Let $d_0\equiv 0$ be the everywhere-zero metric. At step $n$, let
$\tau_{s_n,t_n,a_n}\in\cT$ be a sample from $\cD$ and define $d_n$ as:
\begin{align}
  \label{eqn:online_update}
  d_n(s, t) & = d_{n-1}(s, t),\qquad \forall s\neq s_n, t\neq t_n \nonumber \\
  d_n(s_n, t_n) & = \max
  \begin{bmatrix}
    d_{n-1}(s_n, t_n), \\
   |\cR(s_n, a_n) - \cR(t_n, a_n)| + \\
    \quad\gamma d_{n-1}(\cN(s_n, a_n), \cN(t_n, a_n))
  \end{bmatrix}
\end{align}

In words, when we sample a pair of states, we only update the distance estimate
for these two states if applying the $\cF$ operator gives us a larger estimate.
Otherwise, our estimate remains unchanged.

\begin{restatable}{theorem}{onlinebisim}
  \label{thm:online_bisim}
  If $d_n$ is updated as in \autoref{eqn:online_update} and $d_0\equiv 0$,
  $\lim_{n\rightarrow\infty}d_n = d_{\sim}$ almost surely.
\end{restatable}
\begin{proof}
  (Sketch) We first show that since we are sampling state pairs and actions
  infinitely often, all state pairs will receive a non-vacuous update at least
  once (Maximizing action lemma); then show that $d_n\leq d_{\sim}$ for all
  $n\in\rN$ (Monotonicity lemma). We then use these two results to show that the
  difference $\| d_{\sim} - d_n\|_{\infty}$ is a contraction and the result
  follows by the Banach fixed-point theorem.
  Note that the maximizing action lemma as presented here is for the original
  bisimulation metric; for the on-policy variant, the equivalent result is
  that all states in the Markov chain induced by $\pi$ are updated infinitely
  often.

\end{proof}

\section{Learning an approximation}
\label{sec:learn_bisim}
We leverage the sampling approach from the previous section to devise a
learning algorithm for approximating $d_{\sim}$ and $d^{\pi}_{\sim}$ for MDPs
with large (or continuous) state spaces, using function approximators in the
form of neural networks. Let $\phi:\cS\rightarrow\rR^k$ be a $k$-dimensional
representation of the state space and let $\psi_\theta:\rR^{2k}\rightarrow \rR$
be a neural network parameterized by $\theta$, that receives a concatenation of
two state representations such that $\psi_\theta([\phi(s), \phi(t)])\approx
d_{\sim}(s, t)$ (see \autoref{fig:network}).

\begin{figure}[!b]
  \centering
  \includegraphics[width=0.45\textwidth]{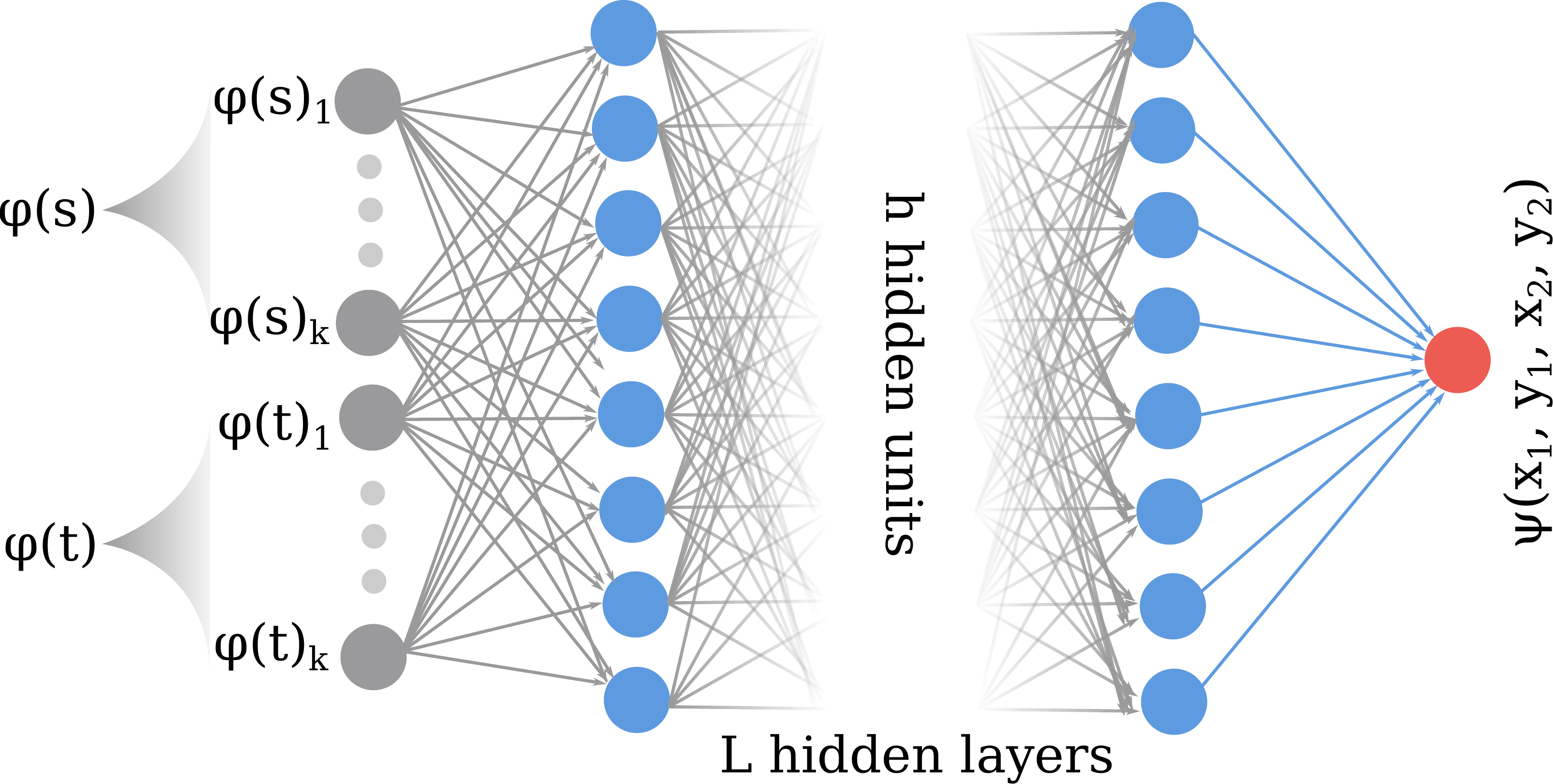}
  \caption{Using a neural network for learning $\psi$ as an approximant to
  $d_{\sim}$ or $d^{\pi}_{\sim}$.}
  \label{fig:network}
\end{figure}

Following the practice introduced by \cite{mnih15humanlevel} we make use of
online parameters $\theta$ and target parameters $\theta^{-}$, where the online
parameters are updated at each iteration while the target parameters are
updated every $C$ iterations. Given a pair of states $s\ne t$ and action
$a\in\cA$, at iteration $i$ we define the target objective
$\mathbf{T}_{\theta^{-}_i}(s, t, a)$ for $d_{\sim}$ as:
\small
\begin{align*}
  \max
  \begin{bmatrix}
    |\cR(s, a) - \cR(t, a)| + \gamma \psi_{\theta^{-}_i}([\phi(\cN(s, a)), \phi(\cN(t, a))]), \\
    \psi_{\theta^{-}_i}([\phi(s), \phi(t)])
  \end{bmatrix}
\end{align*}
\normalsize
and equal to $0$ whenever $s=t$. The target objective
$\mathbf{T}^{\pi}_{\theta^{-}_i}(s, t)$ for $d^{\pi}_{\sim}$ is:
\begin{align*}
  |\cR(s, \pi(s)) & - \cR(t, \pi(t))| + \\
  & \gamma \psi^{\pi}_{\theta^{-}_i}([\phi(\cN(s, \pi(s))), \phi(\cN(t, \pi(t)))])
\end{align*}

We can then define our loss $\cL^{(\pi)}_{s,t,a}$ as: $\rE_{\cD}\left(
\mathbf{T}^{(\pi)}_{\theta^{-}_i}(s, t, a) - \psi^{(\pi)}_{\theta_i}([\phi(s),
\phi(t)]) \right)^2$.  This loss is specified for a single pair of transitions,
but we can define an analogous target and loss with mini-batches, which allows
us to train our approximant more efficiently using specialized hardware such as
GPUs:
\begin{align*}
  \mathbf{T} & = (1 - \mathbf{I}) \otimes \max\left(\mathbf{R}^2 +
  \gamma\beta\psi_{\theta^{-}_i}(\mathbf{N}^2), \beta\psi_{\theta^{-}_i}(\mathbf{S}^2)\right) \\
  \cL_i(\theta_i) & = \rE_{\cD}\left[\mathbf{W}\otimes\left(\psi_{\theta_i}(\mathbf{S}^2) - \mathbf{T} \right)^2\right] \\
  \mathbf{T}^{\pi} & = (1 - \mathbf{I}) \left( \mathbf{R}^2 + \gamma\beta\psi^{\pi}_{\theta^{-}_i}(\mathbf{N}^2)\right) \\
  \cL^{\pi}_i(\theta_i) & = \rE_{\cD}\left[\left(\psi^{\pi}_{\theta_i}(\mathbf{S}^2) - \mathbf{T} \right)^2\right]
\end{align*}

\begin{figure}[!t]
  \centering
  \includegraphics[width=0.25\textwidth]{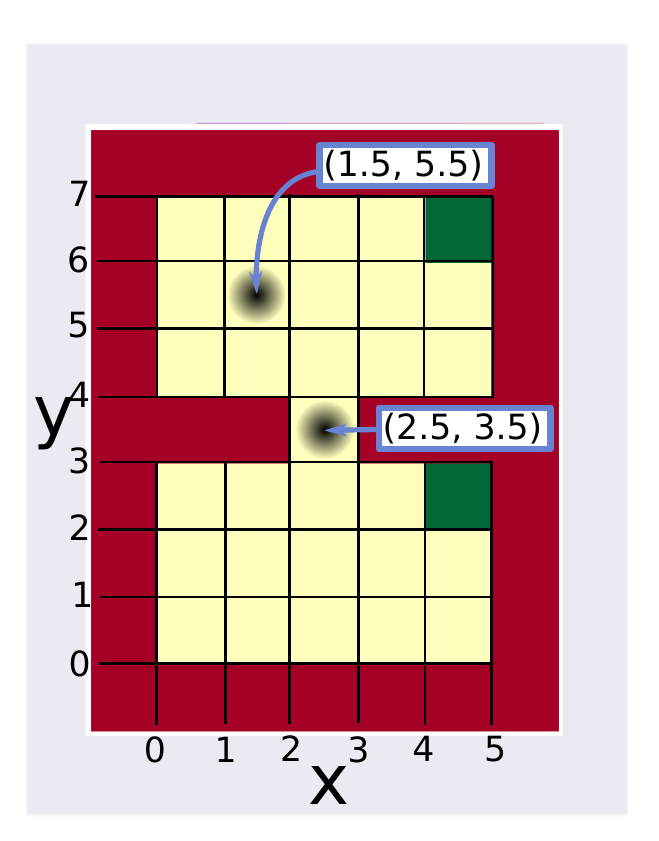}
  \caption{GridWorld and its (x, y) representation. Black densities
  illustrate the sampled states when adding noise.}
  \label{fig:mirrored_rooms}
\end{figure}

$\mathbf{R}^2$, $\mathbf{S}^2$, and $\mathbf{N}^2$ are batches of rewards,
states, and next-states, respectively, $\mathbf{W}$ is a mask used to enforce
action matching when approximating $d_{\sim}$, and $\mathbf{I}$ is the identity
matrix. $\psi(\mathbf{X})$ indicates applying $\psi$ to a matrix $\mathbf{X}$
elementwise, and $\otimes$ stands for the Hadamard product.  We multiply by $(1
- \mathbf{I})$ to zero out the diagonals, since those represent approximations
to $d^{(\pi)}_{\sim}(s, s)\equiv 0$.  The parameter $\beta$ is a stability
parameter that begins at $0$ and is incremented towards $1$ every $C$
iterations. Its purpose is to gradually ``grow'' the effective horizon of the
bisimulation backup and maximization. This is helpful since the approximant
$\psi_\theta$ can have some variance initially, depending on how $\theta$ is
initialized. Further, \citet{jiang15dependence} demonstrate that using shorter
horizons during planning can often be better than using the true horizon,
especially when using a model estimated from data.  Note that, in general, the
approximant $\psi$ is not guaranteed to be a proper pseudometric.  A lengthier
discussion, including the derivation of these matrices, is provided in the
supplemental material.

\begin{figure*}[!t]
  \begin{subfigure}[t]{0.5\textwidth}
    \centering\captionsetup{width=.8\linewidth}%
    \includegraphics[width=\textwidth]{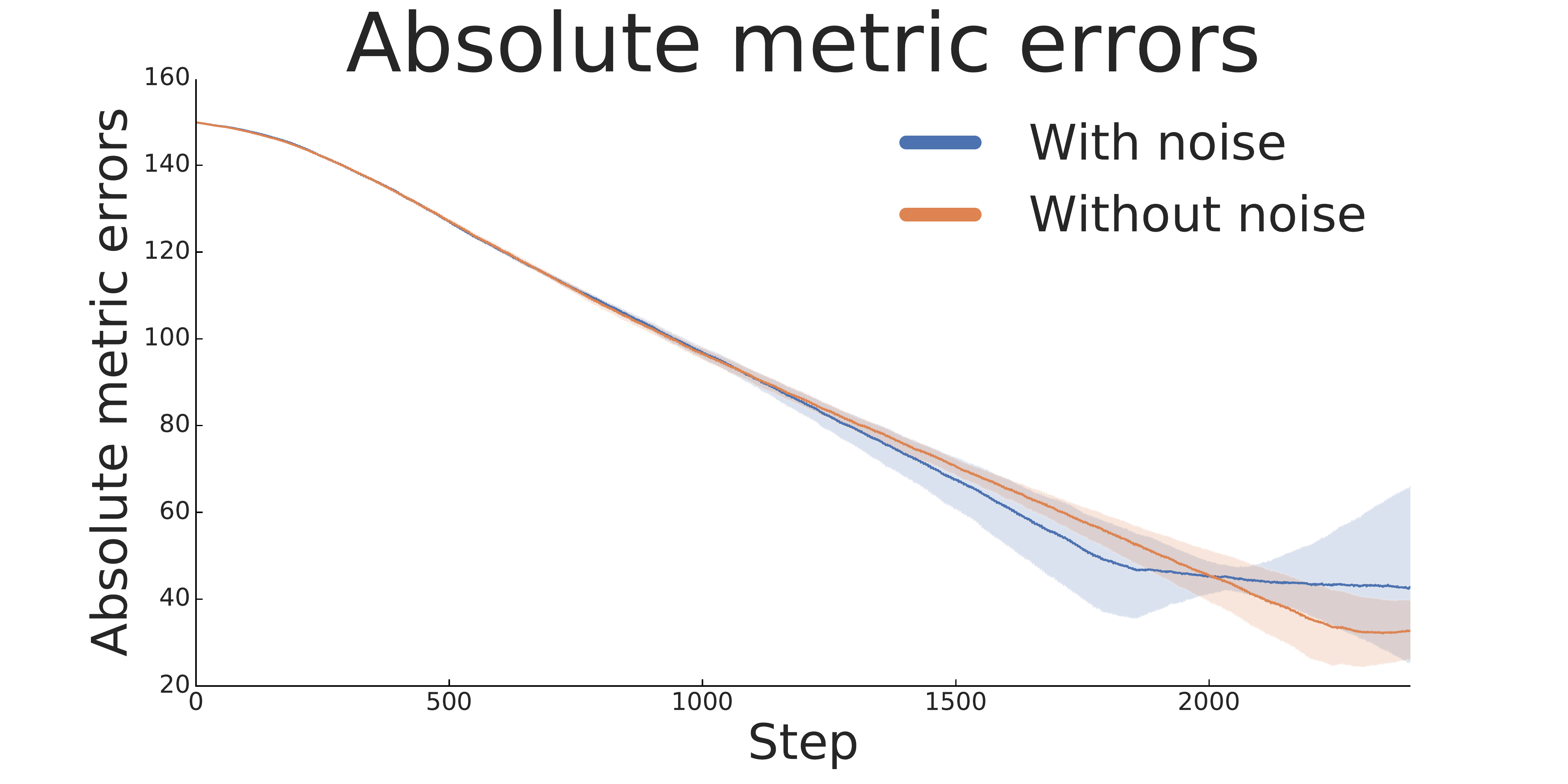}
    \caption{\label{fig:metric_errors}Absolute metric errors on the GridWorld.}
  \end{subfigure}
  \begin{subfigure}[t]{0.5\textwidth}
    \centering\captionsetup{width=.8\linewidth}%
    \includegraphics[width=\textwidth]{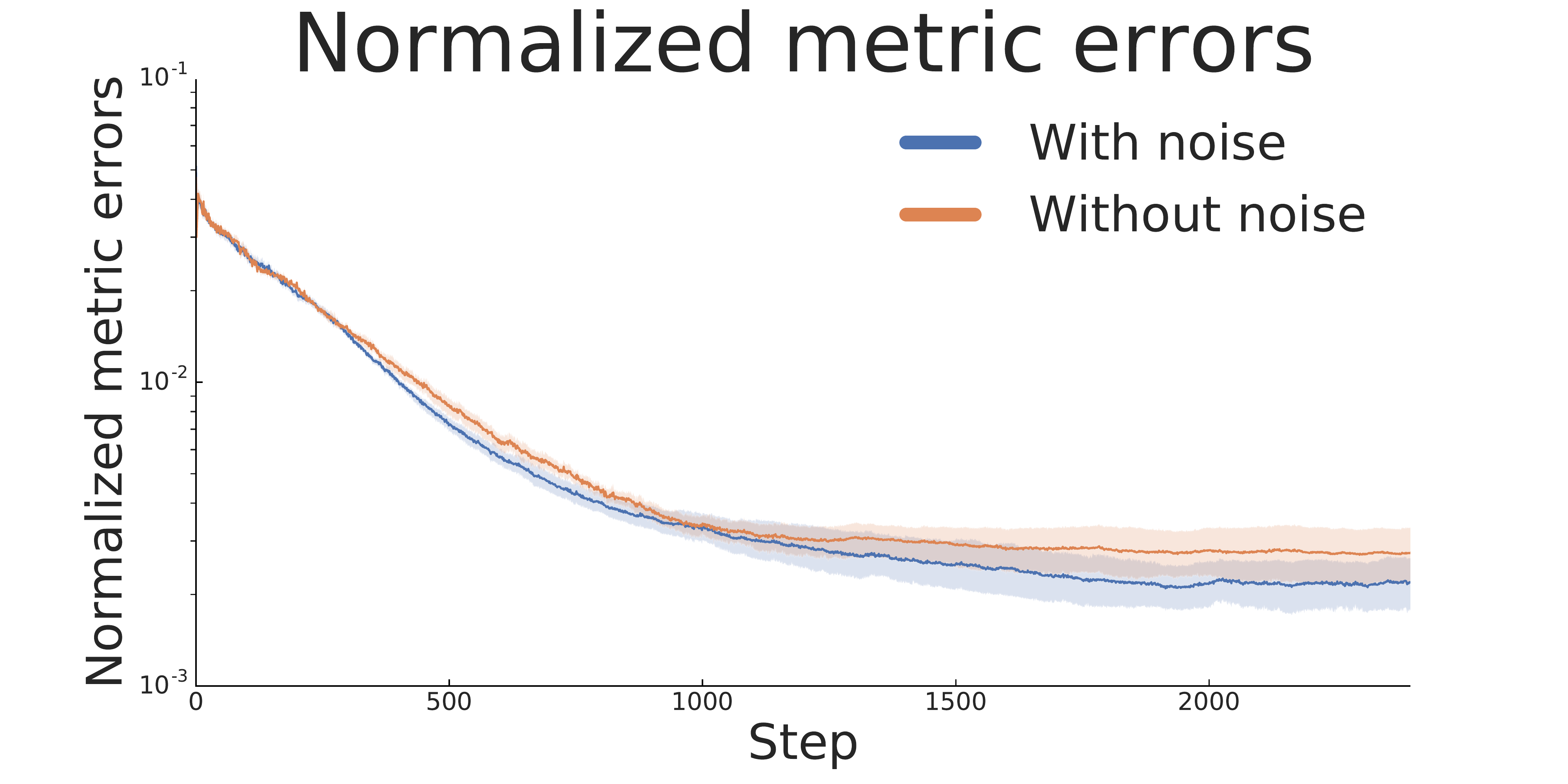}
    \caption{\label{fig:normalized_metric_errors}Normalized metric errors.}
  \end{subfigure}
  \caption{Metric errors for the learned metric on the GridWorld as training
  progresses.}
  \label{fig:errors}
\end{figure*}

\section{Empirical evaluation}
\label{sec:empirical}
In this section we provide empirical evidence for the effectiveness of our
bisimulation approximants\footnote{Code available at
https://github.com/google-research/google-research/tree/master/bisimulation\_aaai2020}.
We begin with a simple 31-state GridWorld, on which
we can compute the bisimulation metric exactly, and use a ``noisy''
representation which yields a continuous-state MDP. Having the exact metric
for the 31-state MDP allows us to quantitatively measure the quality of our
learned approximant.

We then learn a $\pi$-bisimulation approximant over policies generated by
reinforcement learning agents trained on Atari 2600 games.  In the supplemental
material we provide an extensive discussion of the hyperparameter search we
performed to find the settings used for both experiments. Training was done on
a Tesla P100 GPU.

\subsection{GridWorld}
\label{sec:grid_world}
We first evaluate our learning algorithms on the 31-state GridWorld environment
illustrated in \autoref{fig:mirrored_rooms}. There are 4 actions (up, down,
left, right) with deterministic transitions, and where an action driving the
agent towards a wall keeps the agent in the same cell. There is a single reward
of $+1.0$ received upon entering either of the green cells, and a reward of
$-1.0$ for taking an action towards a wall. We display the bisimulation
distances between all states in the supplemental, computed using the sampling
approach.

We represent each state by its coordinates $(x, y)\in\rR^2$, as illustrated in
\autoref{fig:mirrored_rooms}.  To estimate $d_{\sim}$ we use a network with an
input layer of dimension $4$, one fully connected hidden layer with $729$
hidden units, and an output of length $1$. The input is the concatenation of
two state representations, normalized to be in $[-1, 1]$, while the output
value is the estimate to $d_{\sim}$. We sampled state pairs and actions
uniformly randomly, and ran our experiments with $\gamma=0.99$, $C=500$,
$b=256$, and increased $\beta$ from $0$ to $1$ by a factor of $0.9$ every time
the target network was updated; we used the Adam optimizer \cite{kingma15adam}
with a learning rate of $0.01$. Because of the maximization term in the target,
these networks can have a tendency to overshoot (although the combination of
target networks and the $\beta$ term helps stabilize this); we ran the training
process for $2500$ steps, which, for this problem, was long enough to converge
to a reasonable solution before overshooting.  The full hyperparameter settings
are provided in the supplemental.

To evaluate the learning process, we measure the absolute error: $\|
d_{\sim} - \psi\|_{\infty}$ using the true underlying state space for which we
know the value of $d_{\sim}$.
Note that because there is no fixed learning target, absolute errors are not
guaranteed to be bounded. For this reason we also report the normalized error:
$\|\frac{d_{\sim}}{\|d_{\sim}\|_2} - \frac{\psi}{\|\psi\|_2}\|_{\infty}$ as in
practice one is mostly interested in relative, rather than absolute, distances.
\autoref{fig:metric_errors} and \autoref{fig:normalized_metric_errors} display
the results of our experiments over 10 independent runs; the shaded areas
represent the 95\% confidence interval.

In addition to training on the 31 state-MDP, we constructed a continuous
variant by adding Gaussian noise to the state representations; this noise is
centered at $(0, 0)$ with standard deviation $0.1$, and clipped to be in
$[-0.3, 0.3]$.  The per-cell noise is illustrated by the black gradients in
\autoref{fig:mirrored_rooms}.  As \autoref{fig:errors} shows, there is little
difference between learning the metric for the 31-state MDP versus learning it
for its continuous variant. Adding noise does not seem to hurt performance, and
in fact seems to be helpful. We hypotheisze that noise may be acting as a form
of regularization, but this requires further investigation.
In the supplemental material we include a figure
exploring using the metric approximant for aggregating states in the continuous
MDP setting with promising results.



\subsection{Atari 2600}
\label{sec:atari}
To evaluate the performance of our learning algorithm on an existing large
problem, we take a set of reinforcement learning (RL) agents trained on three
Atari 2600 games from the Arcade Learning Environment
\cite{bellemare13arcade}.  The RL agents were obtained from the set of trained
agents provided with the Dopamine library \cite{castro18dopamine}. Because our
methods are designed for deterministic MDPs, we only used those trained without
sticky actions\footnote{Sticky actions add
stochasticity to action outcomes.}~\cite{machado18revisiting} (evaluated in Section~4.3 in
\cite{castro18dopamine}); the trained checkpoints were provided for only three
games: Space Invaders, Pong, and Asterix. We used the Rainbow agent
\cite{hessel18rainbow} as it is the best performing of the provided Dopamine
agents. We used the penultimate layer of the trained Rainbow agent as the
representation $\phi$.

To approximate the on-policy bisimulation metric $d^{\pi}_{\sim}$ we loaded a
trained agent and ran it in evaluation mode for each respective game, filling
up the replay buffer while doing so.  Once 10,000 transitions have been stored
in the replay buffer, we begin sampling mini-batches and update our approximant
$\psi^{\pi}_{\theta}$ using the target and loss defined previously.  (note that
we still continue populating our replay buffer).  We ran our experiments using
a network of two hidden layers of dimension $16$, with
$\gamma=0.99$, $C=500$, $b=128$, and increased the $\beta$ term from $0$ to $1$
by a factor of $0.99$ every time the target network was updated.
We used the Adam optimizer
\cite{kingma15adam} with a learning rate of $7.5e^{-5}$ (except for Pong where
we found $0.001$ yielded better results). We trained the networks for around
$600K$ steps, although in practice we found that about half that many steps
were needed to reach a stable approximant.  The configuration file specifying
the full hyperparameter settings as well as the learning curves are provided in
the supplemental.

After training we evaluated our approximant $\psi^{\pi}_{\theta}([\phi(s),
\phi(t)])$ by fixing $s$ to be the first state in the game and varying $t$
throughout the episode; that is, we evaluate how similar the other states of
the game are to the initial state w.r.t. our metric approximant.  In
\autoref{fig:space_invaders_epoch} we display the first 500 steps of one
evaluation run on Space Invaders; as can be seen, the learned metric captures
more meaningful differences between frames (start of episodes, enemy alien
destroyed) that go beyond simple pixel differences.  Interestingly, when
sorting the frames by distance, the frames furthest away from $s$ are typically
those where the agent is about to be killed. It is worth noting that the way
states are encoded in Dopamine is by stacking the last four frames; in our
visualization we are only displaying the top frame in this stack. We observed
similar results for Asterix and Pong; we include these and more extensive
results, as well as videos for the three games, in the supplemental material.

\section{Related work}
\label{sec:related}
There are a number of different notions of state similarity that have been
studied over the years. \citet{li06towards} provide a unified characterization
and analysis of many of them. MDP-homomorphisms~\cite{ravindran03relativized}
do not require behavioral equivalence under the same action labels, and this
idea was extended to a metric by \citet{taylor09bounding}.

\begin{figure*}[!t]
  \centering
  \includegraphics[width=\textwidth]{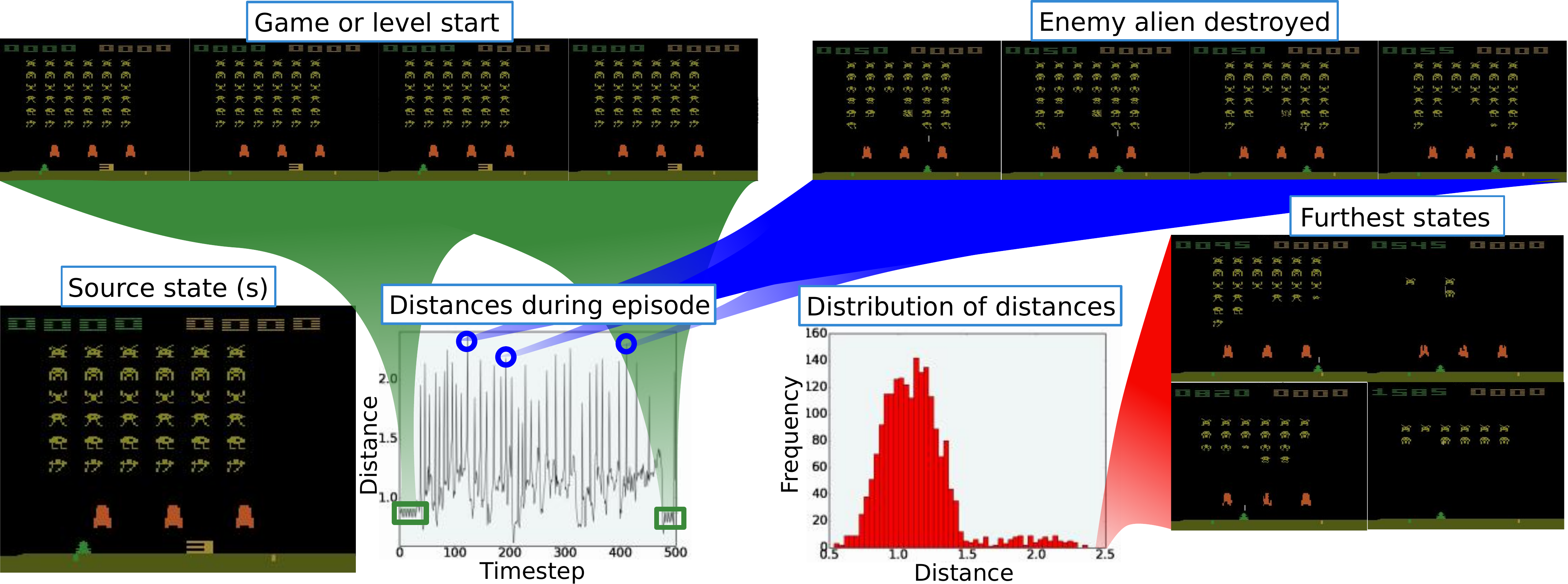}
  \caption{Evaluating the approximant to $d^{\pi}_{\sim}$ with an eval run on
  Space Invaders. We plot the distance between the source state $s$ (pictured in
  the bottom left) and every other state, highlighting the relatively low
  distances in game or level starts (green shading pointing to left and right
  side of the plot), as well as the peaks occuring when an enemy alien is
  destroyed (blue shading pointing to the distance peaks).  On the bottom right
  we display the distribution of distances and the four furthest states for
  this run.}
  \label{fig:space_invaders_epoch}
\end{figure*}

\citet{ferns06methods} introduced a sampling-based approximation to $d_{\sim}$
which exchanges the computation of the Wasserstein with an instance of the
assignment problem from network optimization; although \citet{castro11phd}
derived a PAC-bound for this approximant, the number of samples required is
still prohibitive.  \citet{bacci13computing} exploit the underlying structure
in $\cS$ to compute $d_{\sim}$.

Although there has been some work in concurrency theory to approximate large
systems via `polynomially accurate' simulations
\cite{segala07approximated}, they make no use of function
approximators in the form of neural networks.  We believe our use of neural
networks may grant our approach greater generalizability.

Deterministic on-policy systems can be reduced to a graph. As such, our notion
of $\pi$-bisimulation metrics bears a close relationship to graph similarity
measures \cite{zager08graph}. However, graph similarity notions compare two
full systems (graphs), as opposed to two nodes within a single graph, as we
evaluate here. Nonetheless, the relationship warrants further investigation,
which we leave for future work.

Perhaps most related to our sampling method is the on-the-fly
methods introduced by \citet{comanici12onthefly}. The authors replace the use of
standard dynamic programming in their computation with something akin to
asynchronous dynamic programming \cite{sutton98rl}, where not all state-pairs
are updated at each iteration, but rather $\cS$ is split into disjoint sets
that are updated at different intervals. A few strategies for sampling
state-pairs are discussed, of which the most similar to ours is the ``uniform
asynchronous update''.

\citet{gelada19deepmdp} introduced DeepMDP latent models and established a close
relationship with bisimulation metrics (specifically, Theorem 4 in their paper).
Although closely related, there are some important differences. Their work deals
with state representations, where the distance between states is their distance
in the representation space; by contrast, our proposed neural networks approximate
the bisimulation metric between two states, independent of their representation.
Further, the authors use the DeepMDP losses as an auxiliary task without a
guarantee that their representations are consistent with their theoretical
results. In our work we are able to show that our approximant is valid both
quantitatively (GridWorld) and qualitatively (Atari 2600). Nonetheless, a natural
extension of our work is to use the bisimulation losses we introduced as a
means to learn better representations.

\section{Conclusion}
\label{sec:conclusion}
We introduced new methods for computing and approximating
bisimulation metrics in large deterministic MDPs. The first is an {\em exact}
method that converges to the true metric asymptotically, and the second is a
differentiable method for approximating the metric which we demonstrated can
learn a good approximant even in continuous state spaces.  Since their
introduction, bisimulation metrics have been used for theoretical analysis or in
MDPs of small-to-moderate size, but they have scarcely been used in larger
problems. Our results open the doors for their application in systems with
a large, and even continuous, state space.

One important avenue for research is to extend these results to stochastic
systems. Computing the Wasserstein metric without access to a generative model
is challenging for deep RL environments, as the next-state samples typically
come from single trajectories in replay buffers. One possibility is to build a
model of the transition dynamics from the transitions in the replay buffer and
compute the Wasserstein metrics from this estimate.

Although the network architecture and hyperparameters used to train
$d^{\pi}_{\sim}$ are by no means optimal, the results we presented for the
Atari 2600 domain are very promising and suggest that bisimulation metrics can
be used effectively for deep reinforcement learning. Some promising areas we
are currently exploring are using bisimulation metrics as an auxiliary task for
improved state representation, as a mechanism for compressing replay buffers,
and as a tool for more efficient exploration.

\section{Acknowledgements}
The author would like to thank Marc G. Bellemare, Gheorghe Comanici, Marlos C.
Machado, Doina Precup, Carles Gelada, as well as the rest of the Google Brain
team in Montreal for helpful discussions. The author would also like to thank
the anonymous reviewers for their useful feedback while reviewing this work.

\bibliography{bisimulation}

\begin{thebibliography}{}

\bibitem[\protect\citeauthoryear{Abadi \bgroup et al\mbox.\egroup
  }{2015}]{abadi15tensorflow}
Abadi, M.; Agarwal, A.; Barham, P.; Brevdo, E.; Chen, Z.; Citro, C.; Corrado,
  G.~S.; Davis, A.; Dean, J.; Devin, M.; Ghemawat, S.; Goodfellow, I.; Harp,
  A.; Irving, G.; Isard, M.; Jia, Y.; Jozefowicz, R.; Kaiser, L.; Kudlur, M.;
  Levenberg, J.; Man\'{e}, D.; Monga, R.; Moore, S.; Murray, D.; Olah, C.;
  Schuster, M.; Shlens, J.; Steiner, B.; Sutskever, I.; Talwar, K.; Tucker, P.;
  Vanhoucke, V.; Vasudevan, V.; Vi\'{e}gas, F.; Vinyals, O.; Warden, P.;
  Wattenberg, M.; Wicke, M.; Yu, Y.; and Zheng, X.
\newblock 2015.
\newblock {TensorFlow}: Large-scale machine learning on heterogeneous systems.

\bibitem[\protect\citeauthoryear{Bacci \bgroup et al\mbox.\egroup
  }{2013a}]{bacci13computing}
Bacci, G.; Bacci, G.; Larsen, K.~G.; and Mardare, R.
\newblock 2013a.
\newblock Computing {B}ehavioral {D}istances, {C}ompositionally.
\newblock In {\em Mathematical Foundations of Computer Science 2013}.

\bibitem[\protect\citeauthoryear{Bacci \bgroup et al\mbox.\egroup
  }{2013b}]{bacci13onthefly}
Bacci, G.; Bacci, G.; Larsen, K.~G.; and Mardare, R.
\newblock 2013b.
\newblock On-the-{F}ly {E}xact {C}omputation of {B}isimilarity {D}istances.
\newblock In {\em Proceedings of the 19th Int. Conf. on Tools and Algorithms
  for the Construction and Analysis of Systems}.

\bibitem[\protect\citeauthoryear{Baier \bgroup et al\mbox.\egroup
  }{2006}]{baier06bisimulation}
Baier, C.; Hermanns, H.; Katoen, J.; and Wolf, V.
\newblock 2006.
\newblock Bisimulation and simulation relations for markov chains.
\newblock In {\em Essays on Algebraic Process Calculi}, number~10 in Electronic
  Notes in Theoretical Computer Science,  73--78.
\newblock Elsevier.

\bibitem[\protect\citeauthoryear{Bellemare \bgroup et al\mbox.\egroup
  }{2013}]{bellemare13arcade}
Bellemare, M.~G.; Naddaf, Y.; Veness, J.; and Bowling, M.
\newblock 2013.
\newblock The {A}rcade {L}earning {E}nvironment: {A}n evaluation platform for
  general agents.
\newblock {\em Jour. of AI Research} 47:253--279.

\bibitem[\protect\citeauthoryear{Bellman}{1957}]{bellman57dp}
Bellman, R.
\newblock 1957.
\newblock {\em Dynamic {P}rogramming}.
\newblock Princeton, NJ, USA: University Press.

\bibitem[\protect\citeauthoryear{Castro and Precup}{2010}]{castro10using}
Castro, P.~S., and Precup, D.
\newblock 2010.
\newblock Using bisimulation for policy transfer in {M}{D}{P}s.
\newblock In {\em Proceedings of the 9th International Conference on Autonomous
  Agents and Multiagent Systems (AAMAS-2010)}.

\bibitem[\protect\citeauthoryear{Castro \bgroup et al\mbox.\egroup
  }{2018}]{castro18dopamine}
Castro, P.~S.; Moitra, S.; Gelada, C.; Kumar, S.; and Bellemare, M.~G.
\newblock 2018.
\newblock Dopamine: {A} {R}esearch {F}ramework for {D}eep {R}einforcement
  {L}earning (arxiv.org/abs/1812.06110).

\bibitem[\protect\citeauthoryear{Castro}{2011}]{castro11phd}
Castro, P.~S.
\newblock 2011.
\newblock {\em On planning, prediction and knowledge transfer in {F}ully and
  {P}artially {O}bservable {M}arkov {D}ecision {P}rocesses}.
\newblock Ph.D. Dissertation, McGill {U}niversity.

\bibitem[\protect\citeauthoryear{Chen, van Breugel, and
  Worrell}{2012}]{chen12complexity}
Chen, D.; van Breugel, F.; and Worrell, J.
\newblock 2012.
\newblock On the {C}omplexity of {C}omputing {P}robabilistic {B}isimilarity.
\newblock In {\em Found. of Software Science and Computational Structures}.

\bibitem[\protect\citeauthoryear{Comanici, Panangaden, and
  Precup}{2012}]{comanici12onthefly}
Comanici, G.; Panangaden, P.; and Precup, D.
\newblock 2012.
\newblock On-the-{F}ly {A}lgorithms for {B}isimulation {M}etrics.
\newblock In {\em Proc. of the 9th Int. Conf. on Quantitative Evaluation of
  Systems}.

\bibitem[\protect\citeauthoryear{Desharnais \bgroup et al\mbox.\egroup
  }{1999}]{desharnais99metrics}
Desharnais, J.; Gupta, V.; Jagadeesan, R.; and Panangaden, P.
\newblock 1999.
\newblock Metrics for {L}abeled {M}arkov {S}ystems.
\newblock In {\em CONCUR'99 Concurrency Theory},  258--273.

\bibitem[\protect\citeauthoryear{Ferns and Precup}{2014}]{ferns14bisimulation}
Ferns, N., and Precup, D.
\newblock 2014.
\newblock Bisimulation {M}etrics are {O}ptimal {V}alue {F}unctions.
\newblock In {\em Proceedings of the 30th Conference on Uncertainty in
  Artificial Intelligence}.

\bibitem[\protect\citeauthoryear{Ferns \bgroup et al\mbox.\egroup
  }{2006}]{ferns06methods}
Ferns, N.; Castro, P.~S.; Precup, D.; and Panangaden, P.
\newblock 2006.
\newblock Methods for computing state similarity in {M}arkov decision
  processes.
\newblock In {\em Proceedings of the 22nd Conference on Uncertainty in
  Artificial Intelligence}, UAI '06.

\bibitem[\protect\citeauthoryear{Ferns, Panangaden, and
  Precup}{2004}]{ferns04metrics}
Ferns, N.; Panangaden, P.; and Precup, D.
\newblock 2004.
\newblock Metrics for {F}inite {M}arkov {D}ecision {P}rocesses.
\newblock In {\em Proceedings of the 20th Conference on Uncertainty in
  Artificial Intelligence}.

\bibitem[\protect\citeauthoryear{Ferrer~Fioriti \bgroup et al\mbox.\egroup
  }{2016}]{fioriti16deciding}
Ferrer~Fioriti, L.~M.; Hashemi, V.; Hermanns, H.; and Turrini, A.
\newblock 2016.
\newblock Deciding probabilistic automata weak bisimulation: Theory and
  practice.
\newblock {\em Form. Asp. Comput.} 28(1):109--143.

\bibitem[\protect\citeauthoryear{Gelada \bgroup et al\mbox.\egroup
  }{2019}]{gelada19deepmdp}
Gelada, C.; Kumar, S.; Buckman, J.; Nachum, O.; and Bellemare, M.~G.
\newblock 2019.
\newblock Deep{MDP}: {L}earning {C}ontinuous {L}atent {S}pace {M}odels for
  {R}epresentation {L}earning.
\newblock In {\em Proceedings of the 36th International Conference on Machine
  Learning}.

\bibitem[\protect\citeauthoryear{Givan, Dean, and
  Greig}{2003}]{givan03equivalence}
Givan, R.; Dean, T.; and Greig, M.
\newblock 2003.
\newblock Equivalence notions and model minimization in {M}arkov decision
  processes.
\newblock {\em Artificial Intelligence} 147:163--223.

\bibitem[\protect\citeauthoryear{Hessel \bgroup et al\mbox.\egroup
  }{2018}]{hessel18rainbow}
Hessel, M.; Modayil, J.; van Hasselt, H.; Schaul, T.; Ostrovski, G.; Dabney,
  W.; Horgan, D.; Piot, B.; Azar, M.; and Silver, D.
\newblock 2018.
\newblock Rainbow: {C}ombining {I}mprovements in {D}eep {R}einforcement
  learning.
\newblock In {\em Proceedings of the AAAI Conference on Artificial
  Intelligence}.

\bibitem[\protect\citeauthoryear{Jiang \bgroup et al\mbox.\egroup
  }{2015}]{jiang15dependence}
Jiang, N.; Kulesza, A.; Singh, S.; and Lewis, R.
\newblock 2015.
\newblock The {D}ependence of {E}ffective {P}lanning {H}orizon on {M}odel
  {A}ccuracy.
\newblock In {\em Proceedings of the 14th International Conference on
  Autonomous Agents and Multiagent Systems (AAMAS-15)}.

\bibitem[\protect\citeauthoryear{Katoen \bgroup et al\mbox.\egroup
  }{2007}]{katoen07bisimulation}
Katoen, J.-P.; Kemna, T.; Zapreev, I.; and Jansen, D.~N.
\newblock 2007.
\newblock Bisimulation minimisation mostly speeds up probabilistic model
  checking.
\newblock In {\em Tools and Algorithms for the Construction and Analysis of
  Systems},  87--101.

\bibitem[\protect\citeauthoryear{Kingma and Ba}{2015}]{kingma15adam}
Kingma, D.~P., and Ba, J.
\newblock 2015.
\newblock Adam: A method for stochastic optimization.
\newblock In {\em Proceedings of the International Conference on Learning
  Representations}.

\bibitem[\protect\citeauthoryear{Li, Walsh, and Littman}{2006}]{li06towards}
Li, L.; Walsh, T.~J.; and Littman, M.~L.
\newblock 2006.
\newblock Towards a unified theory of state abstraction for mdps.
\newblock In {\em Proceedings of the Ninth International Symposium on
  Artificial Intelligence and Mathematics},  531--539.

\bibitem[\protect\citeauthoryear{Machado \bgroup et al\mbox.\egroup
  }{2018}]{machado18revisiting}
Machado, M.~C.; Bellemare, M.~G.; Talvitie, E.; Veness, J.; Hausknecht, M.; and
  Bowling, M.
\newblock 2018.
\newblock Revisiting the {A}rcade {L}earning {E}nvironment: Evaluation
  protocols and open problems for general agents.
\newblock {\em Journal of AI Research}.

\bibitem[\protect\citeauthoryear{Mnih \bgroup et al\mbox.\egroup
  }{2015}]{mnih15humanlevel}
Mnih, V.; Kavukcuoglu, K.; Silver, D.; Rusu, A.~A.; Veness, J.; Bellemare,
  M.~G.; Graves, A.; Riedmiller, M.; Fidjeland, A.~K.; Ostrovski, G.; Petersen,
  S.; Beattie, C.; Sadik, A.; Antonoglou, I.; King, H.; Kumaran, D.; Wierstra,
  D.; Legg, S.; and Hassabis, D.
\newblock 2015.
\newblock Human-level control through deep reinforcement learning.
\newblock {\em Nature}.

\bibitem[\protect\citeauthoryear{Puterman}{1994}]{puterman94mdp}
Puterman, M.~L.
\newblock 1994.
\newblock {\em Markov Decision Processes: Discrete Stochastic Dynamic
  Programming}.
\newblock New York, NY, USA: John Wiley \& Sons, Inc., 1st edition.

\bibitem[\protect\citeauthoryear{Ravindran and
  Barto}{2003}]{ravindran03relativized}
Ravindran, B., and Barto, A.~G.
\newblock 2003.
\newblock Relativized {O}ptions: {C}hoosing the {R}ight {T}ransformation.
\newblock In {\em Proceedings of the 20th International Conference on Machine
  Learning}.

\bibitem[\protect\citeauthoryear{Ruan \bgroup et al\mbox.\egroup
  }{2015}]{ruan15representation}
Ruan, S.; Comanici, G.; Panangaden, P.; and Precup, D.
\newblock 2015.
\newblock Representation discovery for mdps using bisimulation metrics.
\newblock In {\em AAAI Conference on Artificial Intelligence}.

\bibitem[\protect\citeauthoryear{Sangiorgi}{2011}]{sangiorgi11introduction}
Sangiorgi, D.
\newblock 2011.
\newblock {\em Introduction to Bisimulation and Coinduction}.
\newblock Cambridge University Press.

\bibitem[\protect\citeauthoryear{Santara \bgroup et al\mbox.\egroup
  }{2019}]{santara19extra}
Santara, A.; Madan, R.; Ravindran, B.; and Mitra, P.
\newblock 2019.
\newblock Extra: Transfer-guided exploration.
\newblock {\em CoRR} abs/1906.11785.

\bibitem[\protect\citeauthoryear{Segala and
  Turrini}{2007}]{segala07approximated}
Segala, R., and Turrini, A.
\newblock 2007.
\newblock Approximated computationally bounded simulation relations for
  probabilistic automata.
\newblock In {\em Proceedings of the 20th IEEE Computer Security Foundations
  Symposium (CSF'07)}.

\bibitem[\protect\citeauthoryear{Sutton and Barto}{1998}]{sutton98rl}
Sutton, R.~S., and Barto, A.~G.
\newblock 1998.
\newblock {\em Introduction to Reinforcement Learning}.
\newblock Cambridge, MA, USA: MIT Press.

\bibitem[\protect\citeauthoryear{Taylor, Precup, and
  Panagaden}{2009}]{taylor09bounding}
Taylor, J.; Precup, D.; and Panagaden, P.
\newblock 2009.
\newblock Bounding {P}erformance {L}oss in {A}pproximate {M}{D}{P}
  {H}omomorphisms.
\newblock In {\em Advances in Neural Information Processing Systems 21}.
\newblock  1649--1656.

\bibitem[\protect\citeauthoryear{Villani}{2008}]{villani08optimal}
Villani, C.
\newblock 2008.
\newblock {\em Optimal Transport}.
\newblock Springer-Verlag Berlin Heidelberg.

\bibitem[\protect\citeauthoryear{Zager and Verghese}{2008}]{zager08graph}
Zager, L.~A., and Verghese, G.~C.
\newblock 2008.
\newblock Graph similarity scoring and matching.
\newblock {\em Applied Mathematics Letters} 21(1):86 -- 94.

\end{thebibliography}
\bibliographystyle{aaai}

\pagebreak

\onecolumn
\appendix

\section{Proofs of the theoretical results}
\pibisimop*
\begin{proof}
  This proof mimics the proof of Theorem~4.5 from \cite{ferns04metrics}
  (included as Theorem~1 in this paper). Lemma~4.1 from that paper
  holds under Definition~3 by definition. We make use of the same
  pointwise ordering on $\cM$: $d\leq d'$ iff $d(s, t)\leq d'(s, t)$ for all
  $s,t\in\cS$, which gives us an $\omega$-cpo with bottom $\bot$, which is the
  everywhere-zero metric. Since Lemma~4.4 from \cite{ferns04metrics} ($\cW$ is
  continuous) also applies in our definition, it only remains to show that
  $\cF^{\pi}$ is continuous:
  \begin{align*}
    \cF^{\pi}(\bigsqcup_{n\in\rN}\lbrace x_n\rbrace)(s, t) & = |\cR^{\pi}_s -
    \cR^{\pi}_t | + \gamma\cW\left(\bigsqcup_{n\in\rN}\lbrace
    x_n\rbrace\right)(\cP^{\pi}_s, \cP^{\pi}_t) \\ & = |\cR^{\pi}_s - \cR^{\pi}_t | +
    \gamma\sup_{n\in\rN}\cW(x_n)(\cP^{\pi}_s, \cP^{\pi}_t) \quad\textrm{by
    continuity of }\cW \\ & = \sup_{n\in\rN}\left(|\cR^{\pi}_s - \cR^{\pi}_t |
    + \gamma\cW(x_n)(\cP^{\pi}_s, \cP^{\pi}_t)\right) \\ & =
    \sup_{n\in\rN}\left\lbrace\cF^{\pi}(x_n)(s, t)\right\rbrace \\ & =
    \left(\bigsqcup_{n\in\rN}\left\lbrace\cF^{\pi}(x_n)\right\rbrace\right)(s,
    t)
  \end{align*}
  The rest of the proof follows in the same way as in \cite{ferns04metrics}.
\end{proof}

\pibisimbound*
\begin{proof}
  We will use the standard value function update: $V^{\pi}_n(s) = \cR^{\pi}_s +
  \gamma\sum_{s'\in\cS}\cP^{\pi}_s(s')V^{\pi}_{n-1}(s')$ with $V^{\pi}_0\equiv 0$
  and our update operator from \autoref{thm:pi_bisim_operator}: $d^{\pi}_{n}(s,
  t) = \cF^{\pi}(d^{\pi}_{n-1})(s, t)$ with $d^{\pi}_0\equiv 0$, and prove this
  by induction, showing that for all $n\in\rN$ and $s,t\in\cS$,
  $|V^{\pi}_n(s)-V^{\pi}_n(t)|\leq d^{\pi}_n(s, t)$.

  The base case holds vacuously: $0 = V^{\pi}_0(s, t) \leq d^{\pi}_0(s, t) = 0$, so assume true
  up to $n$.
  \begin{align*}
    |V^{\pi}_{n+1}(s) - V^{\pi}_{n+1}(t)| & = \left|\cR^{\pi}_s +
    \gamma\sum_{s'\in\cS}\cP^{\pi}_s(s')V^{\pi}_n(s') - \left(\cR^{\pi}_t +
    \gamma\sum_{s'\in\cS}\cP^{\pi}_t(s')V^{\pi}_n(s')\right)\right| \\
    & \leq |\cR^{\pi}_s - \cR^{\pi}_t| + \left|\gamma\sum_{s'\in\cS}V^{\pi}_n(s')(\cP^{\pi}_s(s') - \cP^{\pi}_t(s'))\right| \\
    & \leq |\cR^{\pi}_s - \cR^{\pi}_t| + \left|\gamma\cW(d^{\pi}_n)(\cP^{\pi}_s, \cP^{\pi}_t)\right| \\
    & = |\cR^{\pi}_s - \cR^{\pi}_t| + \gamma\cW(d^{\pi}_n)(\cP^{\pi}_s, \cP^{\pi}_t) \quad\textrm{since $\cW(d^{\pi}_n)$ is a metric} \\
    & = \cF^{\pi}(d^{\pi}_n)(s, t) \\
    & = d^{\pi}_{n+1}(s, t)
  \end{align*}
  where the second inequality follows from noticing that, by induction, for all
  $s,t\in\cS$, $V^{\pi}_n(s) - V^{\pi}_n(t)\leq d_n(s, t)$, which means $V^{\pi}$ is a feasible solution
  to the primal LP objective of $\cW(d^{\pi}_n)(\cP^{\pi}_s, \cP^{\pi}_t)$ (see Equation~1).
\end{proof}

\kantisd*
\begin{proof}
  The primal LP defined in Equation~1 can be expressed in its
  dual form (which, incidentally, is a minimum-cost flow problem):
  \begin{align*}
    & \min_{\boldsymbol{\lambda}}\sum_{s',t'\in\cS}\lambda_{s',t'}d(s', t') \\
    \textrm{s.t.}\quad &  \forall s'\in\cS, \quad \sum_{t'}\lambda_{s', t'} = \cP(s, a)(s') \\
                  &  \forall t'\in\cS, \quad \sum_{s'}\lambda_{s', t'} = \cP(t, a)(t') \\
                  &  \boldsymbol{\lambda} \geq 0
  \end{align*}
  By the deterministic assumption it follows that $\lambda_{s', t'} = 0$ whenever
  $s'\ne\cN(s, a)$ or $t'\ne\cN(t, a)$ (since otherwise one of the first two
  constraints will be violated). This means that only
  $\lambda_{\cN(s, a), \cN(t, a)}$ is positive. By the equality constraints
  it then follows that $\lambda_{\cN(s, a), \cN(t, a)}=1$, resulting in
  $d(\cN(s, a), \cN(t, a))$ as the minimal objective value.
\end{proof}

\begin{lemma}
  \label{lemma:max_act_exists}
  (Maximizing action) If $d_0\equiv 0$ and subsequent $d_n$ are udpated as in
  Equation~3, then for all $s,t\in\cS$ and $\delta\in\rR$
  there exists an $n<\infty$ and $a\in\cA$ such that\\
  $Pr(d_n(s, t) = |\cR(s, a) - \cR(t, a)| + \gamma d_{n-1}(\cN(s, a), \cN(t, a))) > 1-\delta$.
\end{lemma}
\begin{proof}
  Since we start at $d_0\equiv 0$, the result will hold as long as we can
  guarantee that $\tau_{s,t,a}$ for some $a\in\cA$ will be sampled at least
  once. By assumption $\cD(\tau_{s,t,a}) > 0$, which means that the probability
  that $\tau_{s,t,a}$ is {\em not} sampled by step $n$ is
  $(1-\cD(\tau_{s,t,a}))^n$.  We obtain our result by taking $n >
  \frac{\ln\delta}{\ln(1-\cD(\tau_{s,t,a}))}$.
\end{proof}

\begin{lemma}
  \label{lemma:monotonic}
  (Monotonicity) If $d_0\equiv 0$ and subsequent $d_n$ are udpated as in Equation~3, then $d_n\leq d_{\sim}$ for all $n\in\rN$.
\end{lemma}
\begin{proof}
  Obviously $d_{n-1}\leq d_n$ for all $n$. We will show by induction that
  $d_n\leq d_{\sim}$ for all $n$. The base case $d_0\equiv 0\leq d_{\sim}$
  follows by definition, so assume true up to $n$ and consider any $s,t\in\cS$,
  where $n$ is large enough so that we have a high likelihood of sampling
  $\tau_{s,t,a}$ for some $a\in\cA$ (from \autoref{lemma:max_act_exists}).
  First note that this implies that for all $n'>n$ there exists an
  $a^*_{n'}\in\cA$ such that $d_{n'}(s, t) = |\cR(s, a^*_{n'}) - \cR(t,
  a^*_{n'})| + \gamma d_{n'-1}(\cN(s, a^*_{n'}), \cN(t, a^*_{n'}))$.

  Define $a^*_{\sim} = \arg\max_{a\in\cA}\left(|\cR(s, a) - \cR(t, a)| + \gamma d_{\sim}(\cN(s, a), \cN(t, a))\right)$.
  We then have:
  \begin{align*}
    d_{\sim}(s, t) & = |\cR(s, a^*_{\sim}) - \cR(t, a^*_{\sim})| + \gamma d_{\sim}(\cN(s, a^*_{\sim}), \cN(t, a^*_{\sim})) \\
                   & \geq |\cR(s, a^*_{n}) - \cR(t, a^*_{n})| + \gamma d_{\sim}(\cN(s, a^*_{n}), \cN(t, a^*_{n})) \\
                   & \geq |\cR(s, a^*_{n}) - \cR(t, a^*_{n})| + \gamma d_{n}(\cN(s, a^*_{n}), \cN(t, a^*_{n})) \\
                   & = d_{n+1}(s, t)
  \end{align*}
  where the second line follows from the fact that $a^*_{\sim}$ is the action
  that maximizes the bisimulation distance, and the third line follows from the
  inductive hypothesis.
\end{proof}

\onlinebisim*
\begin{proof}
  To prove this we will look at the difference \\
  $\| d_{\sim} - d_n\|_{\infty} = \max_{s,t\in\cS}|d_{\sim}(s, t) - d_n(s, t)|$.
  For any $\delta\in\rR$ define
  $n^*=\max_{s,t\in\cS,a\in\cA}\left(\frac{\ln\delta}{\ln(1-\cD(\tau_{s,t,a}))}\right)$.
  For some $n>n^*$, let $s,t$ be the state-pair that maximizes
  $\| d_{\sim} - d_n\|_{\infty}$ at time $n$.
  \small
  \begin{align*}
    \| d_{\sim} - d_n\|_{\infty} = & | d_{\sim}(s, t) - d_n(s, t) | \\
          = & \enspace d_{\sim}(s, t) - d_n(s, t) \qquad\qquad\textrm{by \autoref{lemma:monotonic}} \\
          = & \enspace |\cR(s, a^*_{\sim}) - \cR(t, a^*_{\sim})| + \gamma d_{\sim}(\cN(s, a^*_{\sim}), \cN(t, a^*_{\sim})) - \\
            & \enspace \left(|\cR(s, a^*_n) - \cR(t, a^*_n)| + \gamma d_{n-1}(\cN(s, a^*_n), \cN(t, a^*_n))\right) \qquad\textrm{w.p. at least } 1-\delta\textrm{ by~\autoref{lemma:max_act_exists}} \\
       \leq & \enspace |\cR(s, a^*_{\sim}) - \cR(t, a^*_{\sim})| + \gamma d_{\sim}(\cN(s, a^*_{\sim}), \cN(t, a^*_{\sim})) \\
          - & \enspace \left(|\cR(s, a^*_{\sim}) - \cR(t, a^*_{\sim})| + \gamma d_{n-1}(\cN(s, a^*_{\sim}), \cN(t, a^*_{\sim}))\right) \\
          = & \enspace \gamma(d_{\sim}(\cN(s, a^*_{\sim}), \cN(t, a^*_{\sim})) - d_{n-1}(\cN(s, a^*_{\sim}), \cN(t, a^*_{\sim}))) \\
       \leq & \enspace \gamma\|d_{\sim} - d_{n-1} \|_{\infty}
  \end{align*}
  \normalsize
  Where the first inequality follows from the fact that $a^*_n$ is the action
  that maximizes $d_n(s, t)$.  Thus we have that the sequence
  $\lbrace\|d_{\sim} - d_n\|_{\infty}\rbrace$ is a contraction. By the Banach
  fixed-point theorem, the result follows.
\end{proof}

We note that the last two results are related to Lemma~3 and Theorem~2 by
\cite{comanici12onthefly}, but there are some important differences worth
highlighting, notably in their use of the approximate update function
$\hat{h_k}$, which is not required in our method. Indeed, Lemma~3 in \cite{comanici12onthefly}
is more a statement on what's required of $\hat{h_k}$ to guarantee $h_k\leq d^*$ for all $k$;
specifically, that in order for $h_k\leq d^*$, it is required that
$\hat{h_k}\leq d^*$ for all $k'\leq k$.
In contrast, our \autoref{lemma:monotonic} does not need this requirement as we make
no use of an approximate udpate. Further, our proof of \autoref{lemma:monotonic} relies
on \autoref{lemma:max_act_exists} to guarantee that we can rewrite $d_n$ as an $\cF$ update
for a sufficiently large $n$. Note that this is rather different than the use of a similar
idea with $\nu(k)$ by \cite{comanici12onthefly}, as they use it in their proof
of Lemma 4, which lower-bounds $h_{k_m}$ with $F^m(0)$.

Although the statements of both Theorem 2 by \cite{comanici12onthefly} and our
\autoref{thm:online_bisim} are similar (both methods converge to the true metric),
the approach we take is quite different. \cite{comanici12onthefly} construct
their update via a partition of the state space into 3 sets ($\alpha_k$, $\beta_k$,
$\delta_k$) at each step. The proper handling of these 3 sets, and in particular
of $\delta_k$, makes the proof of Lemma 4 rather involved, which the authors require to
obtain the lower-bound which is then used for the proof of Theorem 2. This extra complication
is somewhat unfortunate, as the authors do not use $\delta_k$ sets in any of their
empirical evaluations, nor do they provide any indication as to what a good choice
of the approximate update function $\hat{h_k}$ would be.

In contrast, the proof of our \autoref{thm:online_bisim} requires only
\autoref{lemma:max_act_exists} and \autoref{lemma:monotonic}. We believe our proof
is much simpler to follow and makes use of fewer techniques.

\subsection{Mini-batch target and loss}
We specify a set of matrix operations for computing them on a {\em batch} of
transitions.  This allows us to efficiently train this approximant using
specialized hardware like GPUs.  The discussion in this section is specific to
approximating $d_{\sim}$, but it is straightforward to adapt it to
approximating $d^{\pi}_{\sim}$.  We provide code for both approximants in the
supplemental material with their implementation in
TensorFlow~\cite{abadi15tensorflow}, as well as implementations of the
algorithms discussed in the previous section.

At each step we assume access to a batch of $b$ samples of states, actions,
rewards, and next states:
\begin{align*}
  \mathbf{S} = 
  \begin{bmatrix}
    \phi(s_1) \\
    \phi(s_2) \\
    \cdots \\
    \phi(s_b)
  \end{bmatrix}
  , \mathbf{A} = 
  \begin{bmatrix}
    a_1 \\
    a_2 \\
    \cdots \\
    a_b
  \end{bmatrix}
  , \mathbf{R} = 
  \begin{bmatrix}
    \cR(s_1, a_1) \\
    \cR(s_2, a_2) \\
    \cdots \\
    \cR(s_b, a_b)
  \end{bmatrix},
  \mathbf{N} = 
  \begin{bmatrix}
    \phi(\cN(s_1, a_1)) \\
    \phi(\cN(s_2, a_2)) \\
    \cdots \\
    \phi(\cN(s_b, a_b))
  \end{bmatrix}
\end{align*}

Letting $[X, Y]$ stand for the concatenation of two vectors $X$ and $Y$, from
$\textbf{S}$ we construct a new square matrix of dimension $b\times b$ as
follows:
\begin{align*}
  \mathbf{S}^2 = 
  \begin{bmatrix}
    [\phi(s_1),  \phi(s_1)], [\phi(s_1), \phi(s_2)], \cdots, [\phi(s_1), \phi(s_b)] \\
    [\phi(s_2), \phi(s_1)], [\phi(s_2), \phi(s_2)], \cdots, [\phi(s_2), \phi(s_b)] \\
    \cdots \\
    [\phi(s_b), \phi(s_1)], [\phi(s_b), \phi(s_2)], \cdots, [\phi(s_b), \phi(s_b)] \\
  \end{bmatrix}
\end{align*}

Each element in this matrix is a vector of dimension $2k$. We reshape this
matrix to be a ``single-column'' tensor of length $b^2$. We can perform a
similar exercise on the reward and next-state batches:

\begin{align*}
  \mathbf{R}^2 =
  \begin{bmatrix}
    |\cR(s_1, a_1) - \cR(s_1, a_1)| (= 0) \\
    |\cR(s_1, a_1) - \cR(s_2, a_2)| \\
    \cdots, \\
    |\cR(s_1, a_1) - \cR(s_b, a_b)| \\
    |\cR(s_2, a_2) - \cR(s_1, a_1)| \\
    \cdots \\
    |\cR(s_b, a_b) - \cR(s_{b-1}, a_{b-1})| \\
    |\cR(s_b, a_b) - \cR(s_b, a_b)| (= 0) \\
  \end{bmatrix}
  \mathbf{N}^2 =
  \begin{bmatrix}
    [\phi(\cN(s_1, a_1)), \phi(\cN(s_1, a_1))], \\
    [\phi(\cN(s_1, a_1)), \phi(\cN(s_2, a_2))], \\
    \cdots, \\
    [\phi(\cN(s_b, a_b)), \phi(\cN(s_{b-1}, a_{b-1}))] \\
    [\phi(\cN(s_b, a_b)), \phi(\cN(s_b, a_b))] \\
  \end{bmatrix}
\end{align*}

Finally, we define a mask which enforces that we only consider pairs of samples
that have matching actions:
\begin{align*}
  \mathbf{W} =
  \begin{bmatrix}
    a_1 == a_1 \\
    a_1 == a_2 \\
    \cdots \\
    a_b == a_{b-1} \\
    a_b == a_b
  \end{bmatrix}
\end{align*}

In batch-form, the target defined above becomes:
\small
\begin{align}
  \label{eqn:target}
  \mathbf{T} = (1 - \mathbf{I}) * \max\left(\mathbf{R}^2 + \gamma\beta\psi_{\theta^{-}_i}(\mathbf{N}^2),\quad \beta\psi_{\theta^{-}_i}(\mathbf{S}^2)\right)
\end{align}
\normalsize

where $\psi(\mathbf{X})$ indicates applying $\psi$ to a matrix $\mathbf{X}$
elementwise.  We multiply by $(1 - \mathbf{I})$ to zero out the diagonals,
since those represent approximations to $d_{\sim}(s, s)\equiv 0$.  The
parameter $\beta$ is a stability parameter that begins at $0$ and is
incremented every $C$ iterations. Its purpose is to gradually ``grow'' the
effective horizon of the bisimulation backup and maximization. This is
necessary since the approximant $\psi_\theta$ can have some variance initially,
depending on how $\theta$ is initialized. Further, \cite{jiang15dependence}
demonstrate that using shorter horizons during planning can often be better
than using the true horizon, especially when using a model estimated from data.

Finally, our loss $\cL_i$ at iteration $i$ is defined as
$\cL_i(\theta_i) = \rE_{\cD}\left[\mathbf{W}\otimes\left(\psi_{\theta_i}(\mathbf{S}^2) - \mathbf{T} \right)^2\right]$,
where $\otimes$ stands for the Hadamard product.
Note that, in general, the approximant $\psi$ is not a metric: it can violate
the identity of indiscernibles, symmetry, and subadditivity conditions.

\pagebreak

\section{Bisimulation distances between all states in the GridWorld} In
Figure~\ref{fig:bisim_distances} we display the bisimulation distances from
all states in the GridWorld MDP illustrated in Figure~4 in
the main paper. These were computed using the sampling approach of
section~5. Note how $d_{\sim}$ is able to capture similarities
that go beyond simple physical proximity. This is most evident when examining
the distance from the ``hallway'' state to all other states: even though it
neighbours the bottom row in the top room, that row is furthest according to
$d_{\sim}$.

\begin{figure*}[h]
  \begin{subfigure}[t]{0.18\textwidth}
    \centering\captionsetup{width=.8\linewidth}%
    \includegraphics[width=\textwidth]{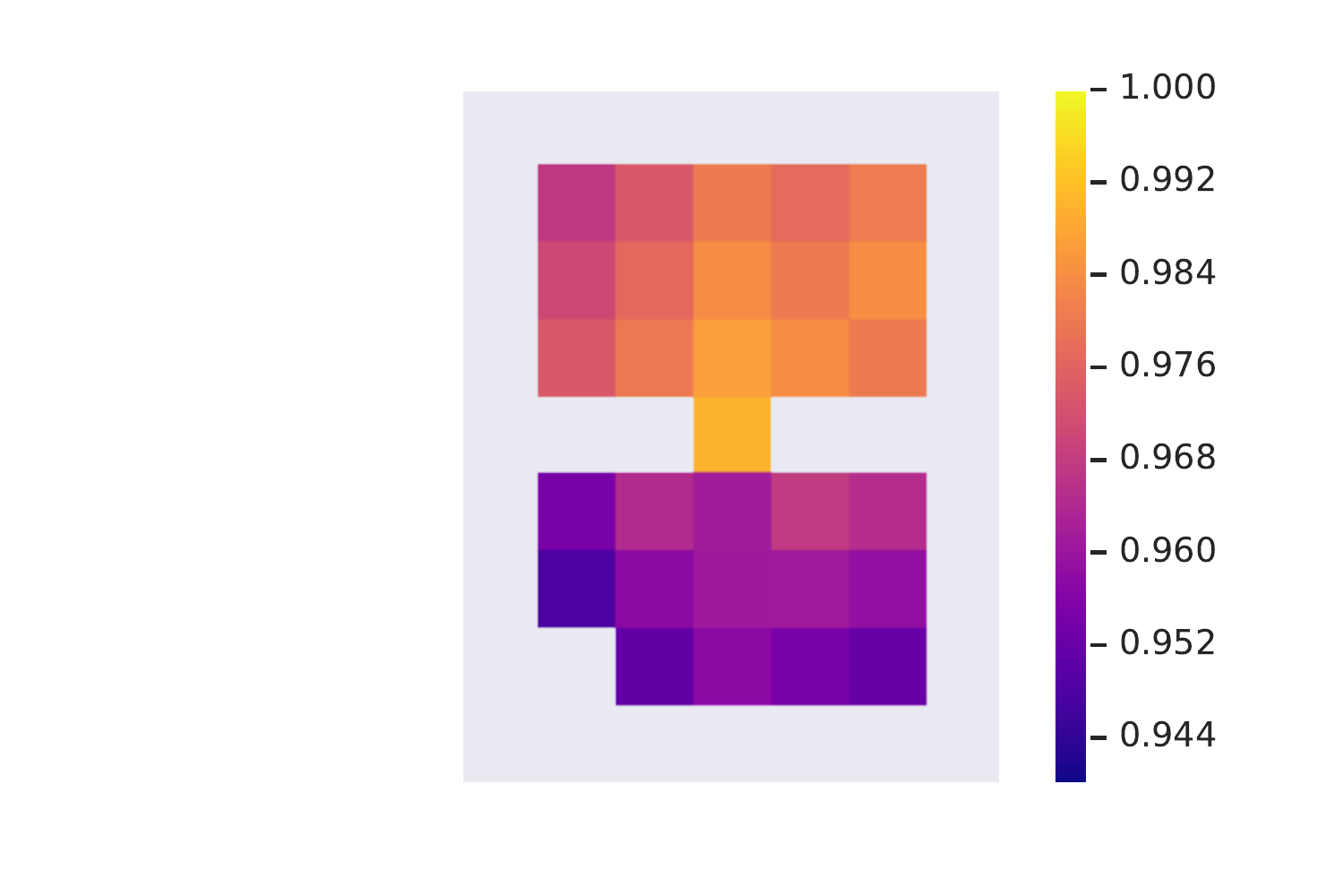}
  \end{subfigure}
  \begin{subfigure}[t]{0.18\textwidth}
    \centering\captionsetup{width=.8\linewidth}%
    \includegraphics[width=\textwidth]{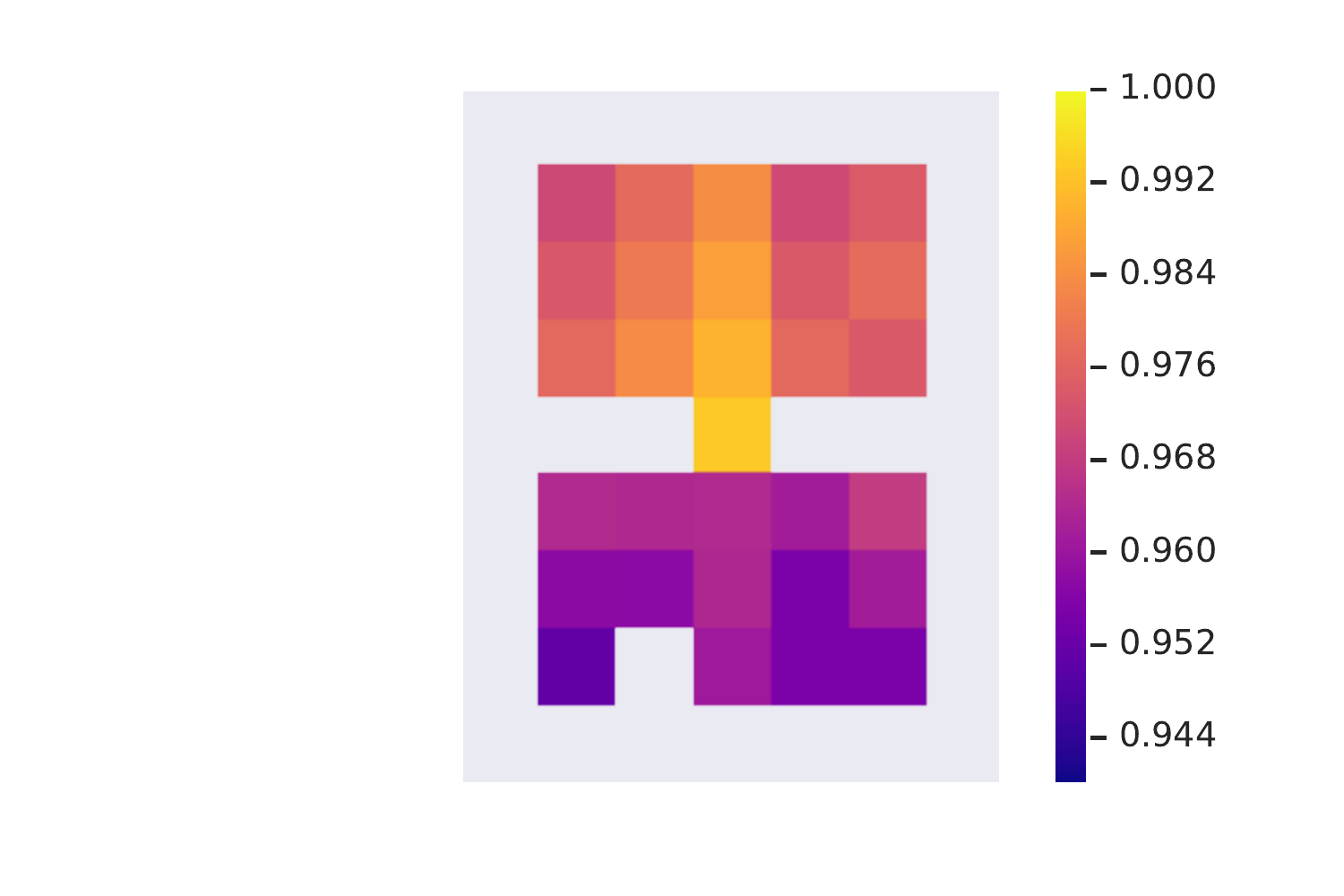}
  \end{subfigure}
  \begin{subfigure}[t]{0.18\textwidth}
    \centering\captionsetup{width=.8\linewidth}%
    \includegraphics[width=\textwidth]{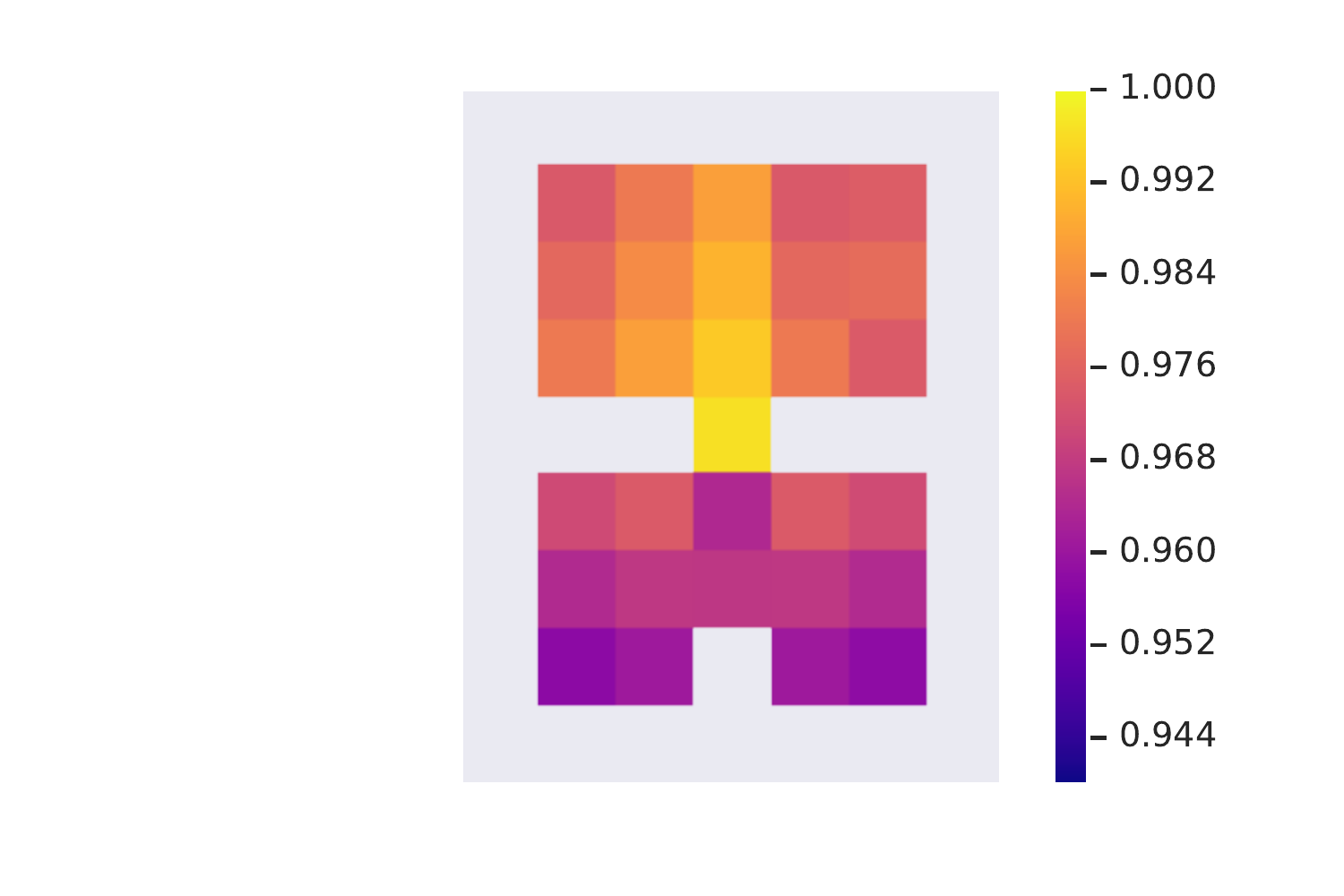}
  \end{subfigure}
  \begin{subfigure}[t]{0.18\textwidth}
    \centering\captionsetup{width=.8\linewidth}%
    \includegraphics[width=\textwidth]{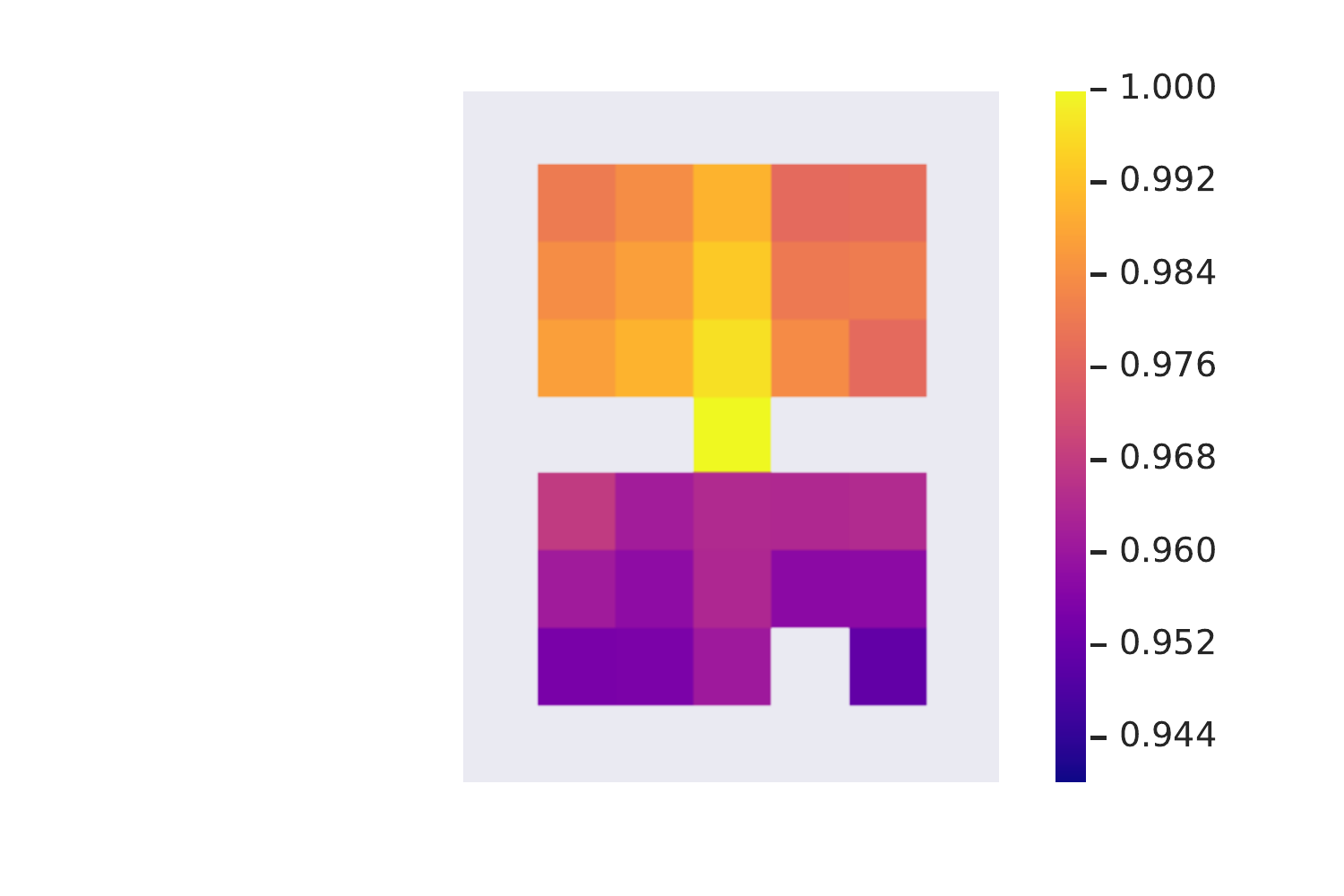}
  \end{subfigure}
  \begin{subfigure}[t]{0.18\textwidth}
    \centering\captionsetup{width=.8\linewidth}%
    \includegraphics[width=\textwidth]{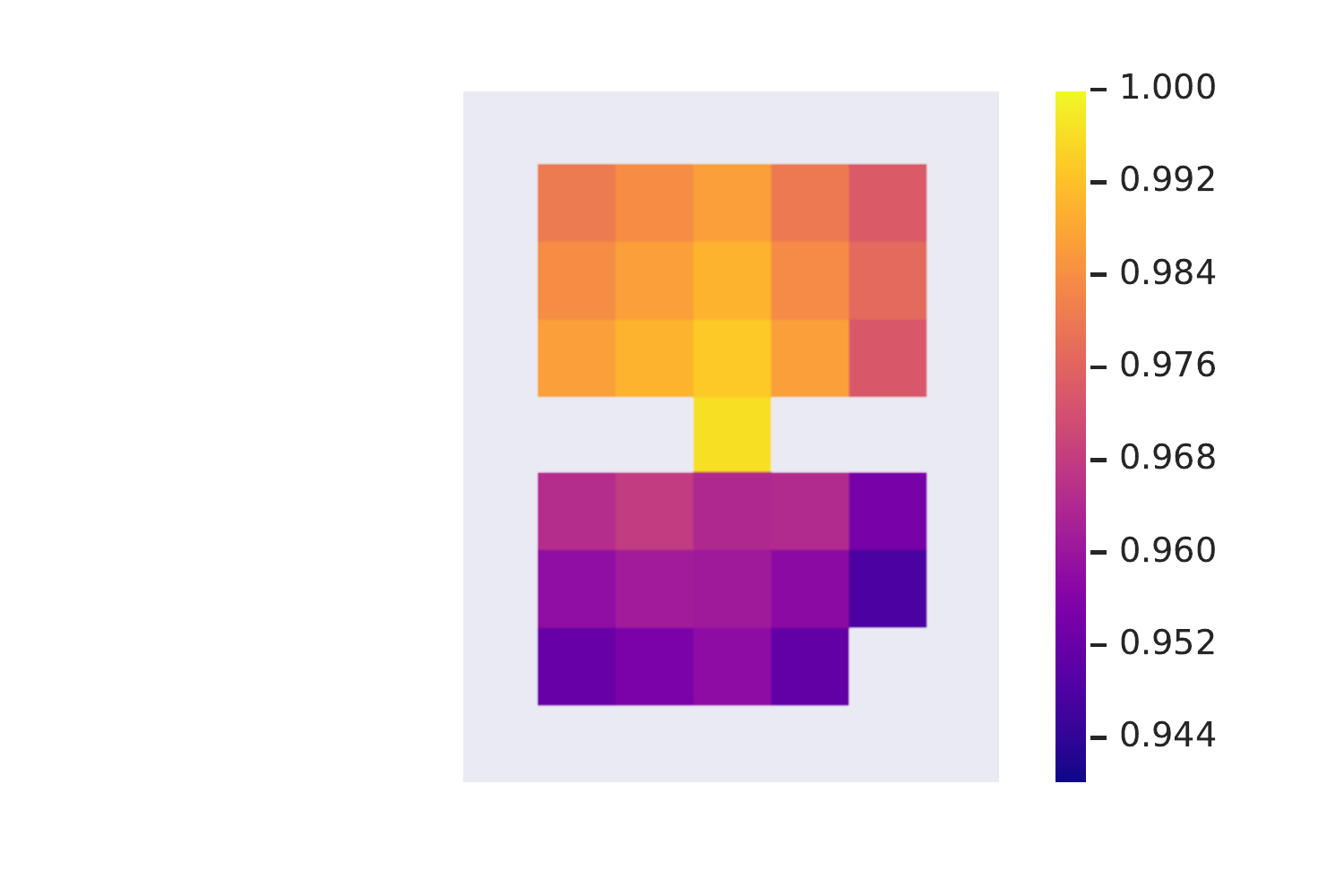}
  \end{subfigure}
  \begin{subfigure}[t]{0.18\textwidth}
    \centering\captionsetup{width=.8\linewidth}%
    \includegraphics[width=\textwidth]{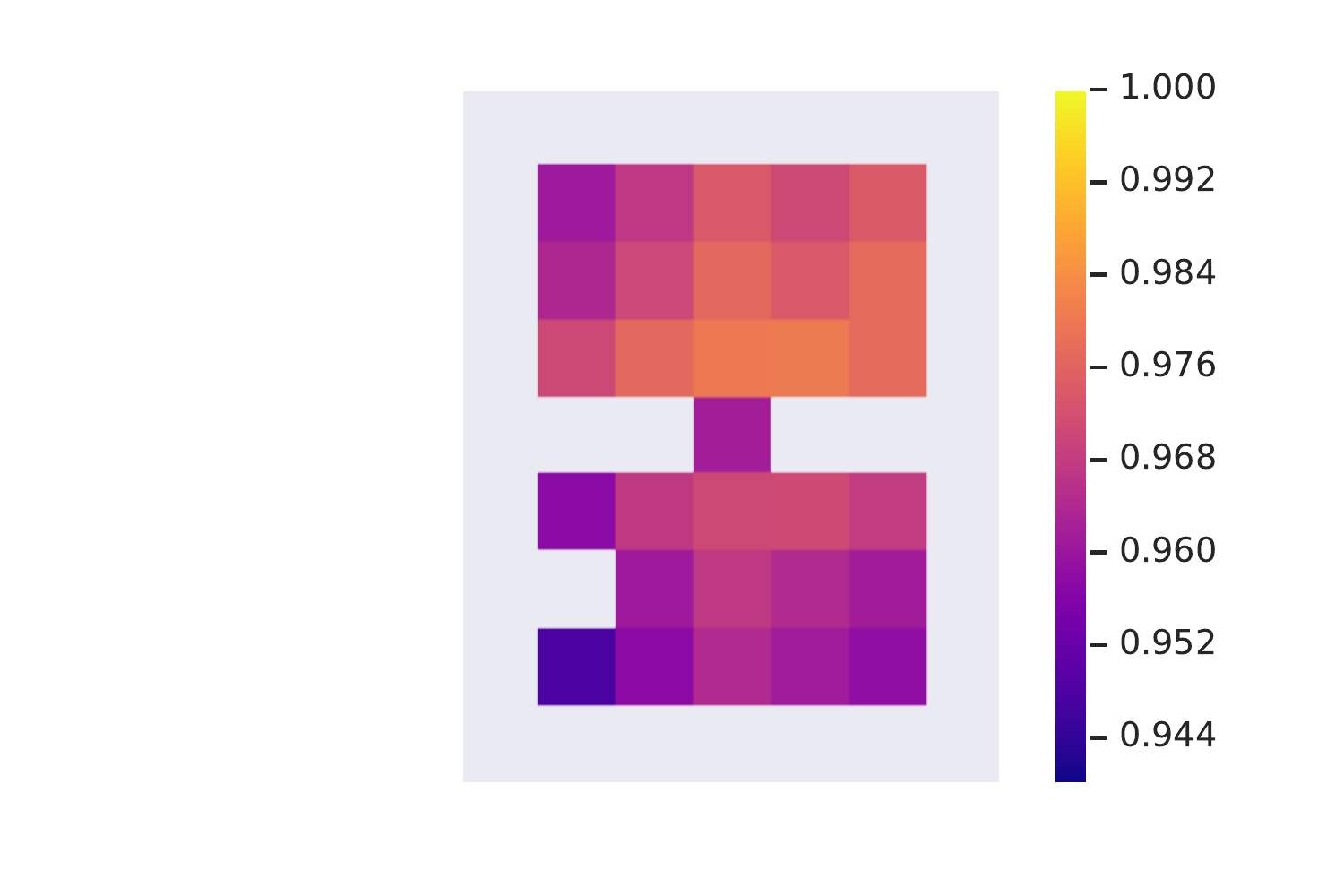}
  \end{subfigure}
  \begin{subfigure}[t]{0.18\textwidth}
    \centering\captionsetup{width=.8\linewidth}%
    \includegraphics[width=\textwidth]{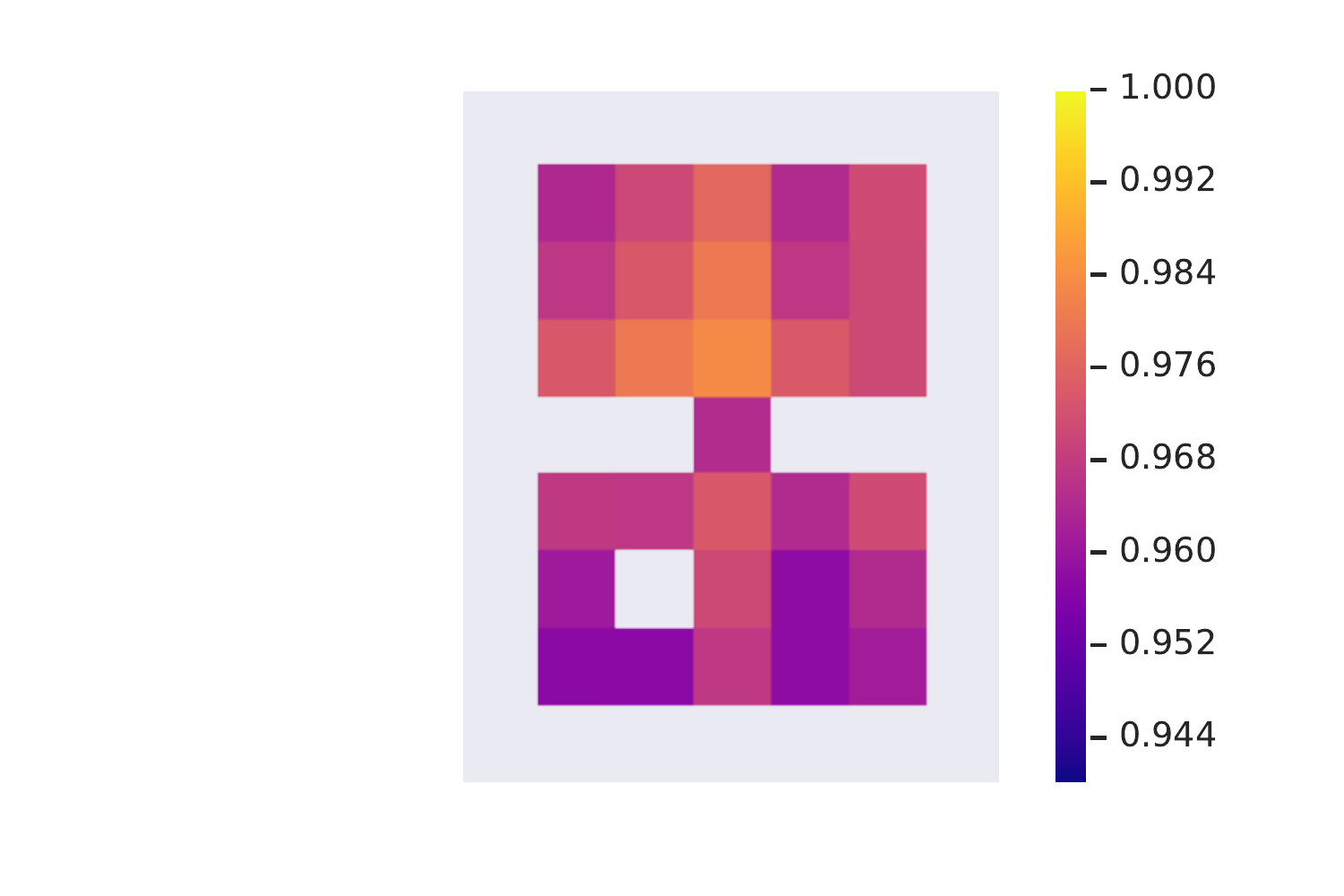}
  \end{subfigure}
  \begin{subfigure}[t]{0.18\textwidth}
    \centering\captionsetup{width=.8\linewidth}%
    \includegraphics[width=\textwidth]{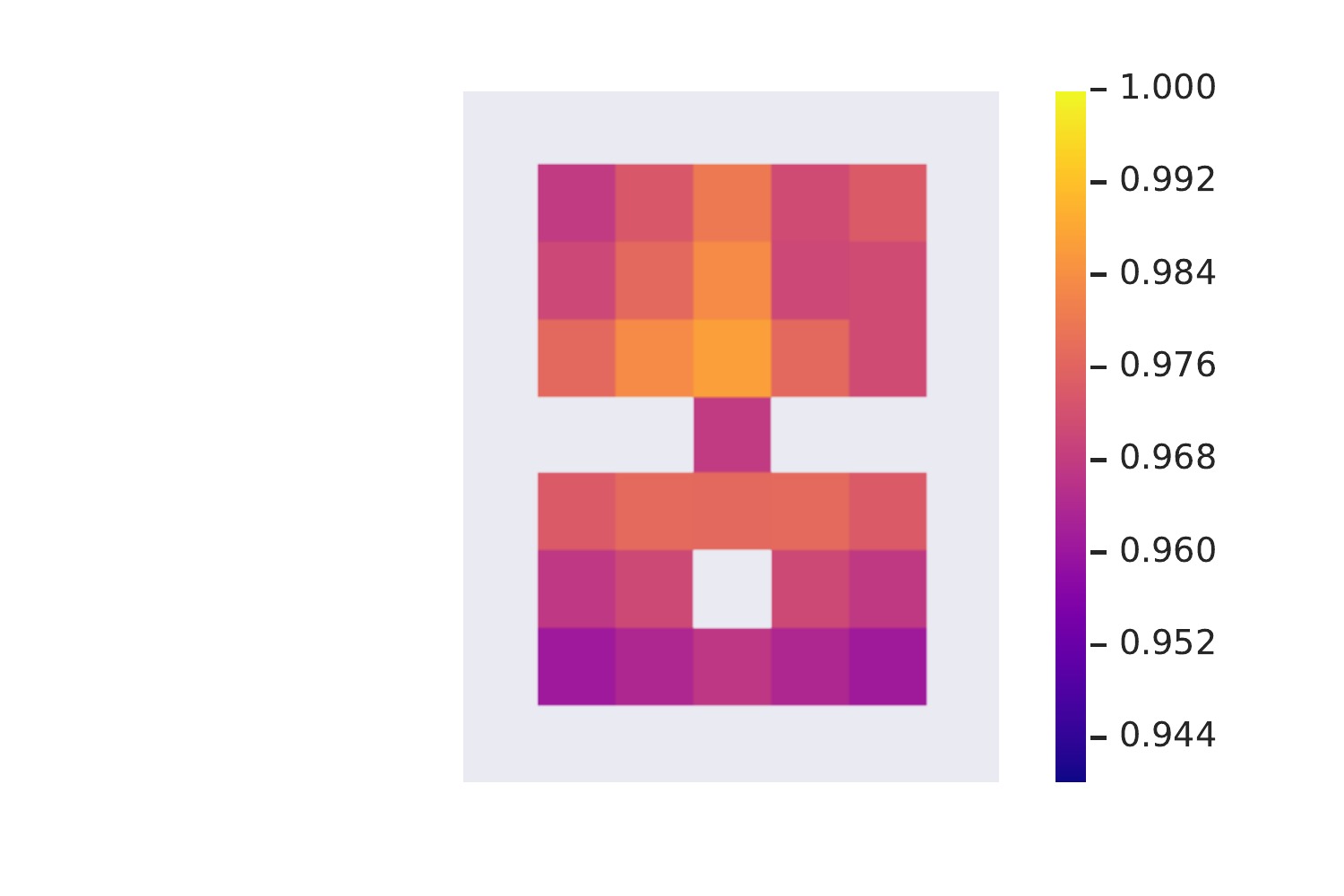}
  \end{subfigure}
  \begin{subfigure}[t]{0.18\textwidth}
    \centering\captionsetup{width=.8\linewidth}%
    \includegraphics[width=\textwidth]{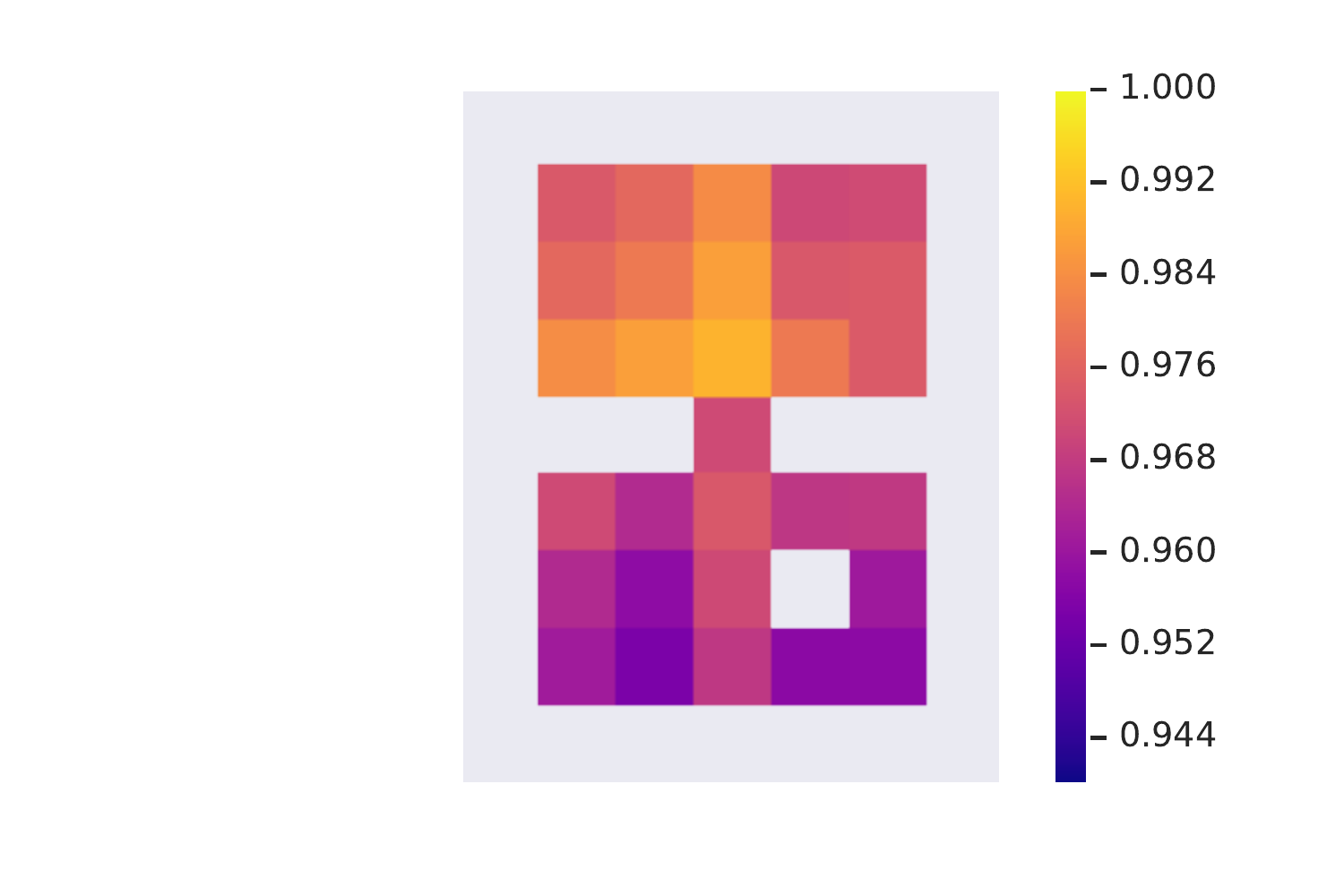}
  \end{subfigure}
  \begin{subfigure}[t]{0.18\textwidth}
    \centering\captionsetup{width=.8\linewidth}%
    \includegraphics[width=\textwidth]{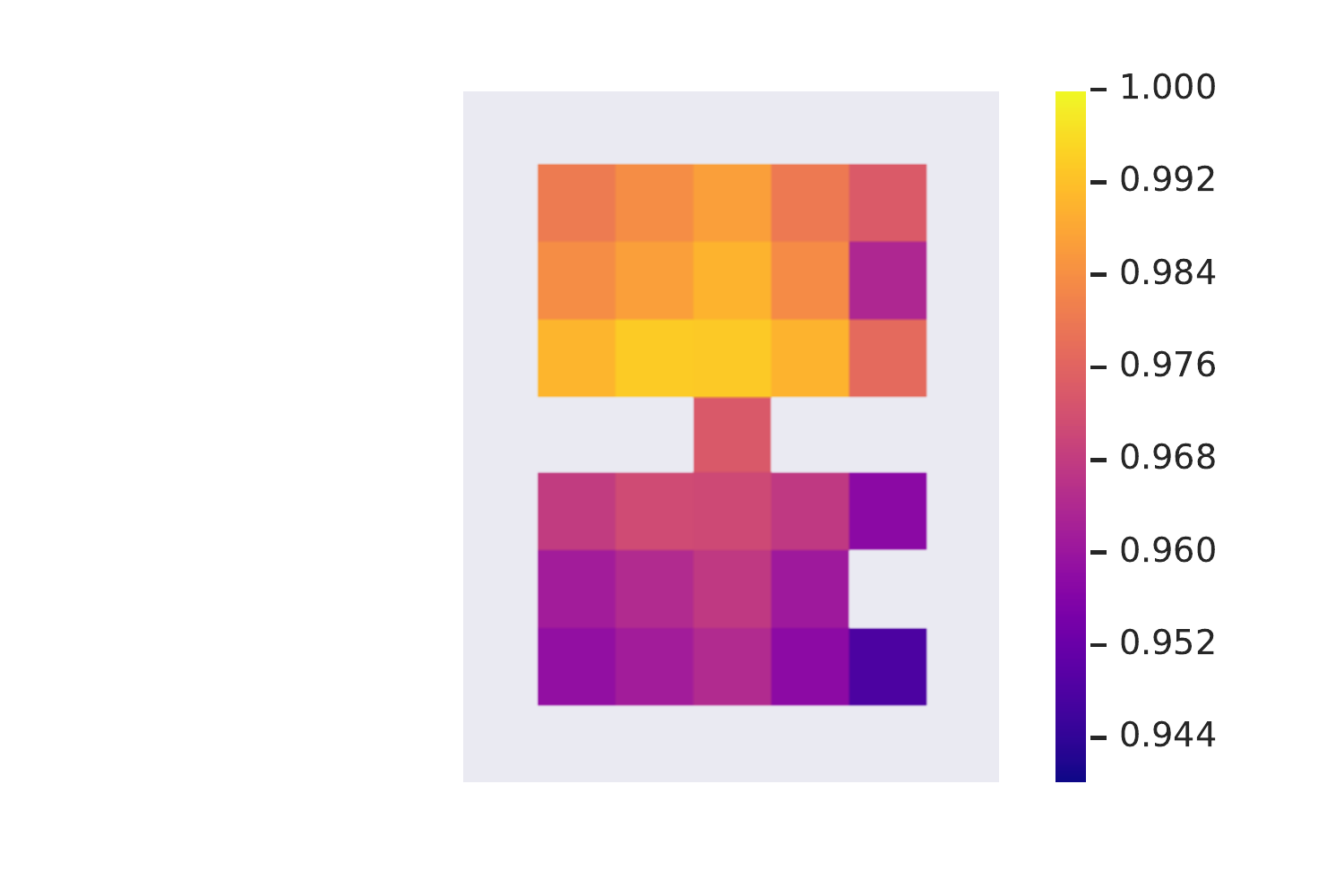}
  \end{subfigure}
  \begin{subfigure}[t]{0.18\textwidth}
    \centering\captionsetup{width=.8\linewidth}%
    \includegraphics[width=\textwidth]{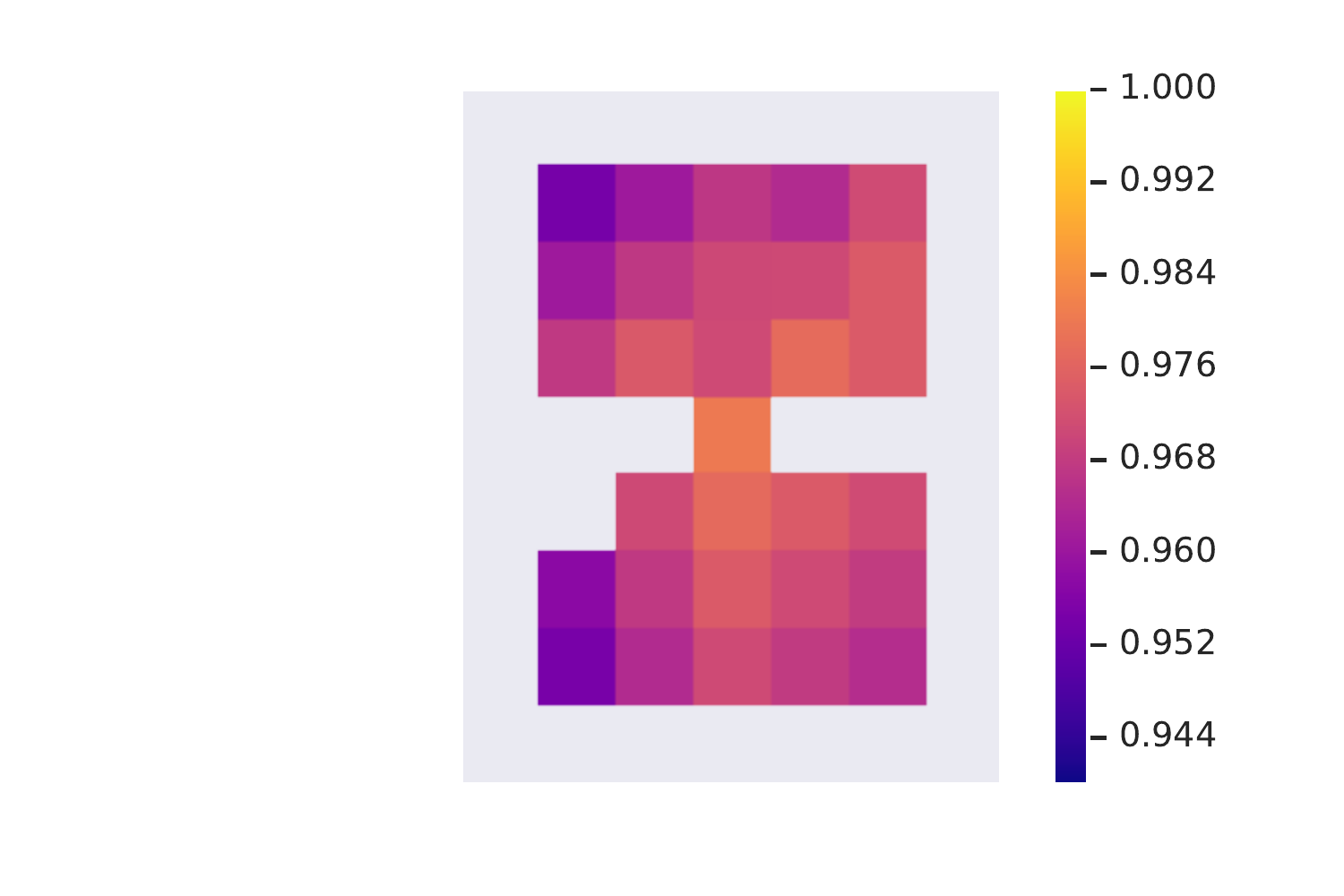}
  \end{subfigure}
  \begin{subfigure}[t]{0.18\textwidth}
    \centering\captionsetup{width=.8\linewidth}%
    \includegraphics[width=\textwidth]{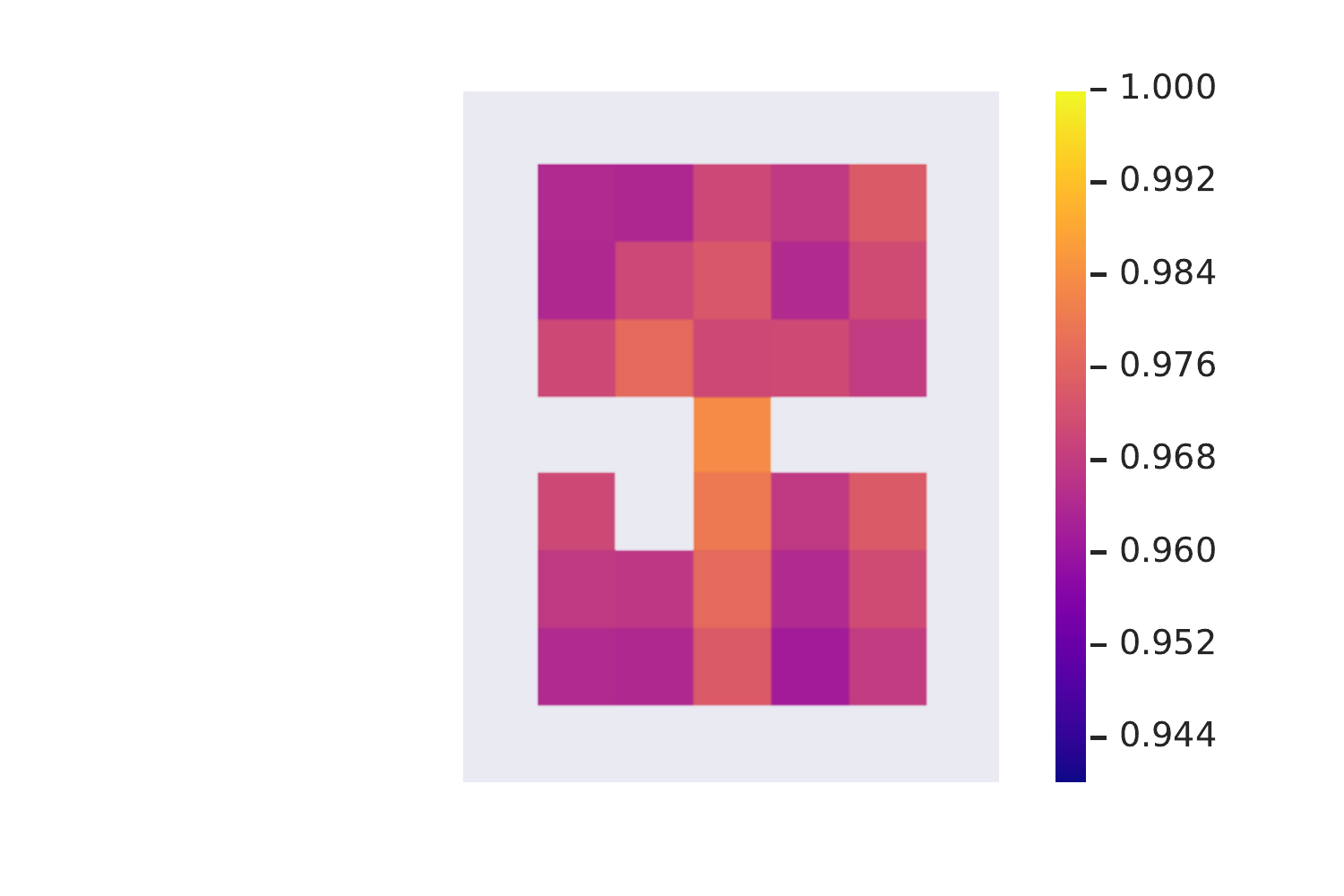}
  \end{subfigure}
  \begin{subfigure}[t]{0.18\textwidth}
    \centering\captionsetup{width=.8\linewidth}%
    \includegraphics[width=\textwidth]{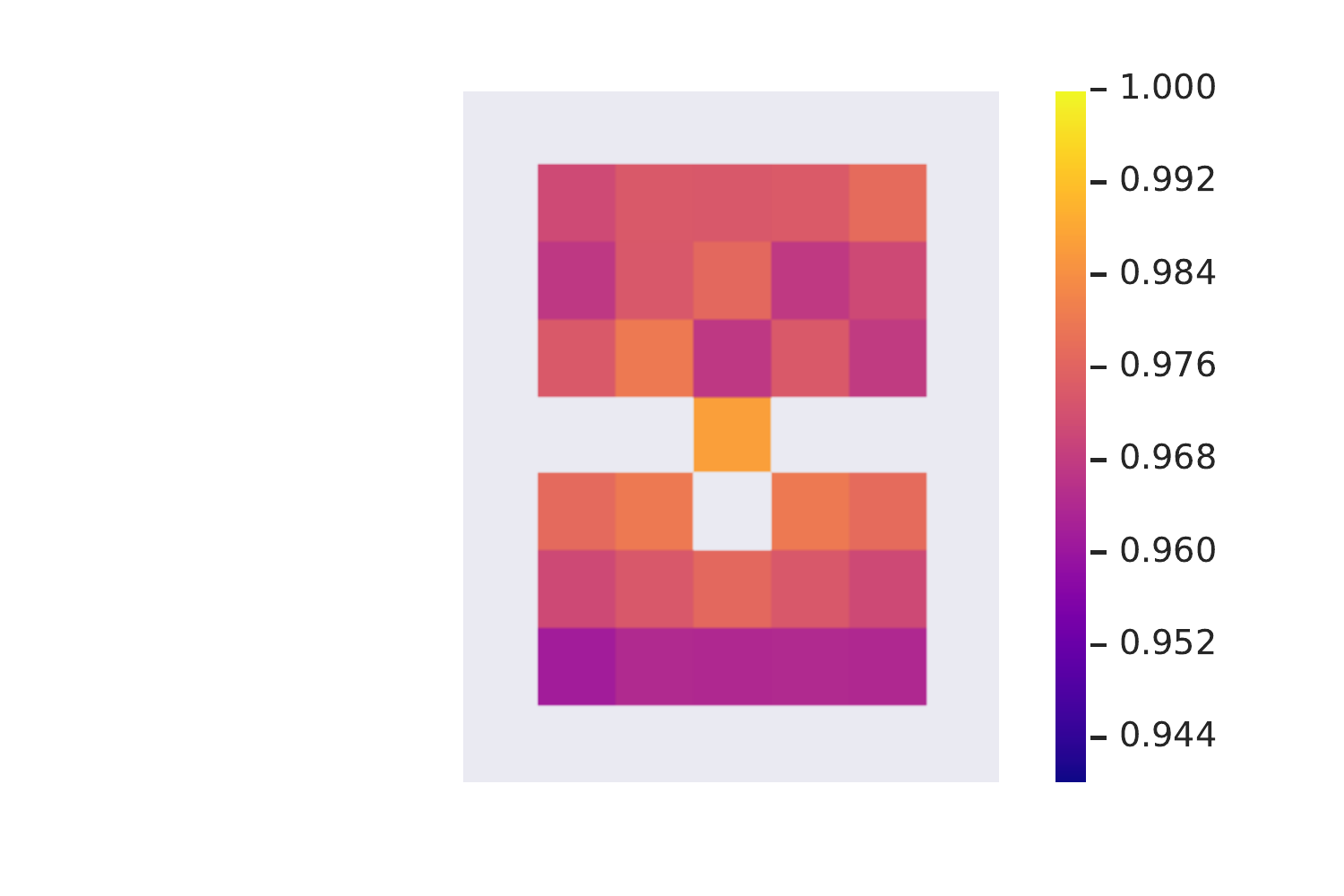}
  \end{subfigure}
  \begin{subfigure}[t]{0.18\textwidth}
    \centering\captionsetup{width=.8\linewidth}%
    \includegraphics[width=\textwidth]{dist_from_3_3}
  \end{subfigure}
  \begin{subfigure}[t]{0.18\textwidth}
    \centering\captionsetup{width=.8\linewidth}%
    \includegraphics[width=\textwidth]{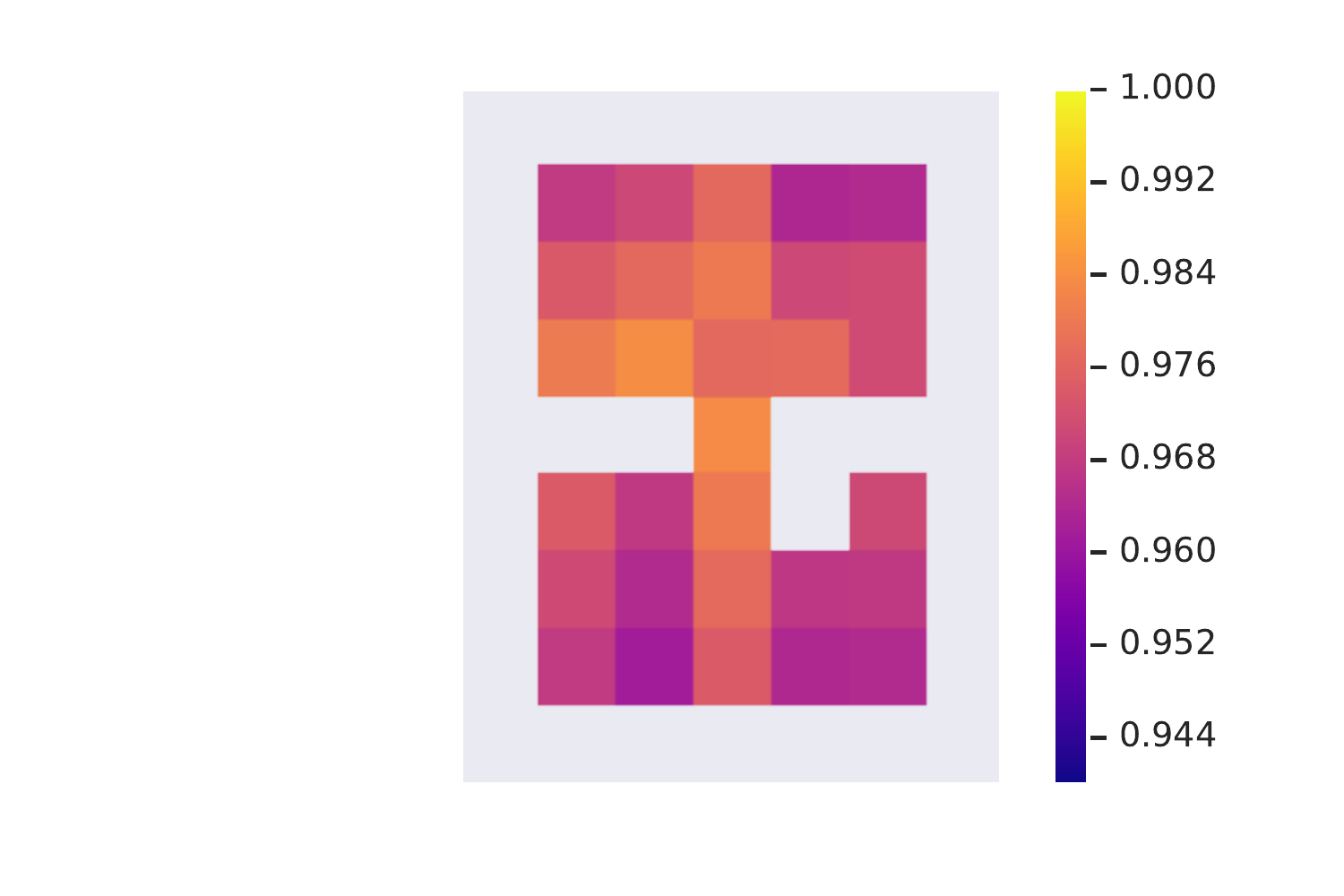}
  \end{subfigure}
  \begin{subfigure}[t]{0.18\textwidth}
    \centering\captionsetup{width=.8\linewidth}%
    \includegraphics[width=\textwidth]{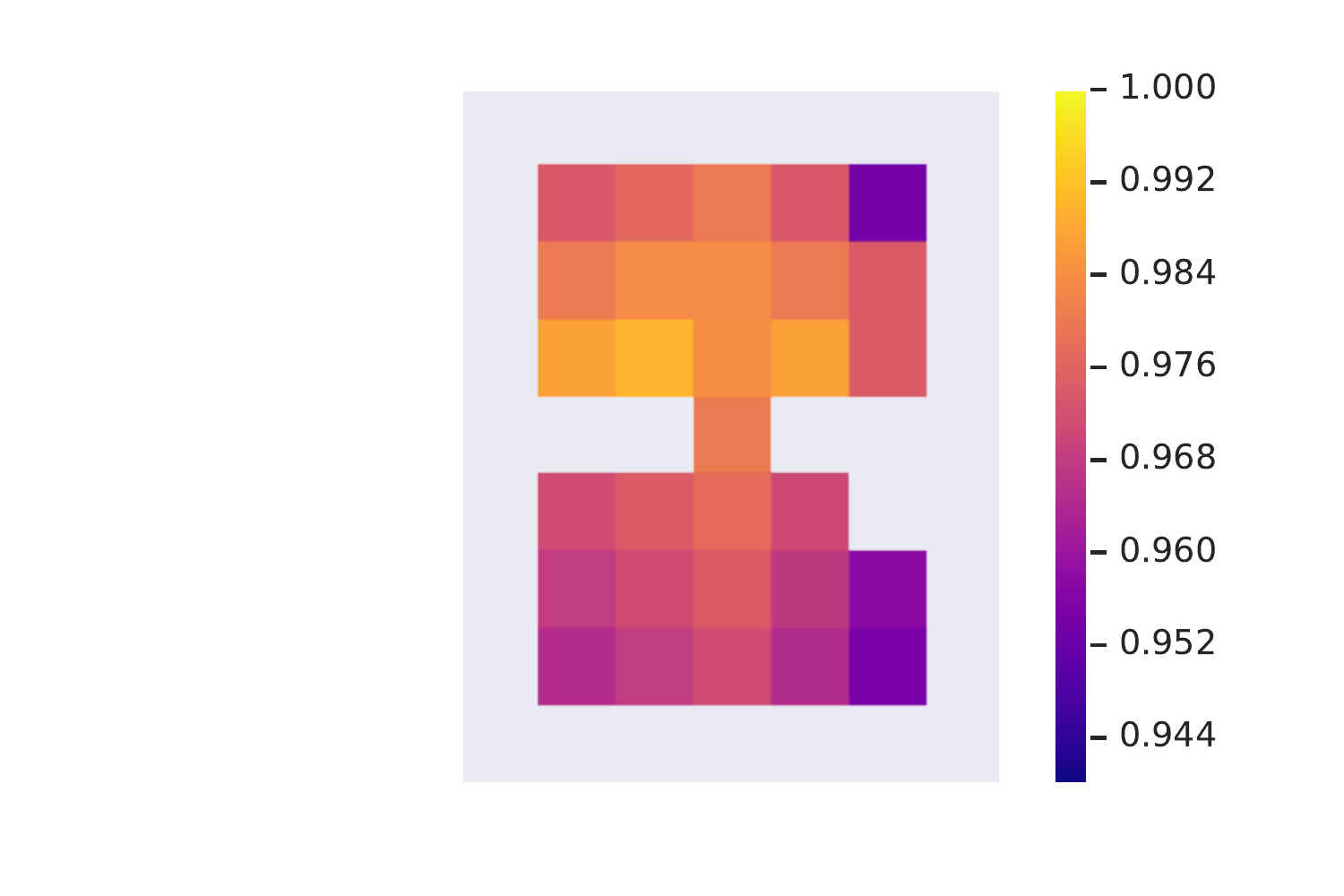}
  \end{subfigure}
  \begin{subfigure}[t]{0.18\textwidth}
    \centering\captionsetup{width=.8\linewidth}%
    \includegraphics[width=\textwidth]{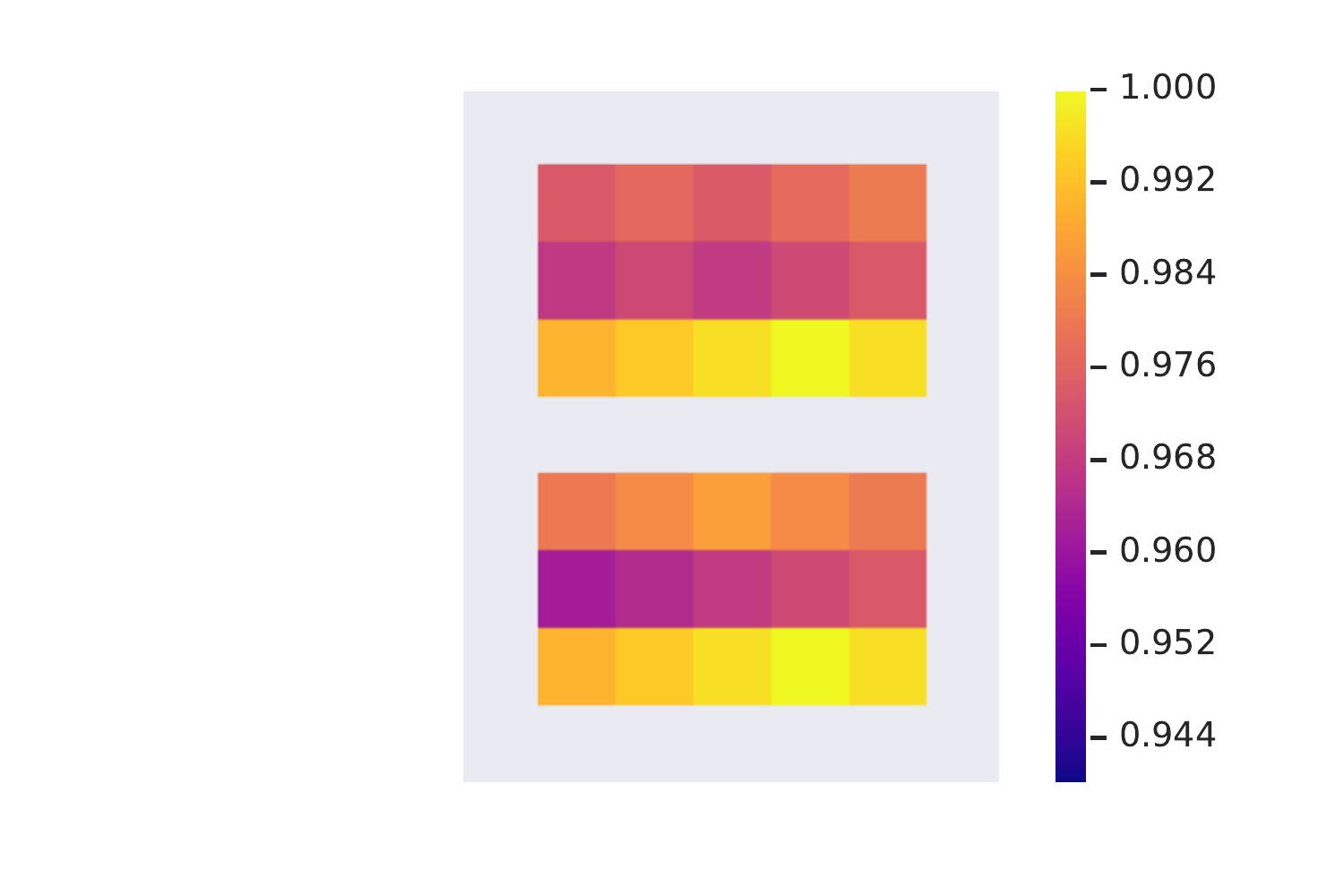}
  \end{subfigure}
  \begin{subfigure}[t]{0.18\textwidth}
    \centering\captionsetup{width=.8\linewidth}%
    \includegraphics[width=\textwidth]{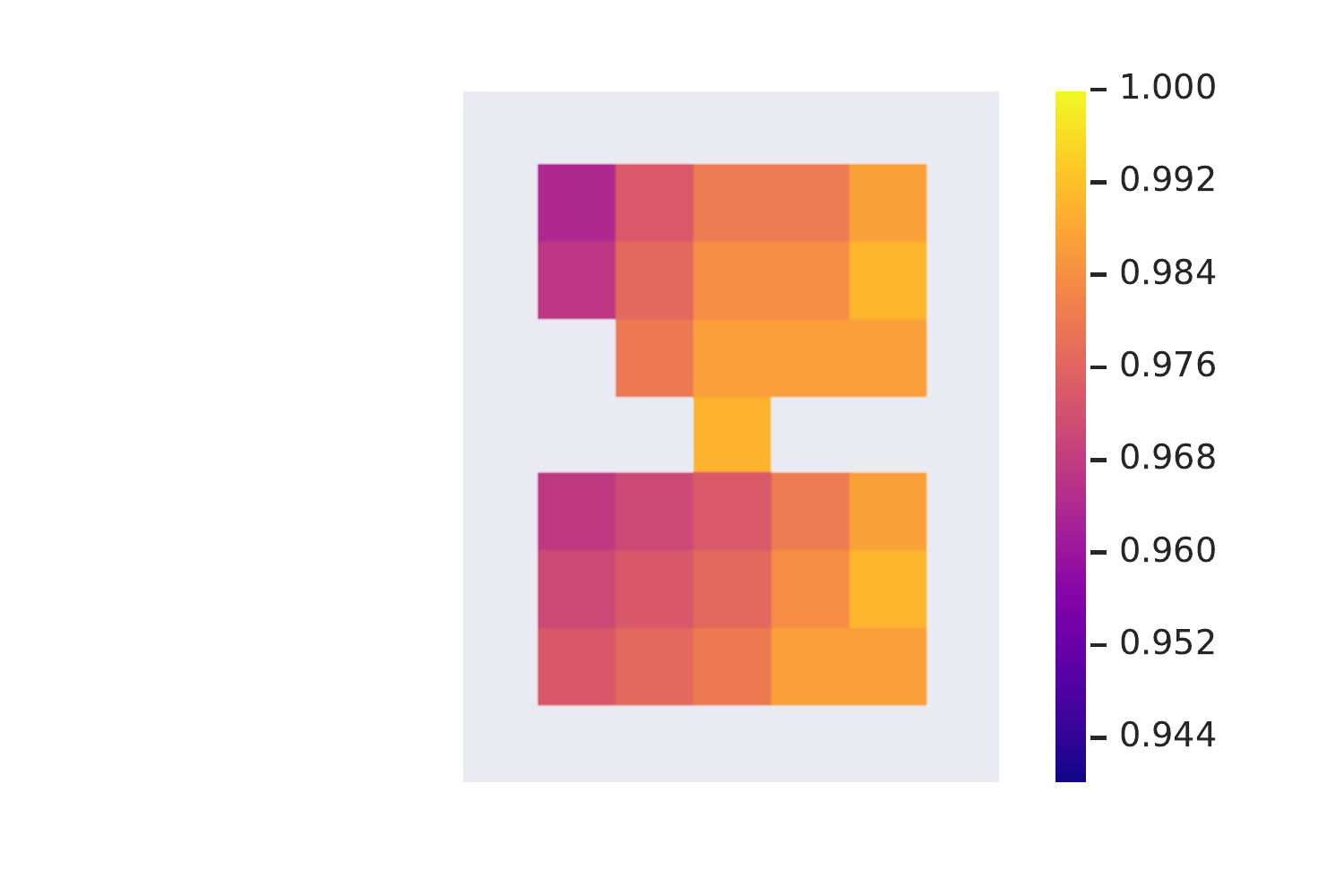}
  \end{subfigure}
  \begin{subfigure}[t]{0.18\textwidth}
    \centering\captionsetup{width=.8\linewidth}%
    \includegraphics[width=\textwidth]{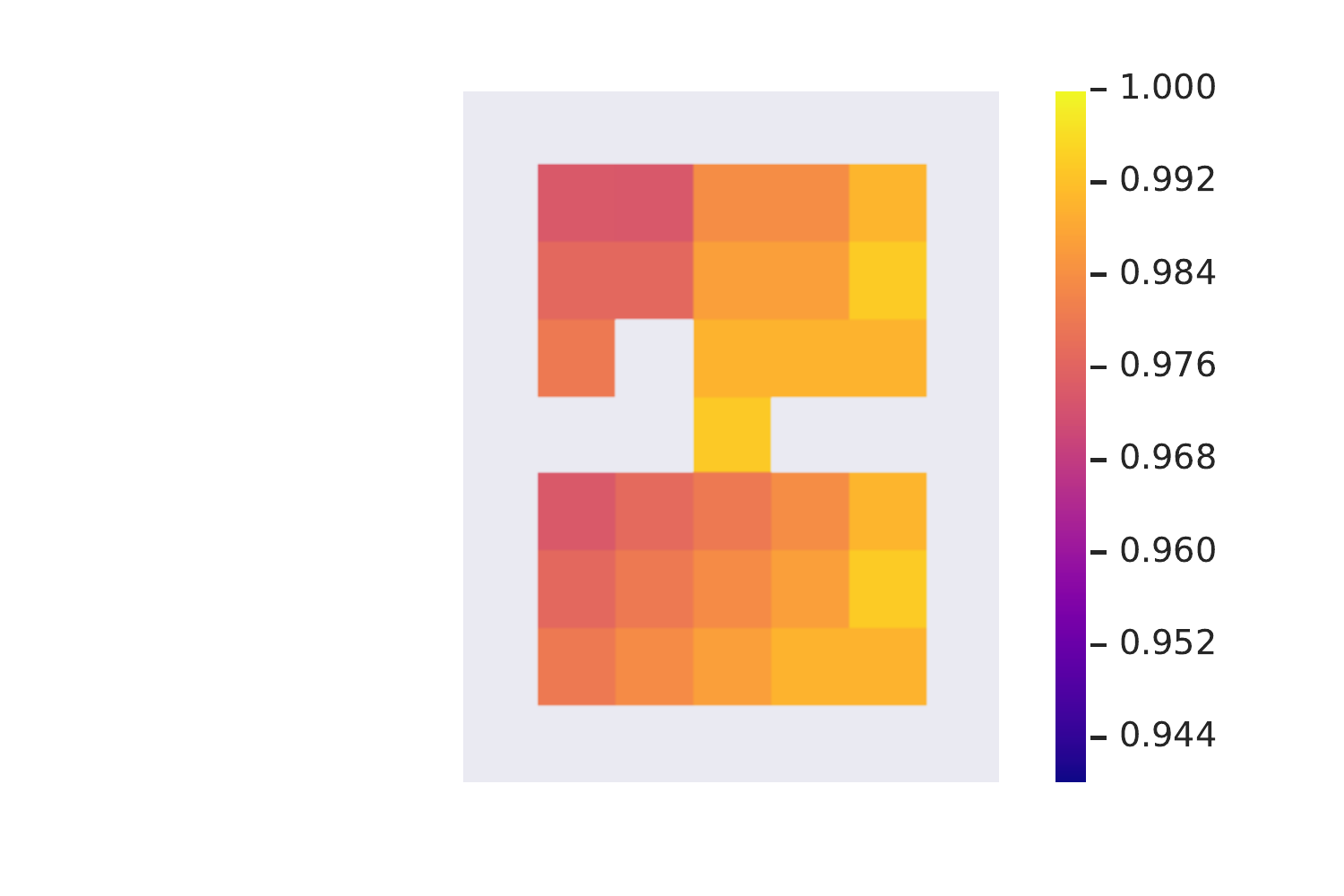}
  \end{subfigure}
  \begin{subfigure}[t]{0.18\textwidth}
    \centering\captionsetup{width=.8\linewidth}%
    \includegraphics[width=\textwidth]{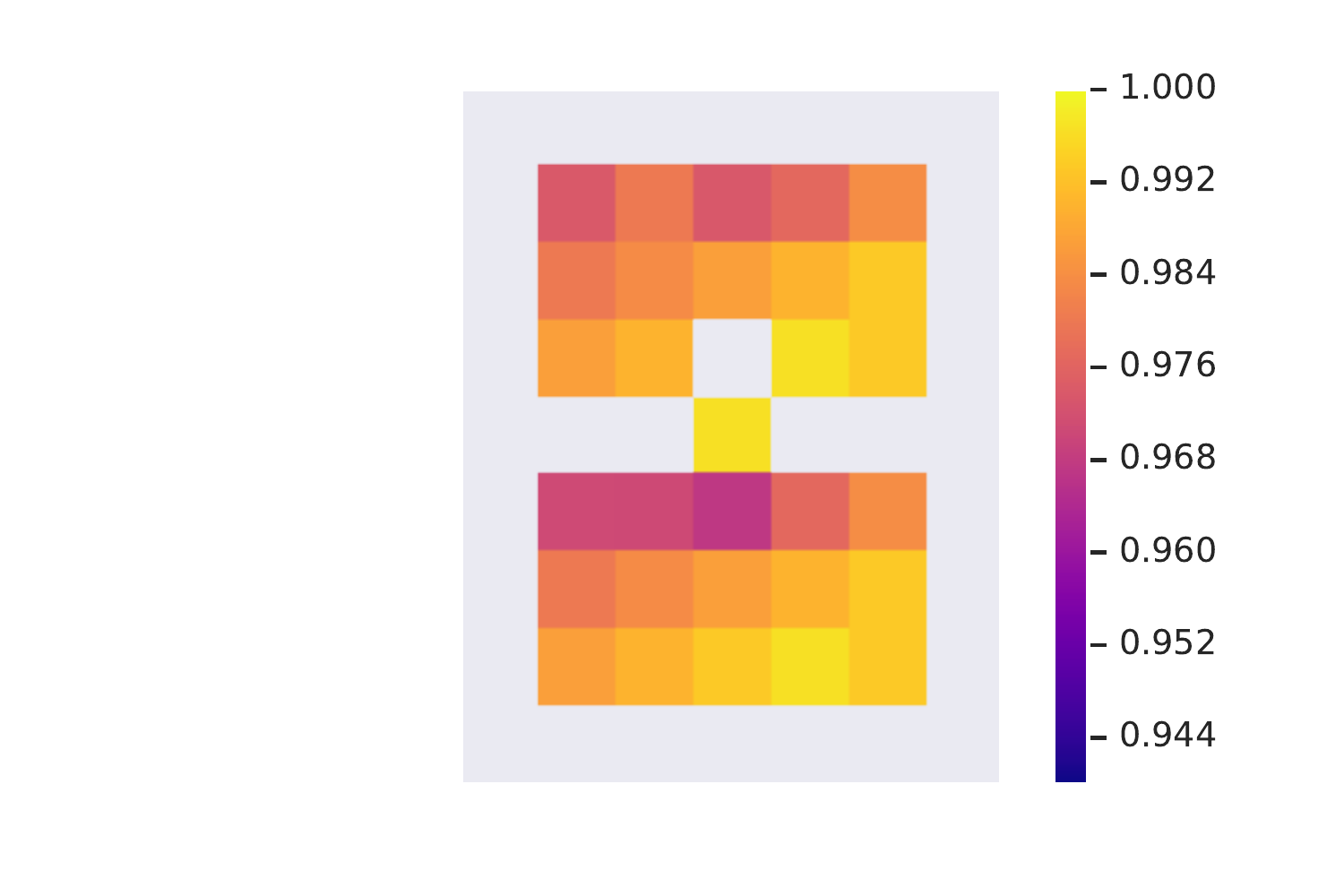}
  \end{subfigure}
  \begin{subfigure}[t]{0.18\textwidth}
    \centering\captionsetup{width=.8\linewidth}%
    \includegraphics[width=\textwidth]{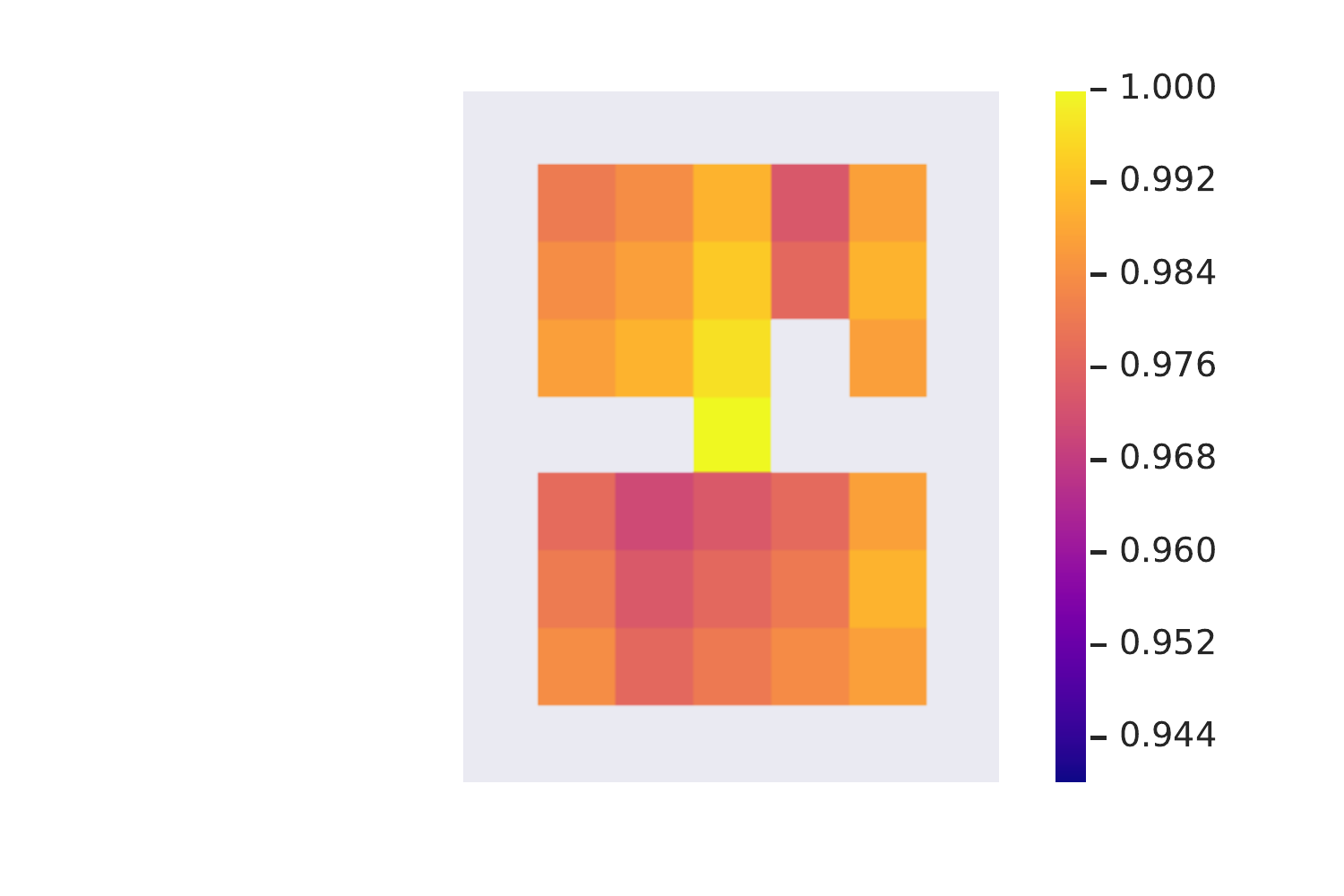}
  \end{subfigure}
  \begin{subfigure}[t]{0.18\textwidth}
    \centering\captionsetup{width=.8\linewidth}%
    \includegraphics[width=\textwidth]{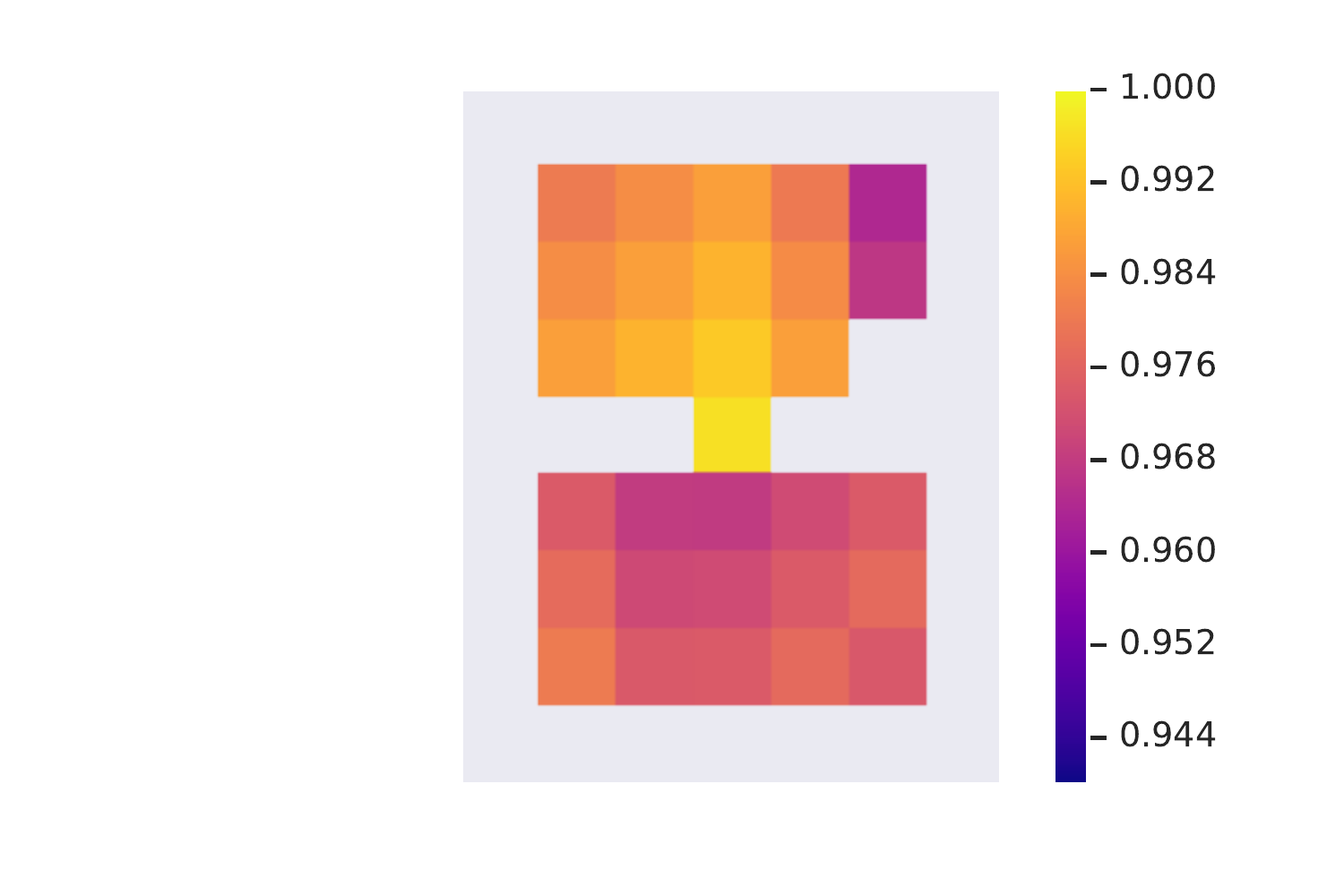}
  \end{subfigure}
  \begin{subfigure}[t]{0.18\textwidth}
    \centering\captionsetup{width=.8\linewidth}%
    \includegraphics[width=\textwidth]{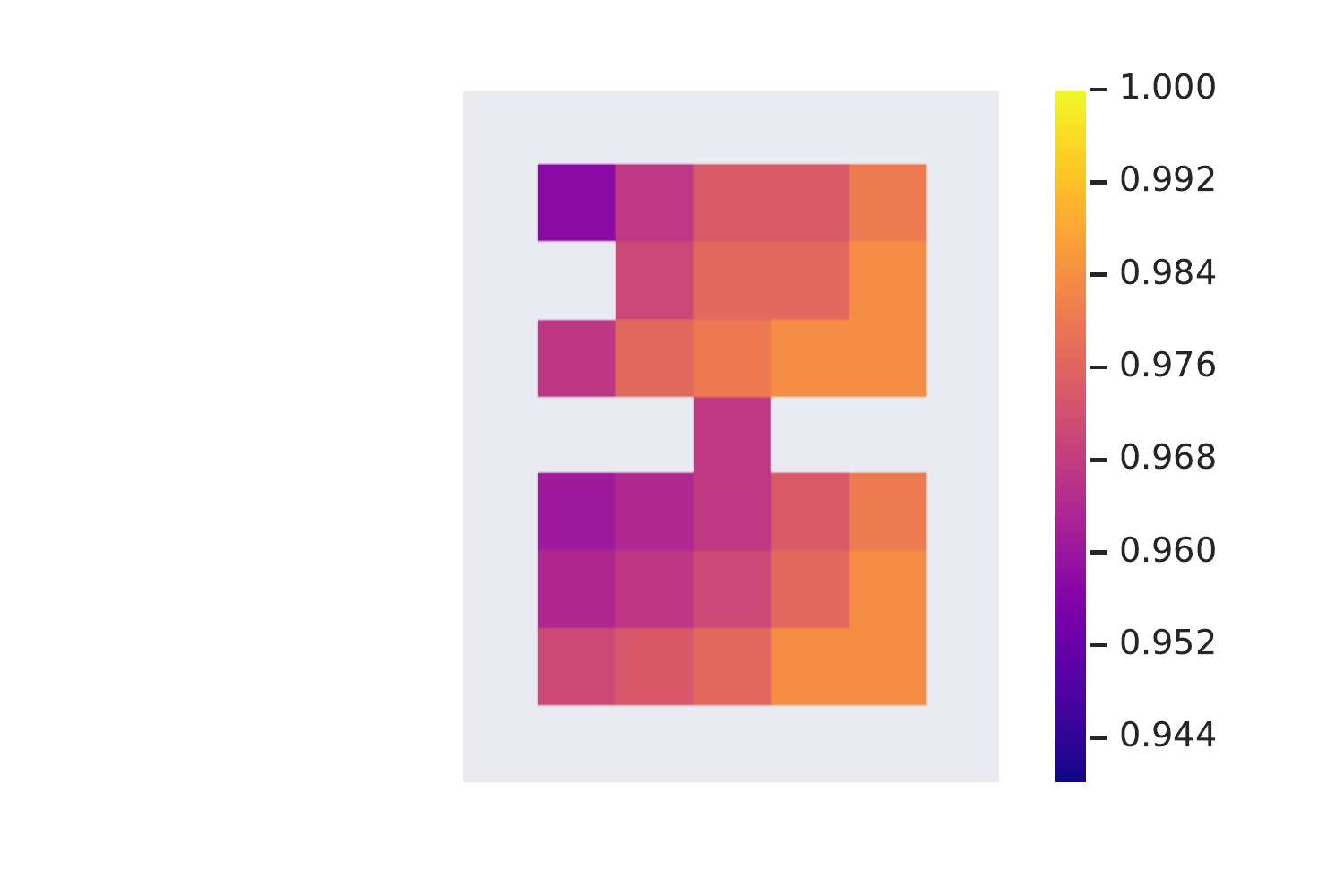}
  \end{subfigure}
  \begin{subfigure}[t]{0.18\textwidth}
    \centering\captionsetup{width=.8\linewidth}%
    \includegraphics[width=\textwidth]{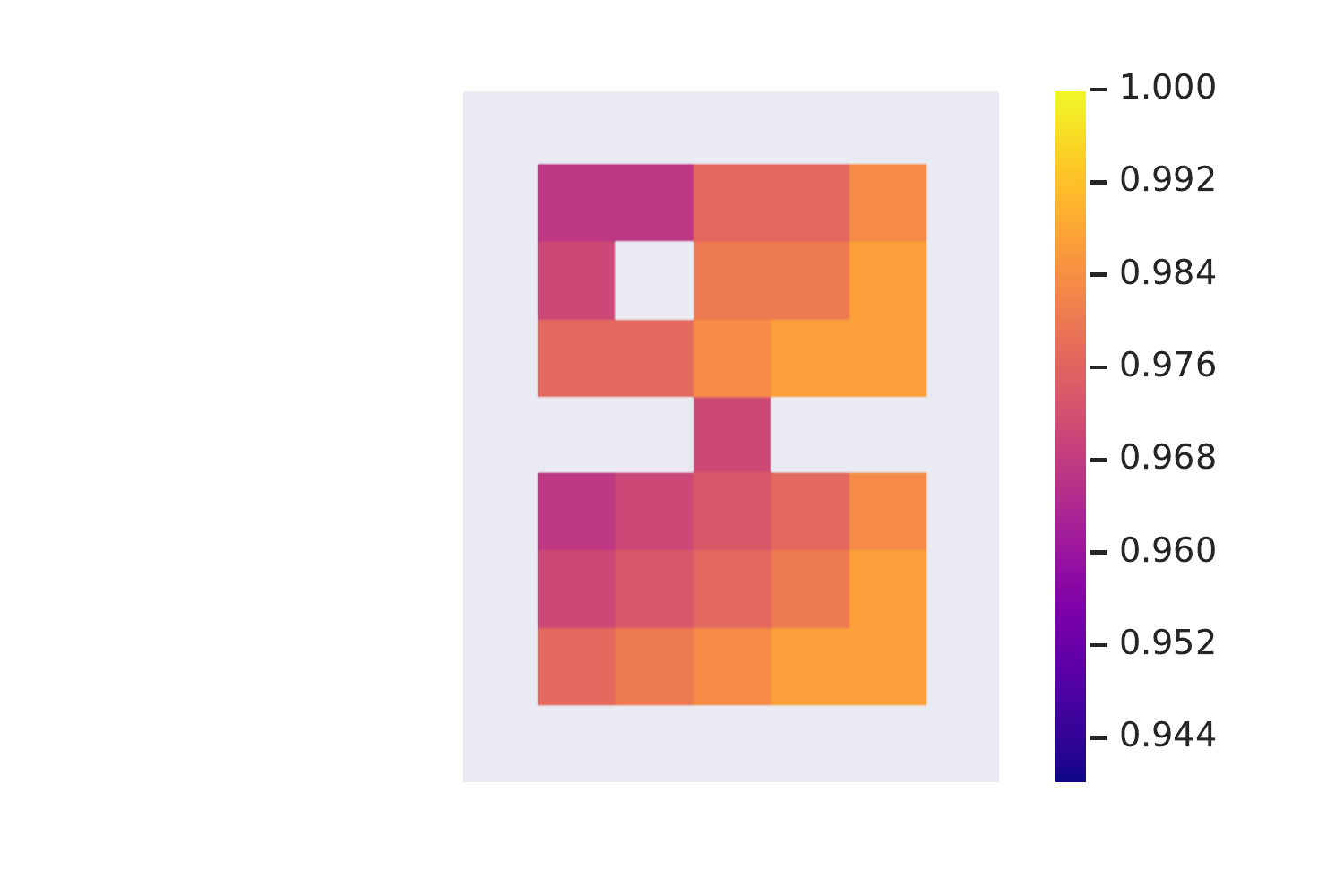}
  \end{subfigure}
  \begin{subfigure}[t]{0.18\textwidth}
    \centering\captionsetup{width=.8\linewidth}%
    \includegraphics[width=\textwidth]{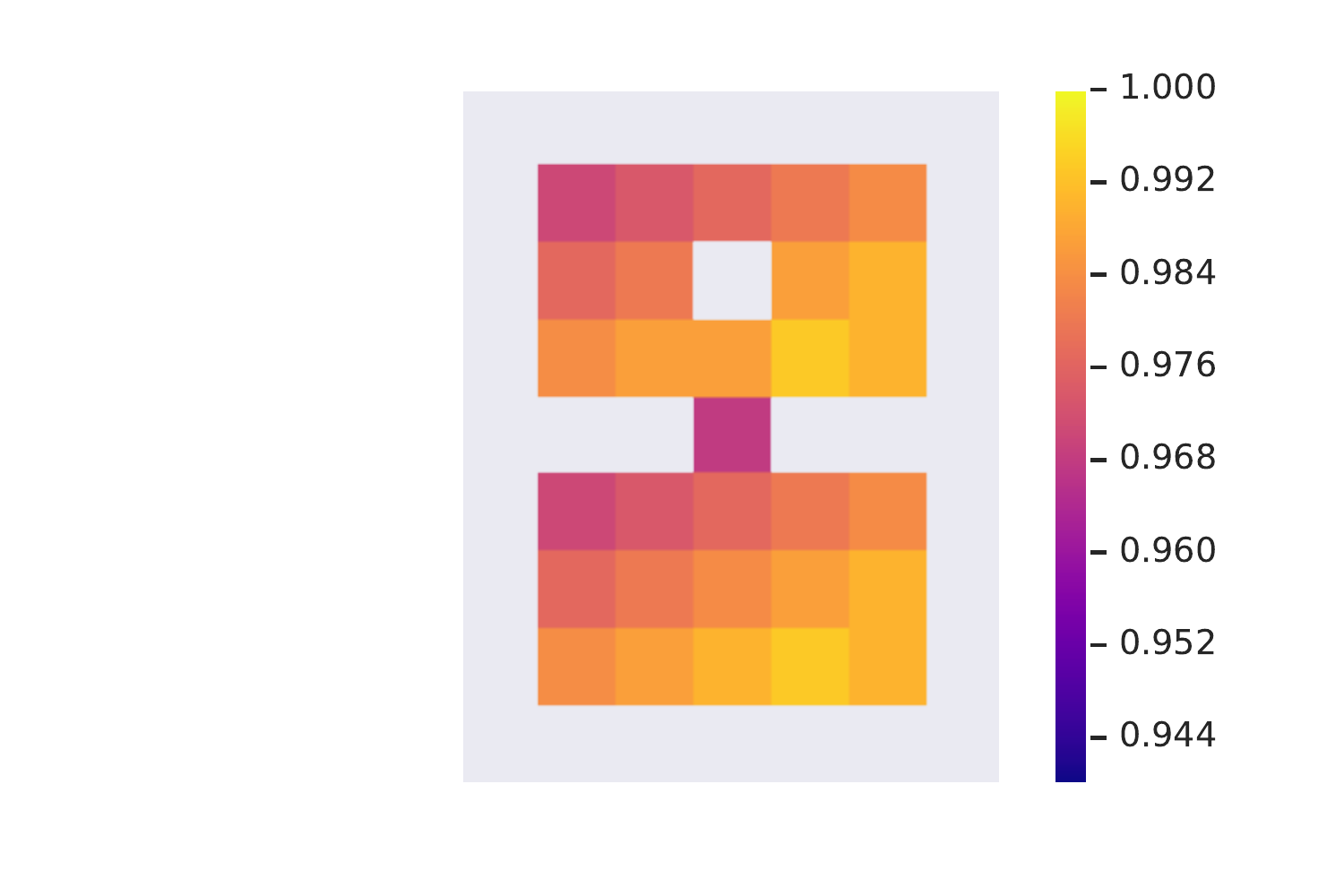}
  \end{subfigure}
  \begin{subfigure}[t]{0.18\textwidth}
    \centering\captionsetup{width=.8\linewidth}%
    \includegraphics[width=\textwidth]{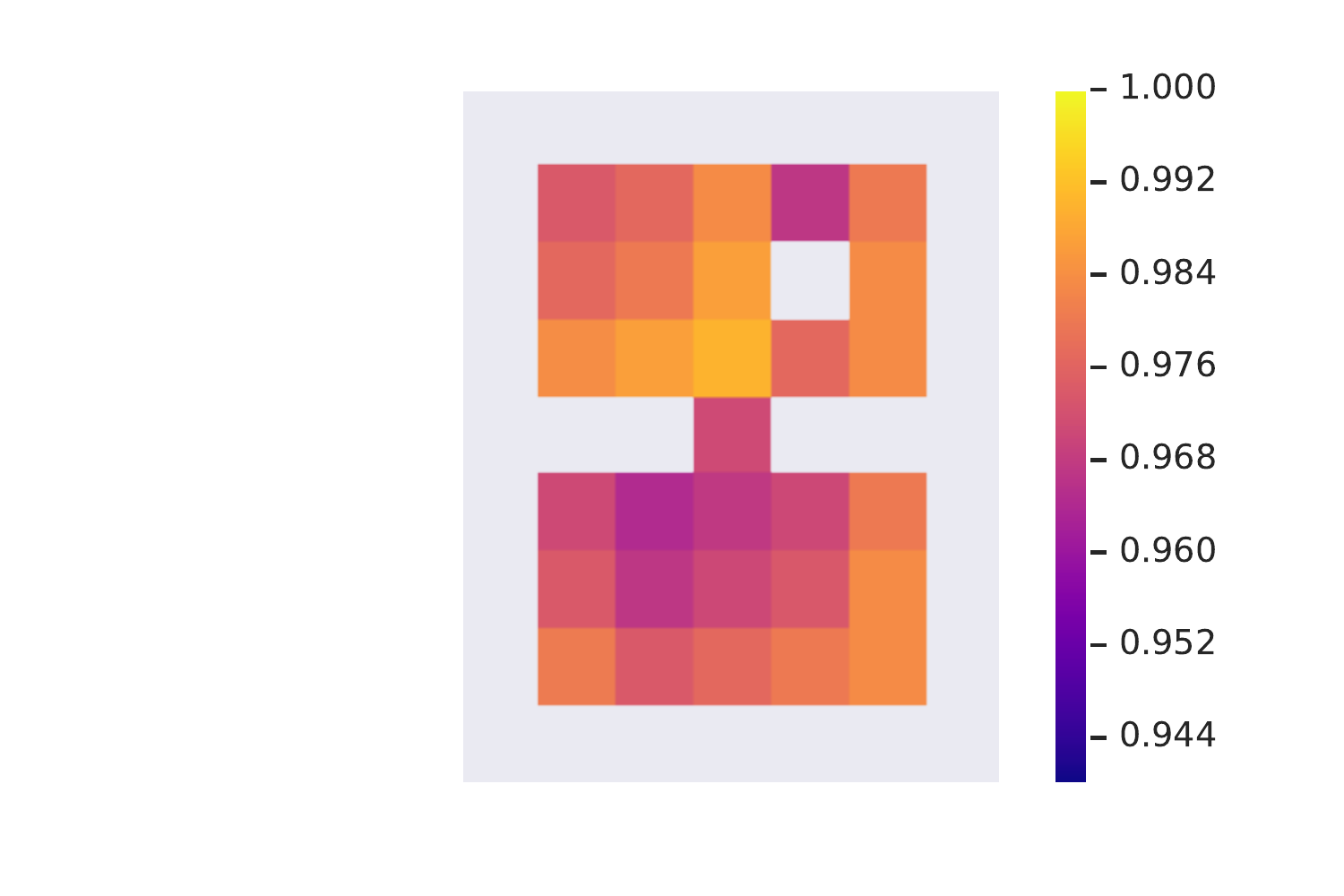}
  \end{subfigure}
  \begin{subfigure}[t]{0.18\textwidth}
    \centering\captionsetup{width=.8\linewidth}%
    \includegraphics[width=\textwidth]{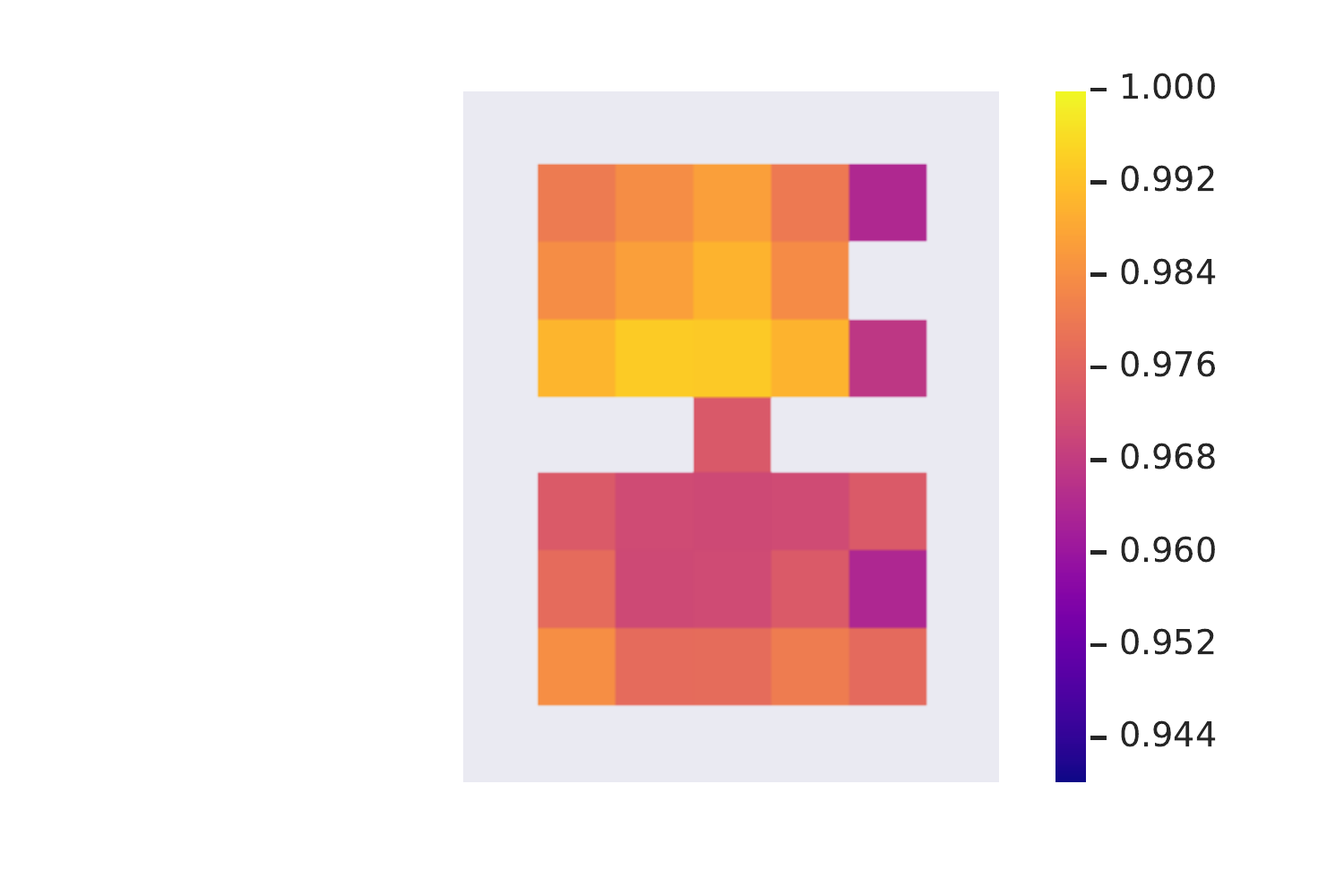}
  \end{subfigure}
  \begin{subfigure}[t]{0.18\textwidth}
    \centering\captionsetup{width=.8\linewidth}%
    \includegraphics[width=\textwidth]{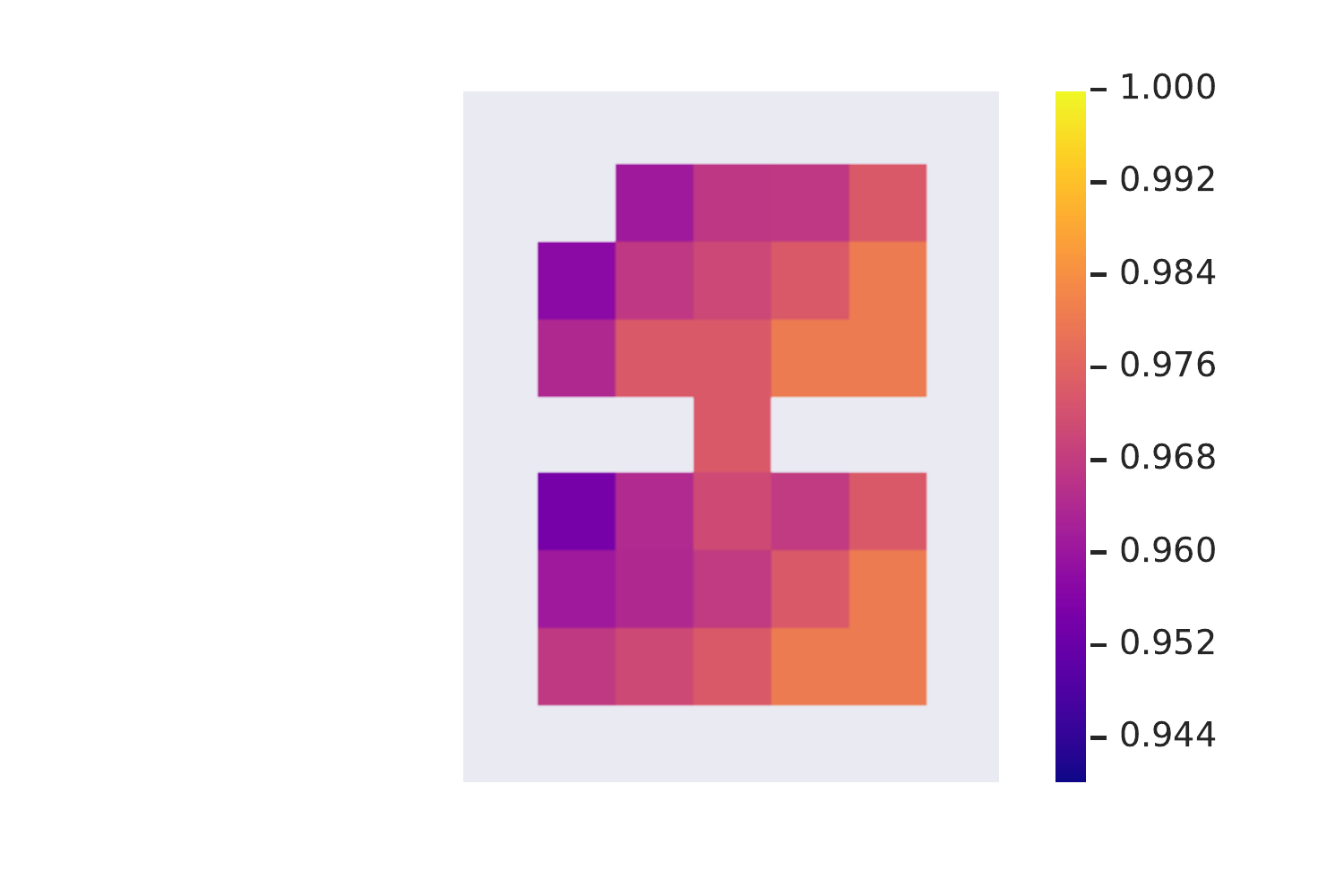}
  \end{subfigure}
  \begin{subfigure}[t]{0.18\textwidth}
    \centering\captionsetup{width=.8\linewidth}%
    \includegraphics[width=\textwidth]{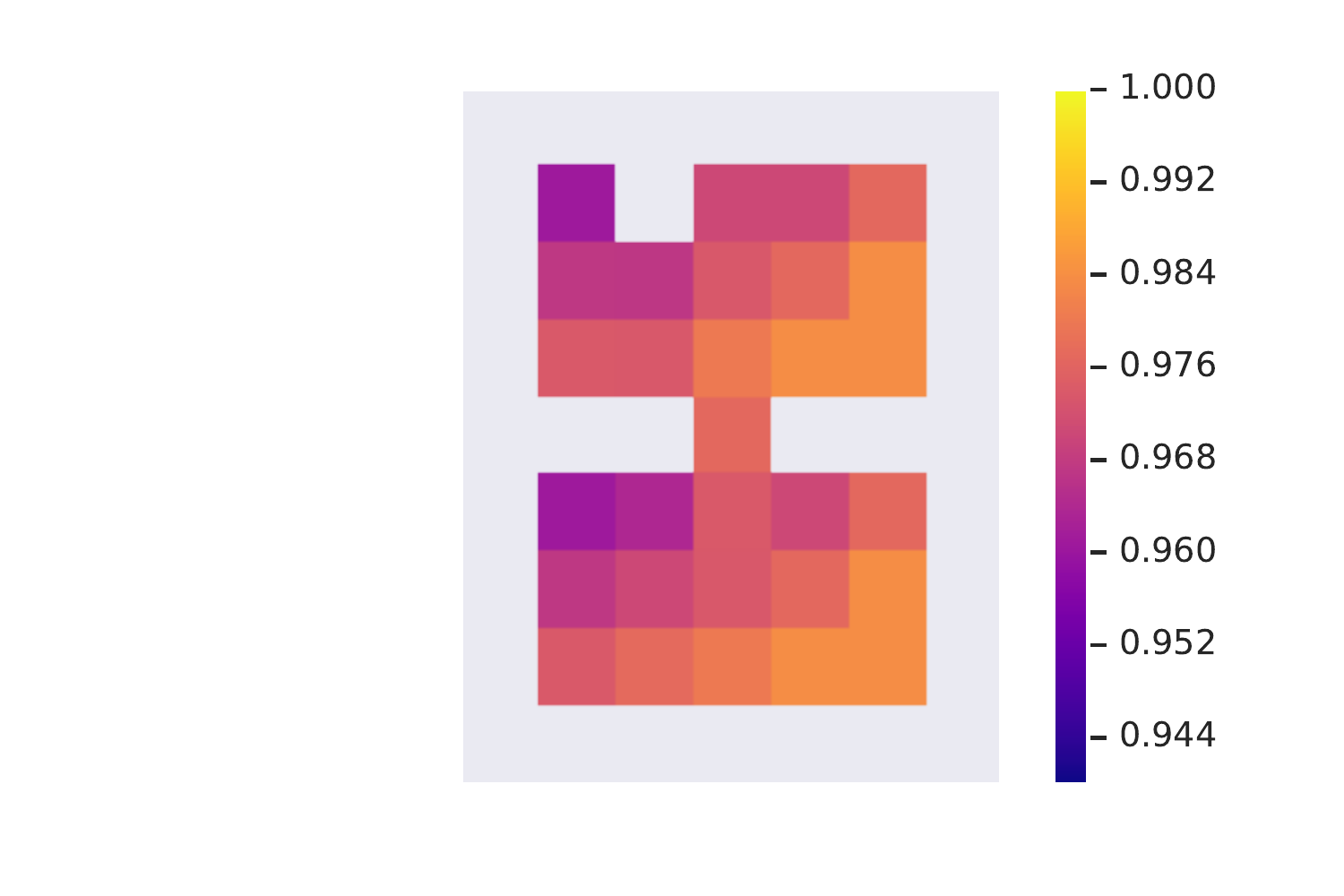}
  \end{subfigure}
  \begin{subfigure}[t]{0.18\textwidth}
    \centering\captionsetup{width=.8\linewidth}%
    \includegraphics[width=\textwidth]{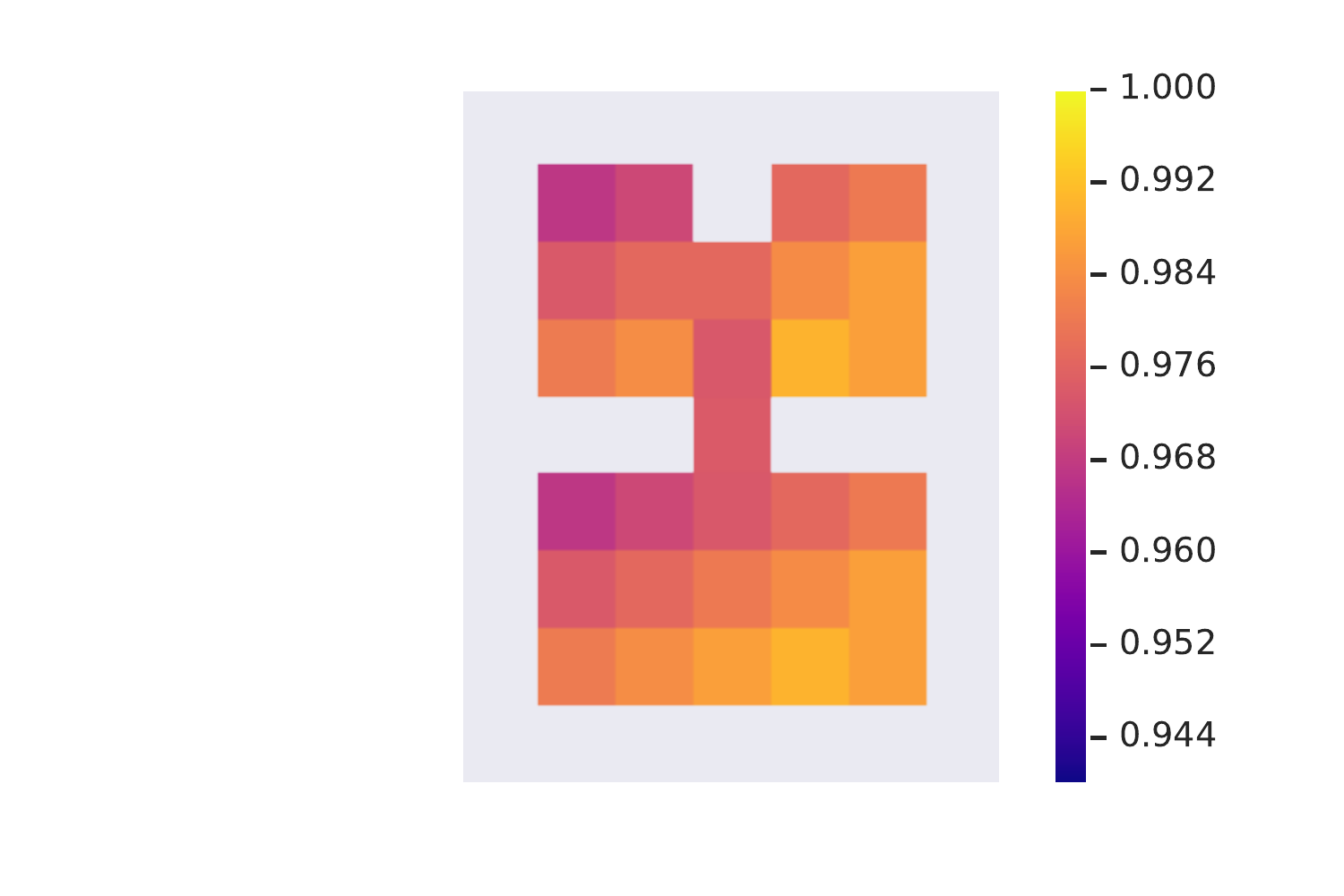}
  \end{subfigure}
  \begin{subfigure}[t]{0.18\textwidth}
    \centering\captionsetup{width=.8\linewidth}%
    \includegraphics[width=\textwidth]{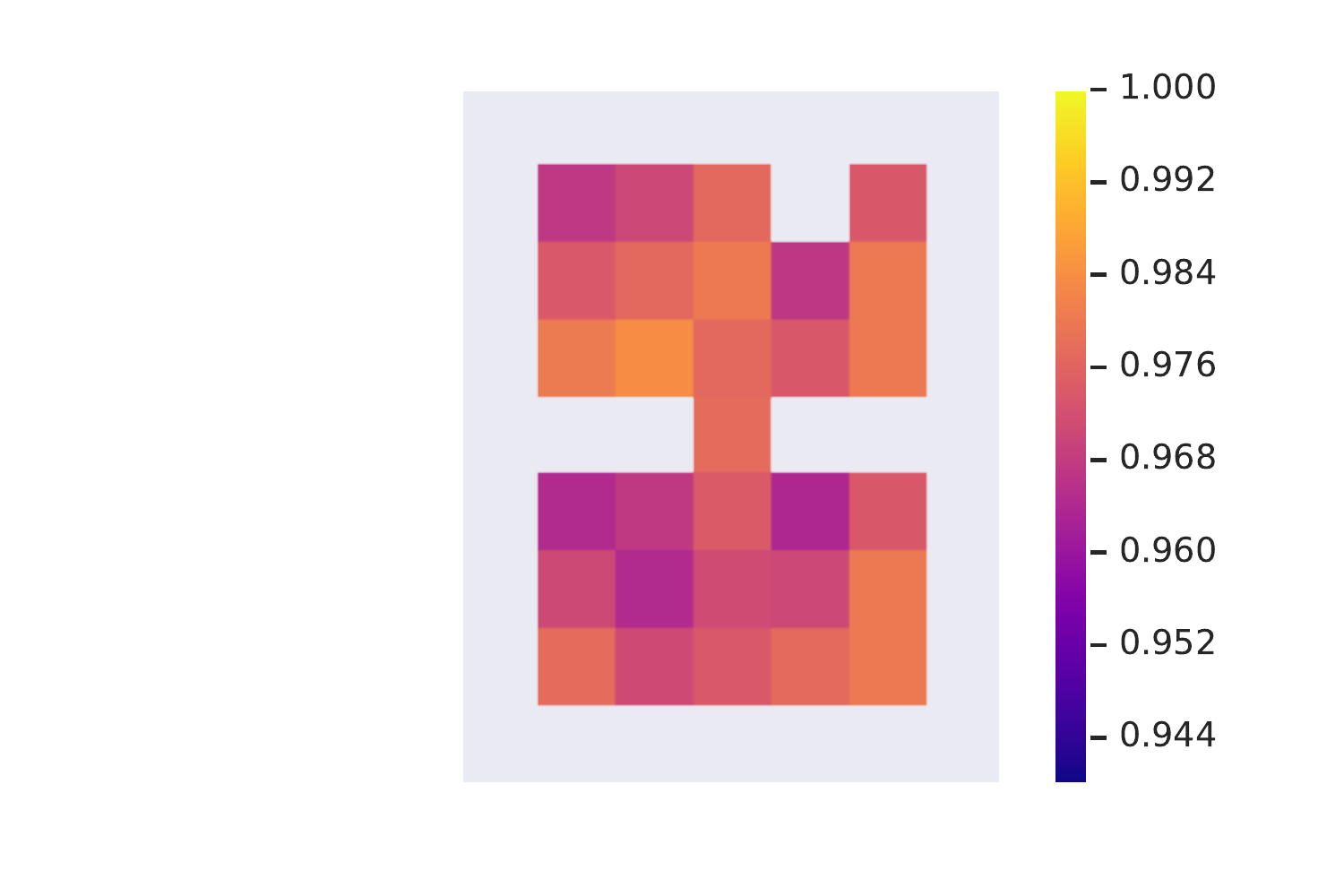}
  \end{subfigure}
  \begin{subfigure}[t]{0.18\textwidth}
    \centering\captionsetup{width=.8\linewidth}%
    \includegraphics[width=\textwidth]{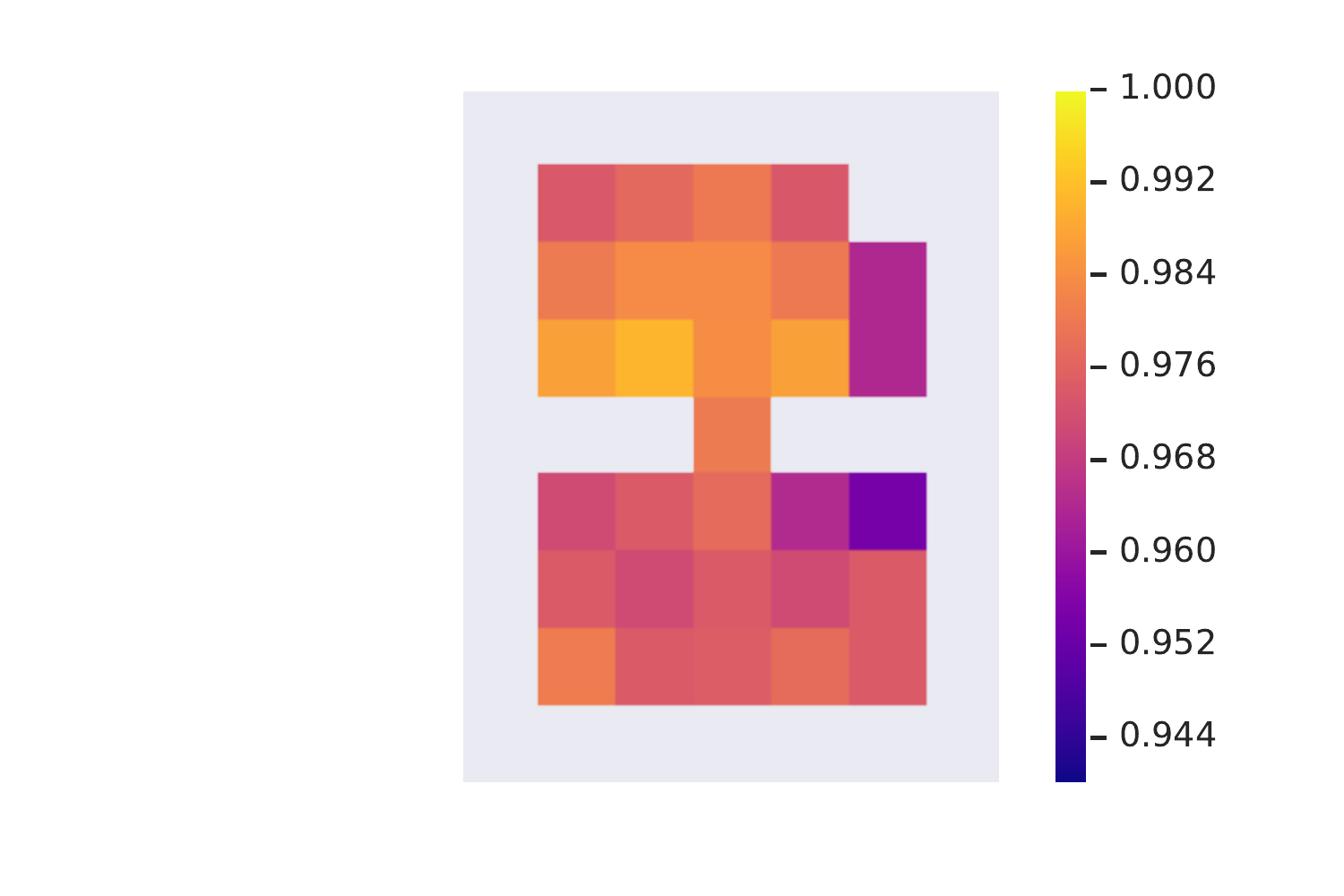}
  \end{subfigure}
  \caption{Bisimulation distances from all states.}
  \label{fig:bisim_distances}
\end{figure*}

\pagebreak

\section{Configuration file for GridWorld}
The following configuration file describes the hyperparameters used in
subsection~7.1.
\begin{lstlisting}[language=Python]
GridWorld.gamma = 0.99
GridWorld.batch_size = 256 
GridWorld.representation_dimension = 729 
GridWorld.num_iterations = 240000
GridWorld.target_update_period = 500 
GridWorld.starting_learning_rate = 0.01
GridWorld.use_decayed_learning_rate = False
GridWorld.learning_rate_decay = 0.9 
GridWorld.epsilon = 1e-8
GridWorld.staircase = False
GridWorld.add_noise = True
GridWorld.bisim_horizon_discount = 0.9 
GridWorld.double_period_halfway = False
GridWorld.debug = False
\end{lstlisting}

\section{Aggregating states}
In \autoref{fig:aggregation} we explore aggregating a set of states sampled
from the continuous variant of the 31-state MDP. We sampled 100 independent
samples for each underlying cell, computed the distances between each pair of
sampled states, and then aggregated them by incrementally growing ``clusters''
of states while ensuring that all states in a cluster are within a certain
distance of each other.  As can be seen, our learned distance is able to
capture many of the symmetries present in the environment: the orange cluster
tends to gather near-goal states, the dark-brown and dark-blue clusters seems
to gather states further away from goals, while the bright red states properly
capture the unique ``hallway'' cell. This experiment highlights the potential
for successfully approximating bisimulation metrics in continuous state MDPs,
which can render continuous environments more manageable.

\begin{figure}[h]
  \centering
  \includegraphics[width=0.45\textwidth]{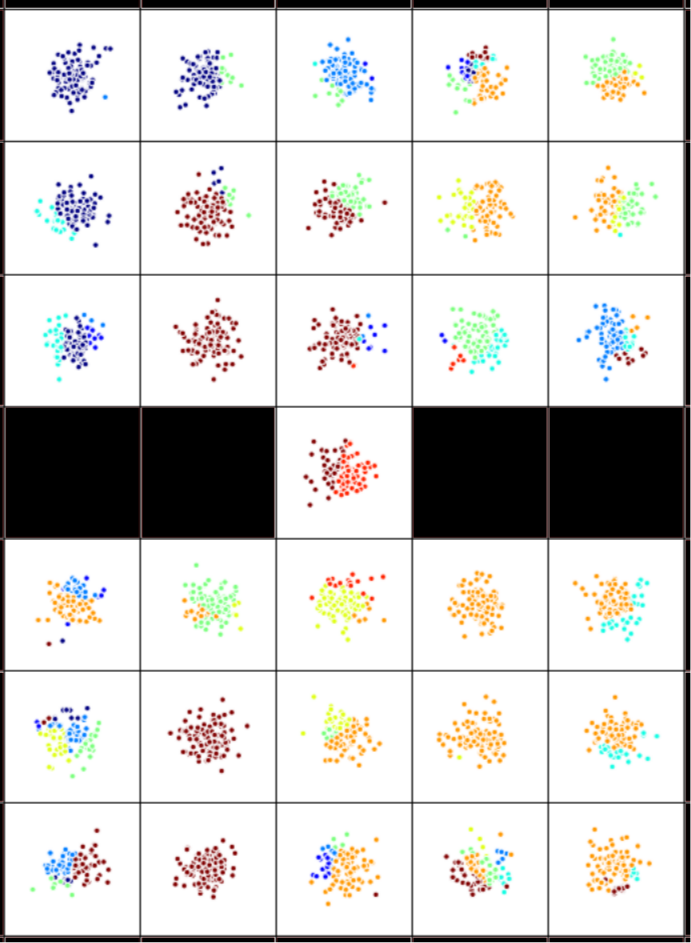}
  \caption{Aggregating samples drawn from a continuous MDP using the learned
  bisimulation metric approximant.}
  \label{fig:aggregation}
\end{figure}

The code for the clustering is displayed below:

\pagebreak

\begin{lstlisting}[language=Python]
import matplotlib.pyplot as plt
from matplotlib import cm
from matplotlib.colors import ListedColormap, LinearSegmentedColormap, BoundaryNorm
import numpy as np

clusters = []
threshold = 0.86
for i in range(len(normalized_distances)):
  present = False
  include = False
  for c in clusters:
    if np.any(np.isin(i, c)):
      present = True
      break
    present = True
    for j in c[0]:
      if normalized_distances[j, i] >= threshold:
        present = False
        break
    if present:
      np.append(c, i)
      break
  if present:
    continue
  indices = np.nonzero(normalized_distances[i] < threshold)
  clusters.append(indices)

labels = np.zeros(len(normalized_distances))
label = 0
for c in clusters:
  for i in c[0]:
    labels[i] = label
  label += 1
cmap = plt.cm.jet
cmaplist = [cmap(i) for i in range(cmap.N)]
cmap = cmap.from_list('Custom cmap', cmaplist, cmap.N)
bounds = np.linspace(0,label, label+1)
norm = BoundaryNorm(bounds, cmap.N)
fig, ax = plt.subplots(figsize=(16, 16))
ax.patch.set_facecolor('white')
ax.grid(color='black')
plt.scatter(sampled_distances['samples'][:, 1],
            sampled_distances['samples'][:, 0],
            c=labels, cmap=cmap, norm=norm)
plt.xlim(0, 7)
plt.ylim(0, 9)
plt.colorbar()
\end{lstlisting}

\pagebreak

\section{Distance plots for all games}
For each game we display four panels:
\begin{itemize}
  \item {\bf Bottom left:} Source frame (state $s$)
  \item {\bf Bottom right:} Closest frame so far
  \item {\bf Top left:} Current frame (state $t$)
  \item {\bf Top right:} Plot of distances from source frame to every other
    frame ($\psi^{\pi}_{\theta}([\phi(s), \phi(t)])$, in black) and the
    difference in value function, according to the trained Rainbow network
    ($|\hat{V}^{\pi}(\phi(s)) - \hat{V}^{\pi}(\phi(t))|$, in blue)
\end{itemize}

To pick the hyperparameters, we did a sweep over the following values, performing
3 independent runs for each setting. We picked a setting which gave the best
overall performance across all games (final values specified in the $rainbow.gin$
file provided with the source code, except for Pong where we used a learning rate
of $0.001$, as specified in the paper):
\begin{itemize}
  \item {\bf Learning rates:} $[0.00005, 0.000075, 0.0001, 0.00025, 0.0005, 0.00075, 0.001]$
  \item {\bf Batch size:} $[4, 8, 16, 32, 64, 128]$
  \item {\bf Target update period:} $[10, 50, 100, 250, 500, 1000, 2000]$
  \item {\bf Number of hidden layers:} $[1, 2, 3]$
  \item {\bf Number of hidden units per layer:} $[16, 32, 64, 128]$
\end{itemize}

\begin{figure}[!h]
  \centering
  \includegraphics[width=0.8\textwidth]{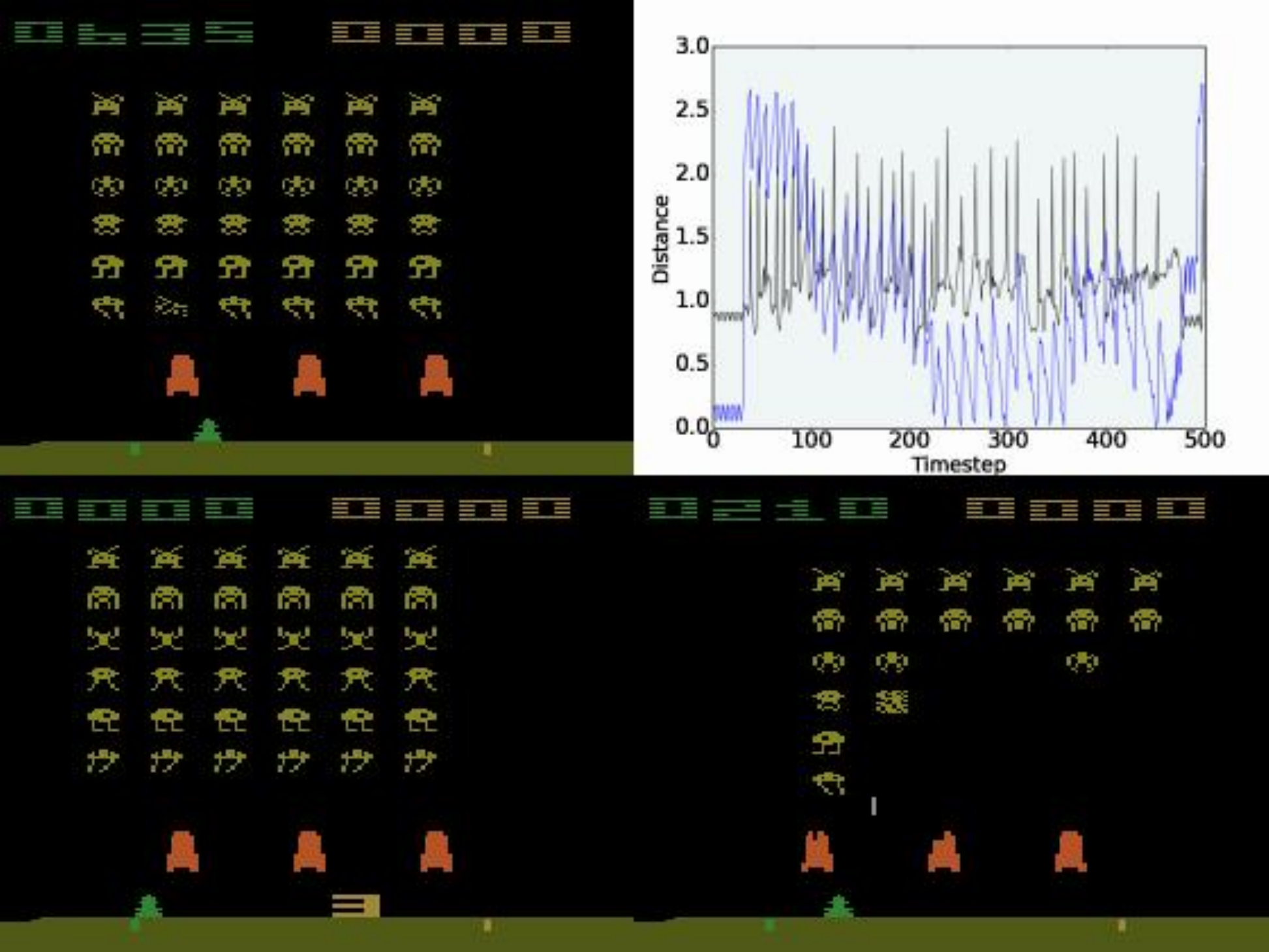}
  \caption{Space Invaders: evaluation at frame 500.}
\end{figure}

\begin{figure}[!h]
  \centering
  \includegraphics[width=0.8\textwidth]{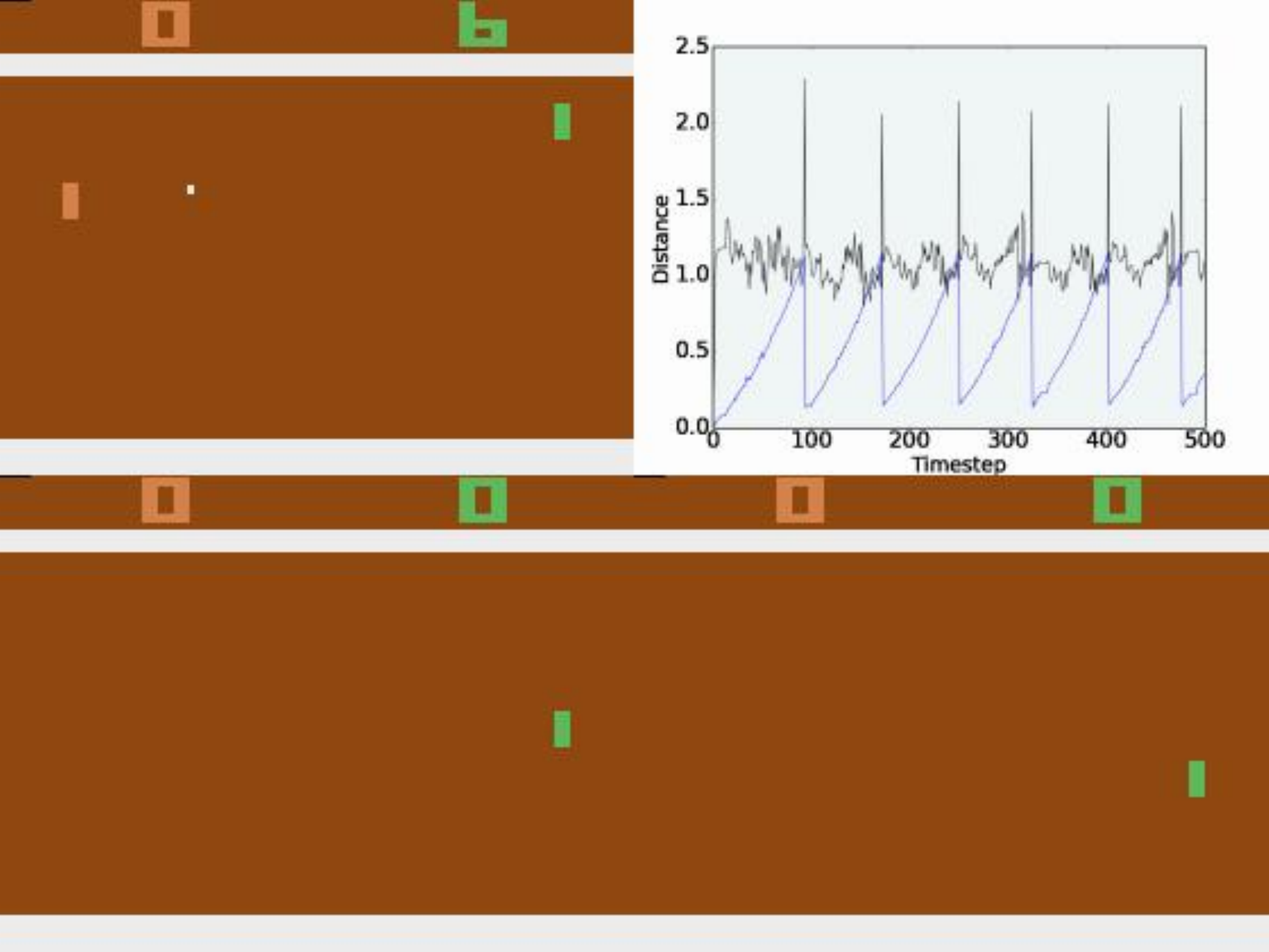}
  \caption{Pong: evaluation at frame 500, peaks are when the agent scores.}
\end{figure}

\begin{figure}[!h]
  \centering
  \includegraphics[width=0.8\textwidth]{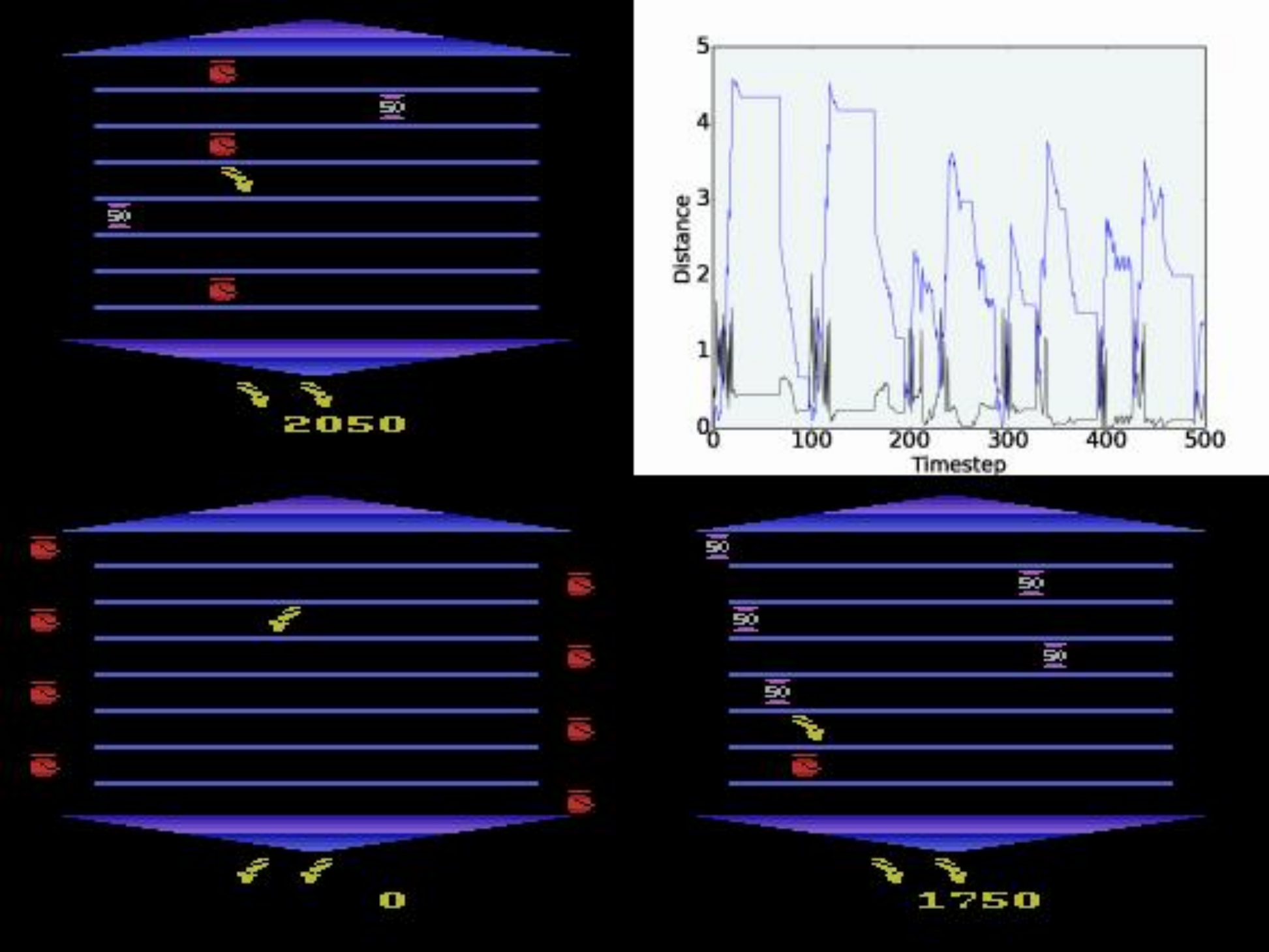}
  \caption{Asterix: evaluation at frame 500. Plateaus are when the agent is not
  moving.}
\end{figure}

\pagebreak

\section{Training curves for $d^{\pi}_{\sim}$}
We provide the training curves for the bisimulation metric $d^{\pi}_{\sim}$
learned on the trained reinforcement learning agents.

\begin{figure*}[t]
  \begin{subfigure}[t]{0.33\textwidth}
    \centering
    \includegraphics[width=\textwidth]{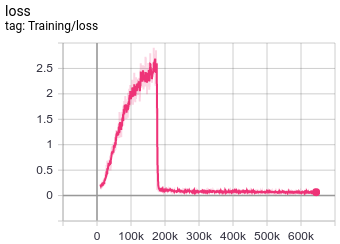}
  \end{subfigure}
  \begin{subfigure}[t]{0.33\textwidth}
    \centering
    \includegraphics[width=\textwidth]{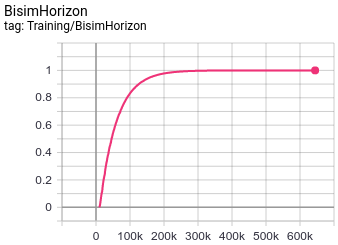}
  \end{subfigure}
  \begin{subfigure}[t]{0.33\textwidth}
    \centering
    \includegraphics[width=\textwidth]{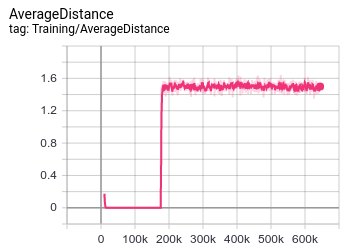}
  \end{subfigure}
  \caption{Training curves for Space Invaders.}
\end{figure*}

\begin{figure*}[t]
  \begin{subfigure}[t]{0.33\textwidth}
    \centering
    \includegraphics[width=\textwidth]{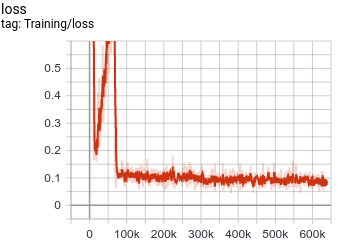}
  \end{subfigure}
  \begin{subfigure}[t]{0.33\textwidth}
    \centering
    \includegraphics[width=\textwidth]{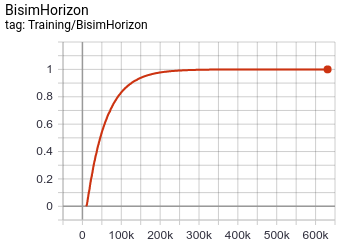}
  \end{subfigure}
  \begin{subfigure}[t]{0.33\textwidth}
    \centering
    \includegraphics[width=\textwidth]{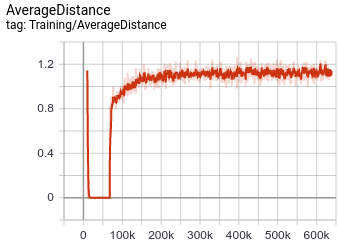}
  \end{subfigure}
  \caption{Training curves for Asterix.}
\end{figure*}

\begin{figure*}[t]
  \begin{subfigure}[t]{0.33\textwidth}
    \centering
    \includegraphics[width=\textwidth]{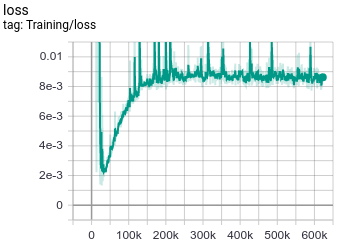}
  \end{subfigure}
  \begin{subfigure}[t]{0.33\textwidth}
    \centering
    \includegraphics[width=\textwidth]{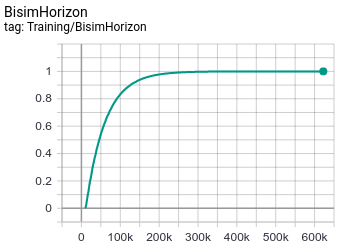}
  \end{subfigure}
  \begin{subfigure}[t]{0.33\textwidth}
    \centering
    \includegraphics[width=\textwidth]{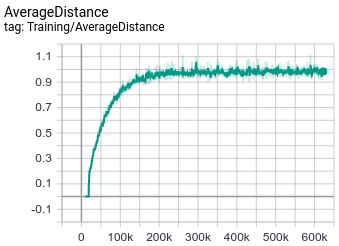}
  \end{subfigure}
  \caption{Training curves for Pong.}
\end{figure*}

\section{Lax bisimulation metrics}
For completeness, in this section we include the main definitions and
theoretical results of lax-bisimulation metrics introduced in
\cite{taylor09bounding}, modified to match the notation used in this paper.
\begin{definition}
  A relation $E\subseteq\cS\times\cS$ is a lax (probabilistic) bisimulation
  relation if whenever $(s,t)\in E$ we have that:
  \begin{enumerate}
    \item $\forall a\in\cA\exists b\in\cA$ such that $\cR(s, a)=\cR(t, b)$
    \item $\forall a\in\cA\exists b\in\cA.\quad\forall c\in\cS_{E}.\cP(s, a)(c) = \cP(t, b)(c)$,\\
      where $\cP(x, y)(c) = \sum_{z\in c}\cP(x, y)(z)$,
  \end{enumerate}
  The lax bisimulation $\sim_{lax}$ is the union of all lax bisimulation relations.
\end{definition}

\begin{definition}
  Given a $1$-bounded pseudometric $d\in\rM$, the metric
  $\delta(d):\cS\times\cA\rightarrow [0,1]$ is defined as follows:
  \[ \delta(d)((s, a), (t, b)) = |\cR(s, a) - \cR(t, b)| + \gamma\cW(d)(\cP(s, a), \cP(t, b)) \]
\end{definition}

\begin{definition}
  Given a finite $1$-bounded metric space $(\mathfrak{M}, d)$ let
  $\mathcal{C}(\mathfrak{M})$ be the set of compact spaces (e.g. closed and
  bounded in $\rR$). The {\em Hausdorff metric}
  $H(d):\mathcal{C}(\mathfrak{M})\times\mathcal{C}(\mathfrak{M})\rightarrow [0,
  1]$ is defined as:
  \[ H(d)(X, Y) = \max\left(\sup_{x\in X}\inf_{y\in Y} d(x, y), \sup_{y\in Y}\inf_{x\in X} d(x, y)\right) \]
\end{definition}

\begin{definition}
  Denote $X_s=\lbrace (s, a) | a\in\cA\rbrace$. We define the operator
  $F:\rM\rightarrow\rM$ as:
  \[ F(d)(s, t) = H(\delta(d))(X_s, X_t) \]
\end{definition}

\begin{theorem}
  $F$ is monotonic and has a least fixed point $d_{lax}$ in which
  $d_{lax}(s, t) = 0$ iff $s\sim_{lax} t$.
\end{theorem}

\begin{theorem}
  $|V^*(s) - V^*(t)| \leq d_{lax}(s, t) \leq d_{\sim}(s, t)$.
\end{theorem}

\end{document}